\documentclass{article} 
\usepackage{iclr2026_conference,times}

\usepackage{amsmath,amsfonts,bm}

\def\eqref#1{equation~\ref{#1}}

\def\1{\bm{1}}

\DeclareMathAlphabet{\mathsfit}{\encodingdefault}{\sfdefault}{m}{sl}
\SetMathAlphabet{\mathsfit}{bold}{\encodingdefault}{\sfdefault}{bx}{n}

\usepackage{hyperref}
\usepackage{url}            
\usepackage{booktabs}       
\usepackage{amsfonts}       
\usepackage{nicefrac}       
\usepackage{microtype}      

\usepackage{tcolorbox}
\usepackage{enumitem}
\tcbuselibrary{skins, breakable}
\usepackage{fontawesome5} 
\usepackage{algorithm}
\usepackage{algpseudocode}
\usepackage{amsmath}

\usepackage{graphicx}
\usepackage{subfigure}
\usepackage{supertabular}
\usepackage{algorithm}
\usepackage{algorithm}
\usepackage{algpseudocode}
\usepackage{url}

\usepackage{amsmath, amsfonts, amssymb, amsthm, amsbsy, amscd, bm, bbm,mathrsfs}   
\newtheorem*{theorem*}{Theorem}
\usepackage[capitalize,noabbrev]{cleveref}
\usepackage{float}  
\usepackage{float}
\usepackage{xcolor}      

\theoremstyle{plain}
\newtheorem{theorem}{Theorem}[section]
\newtheorem{proposition}[theorem]{Proposition}
\newtheorem{lemma}[theorem]{Lemma}

\theoremstyle{definition}
\newtheorem{definition}[theorem]{Definition}

\theoremstyle{remark}
\newtheorem{remark}[theorem]{Remark}
\usepackage{algorithmicx}
\usepackage{algpseudocode}

\usepackage{multirow}

\usepackage{colortbl}

\title{FlowNIB: An Information Bottleneck Analysis of Bidirectional vs. Unidirectional Language Models}

\iclrfinalcopy
\author{\textbf{Md Kowsher}\textsuperscript{1}\thanks{Equal contribution},
\textbf{ Nusrat Jahan Prottasha}\textsuperscript{1}\footnotemark[1],
\textbf{ Shiyun Xu}\textsuperscript{2}, 
\textbf{Shetu Mohanto}\textsuperscript{3}, 
\textbf{Ozlem Garibay}\textsuperscript{1}, \\
\textbf{Niloofar Yousefi}\textsuperscript{1}.
\textbf{Chen Chen}\textsuperscript{1}\\ 
\textsuperscript{1}University of Central Florida
\textsuperscript{2}University of Pennsylvania
\textsuperscript{3}Delineate Inc.\\ 
\faGithub~\href{https://github.com/Kowsher/BidiVsUniLM}{\textcolor{red}{\texttt{github.com/Kowsher/BidiVsUniLM}}}
}

\begin{document}

\maketitle

\begin{abstract}
Bidirectional language models (LMs) consistently show stronger context understanding than unidirectional models, yet the theoretical reason remains unclear. We present a simple information bottleneck (IB) perspective: bidirectional representations preserve more mutual information (MI) about both the input and the target, yielding richer features for downstream tasks. We adopt a layer–wise view and hypothesize that, at comparable capacity, bidirectional layers retain more useful signal than unidirectional ones. To test this claim empirically, we present  \textbf{Flow} \textbf{N}eural \textbf{I}nformation \textbf{B}ottleneck (FlowNIB), a lightweight, post-hoc framework capable of estimating comparable mutual information values for individual layers in LMs, quantifying how much mutual information each layer carries for a dataset. FlowNIB takes three inputs—(i) the original LM's inputs/dataset, (ii) ground–truth labels, and (iii) layer activations—simultaneously estimates the mutual information for both the input–layer and layer–label pairs. Empirically, bidirectional LM layers exhibit higher mutual information than similar—and even larger—unidirectional LMs.  As a result, bidirectional LMs outperform unidirectional LMs across extensive experiments on NLU benchmarks (e.g., GLUE), commonsense reasoning, and regression tasks, demonstrating superior context understanding.
\end{abstract}

\section{Introduction}

Large language models have brought significant advancements in natural language understanding (NLU) tasks. Among them, bidirectional models such as BERT have demonstrated superior performance in natural language understanding, while unidirectional models like GPT dominate generation tasks. As shown in Table 1 of \cite{devlin2019bert}, the BERT-base model outperforms GPT \citep{radford2018improving} across all GLUE benchmarks \citep{wang2018glue} despite having a comparable model size -- for example, achieving 66.4\% accuracy on the RTE task versus GPT’s 56.0\%. Moreover, the empirical evidence \citep{li2022explanation, liu2019roberta,raffel2020exploring,clark2020electra} consistently demonstrate that bidirectional LMs outperform unidirectional LMs on a wide range of NLU tasks. 
\begin{figure}[htbp]
\begin{center}
    \includegraphics[width=1.00\linewidth]{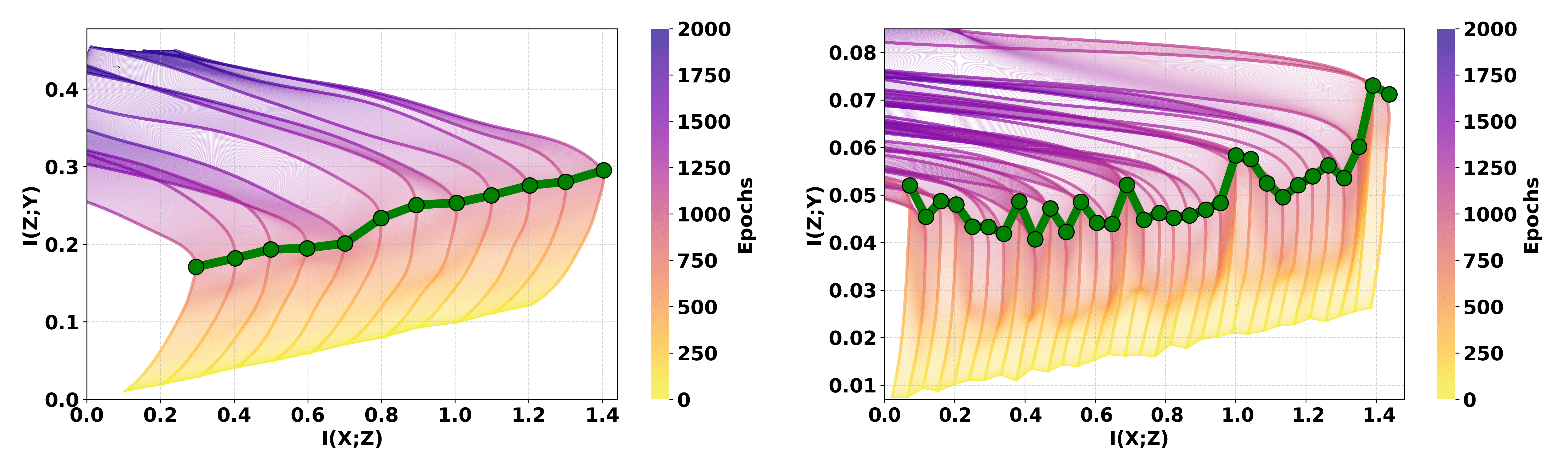}
\end{center}
\caption{
Information\mbox{-}plane trajectories under FlowNIB training for (left) DeBERTaV3\mbox{-}Base and (right) MobileLLM\mbox{-}350M on MRPC. Each curve shows mutual information \(I(Z;Y)\) versus \(I(X;Z)\) over training epochs, colored by epoch progression. A constant offset of \(+0.05\) is added to \(I(X;Z)\) for each successive layer to visually separate the layerwise trajectories. The green line represents the \emph{Optimal Information Coordinate (OIC)} across layers.
}

\label{fig:info_tra}
\end{figure}

While the empirical advantage of bidirectional models is well documented, a clear theoretical account is limited. We adopt an information\mbox{-}theoretic view based on the Information Bottleneck (IB) principle \citep{tishby2000information}. Let \(Z\) be a layer representation and write \(I(X;Z)\) for the mutual information between the input \(X\) and \(Z\), and \(I(Z;Y)\) for the mutual information between \(Z\) and the label \(Y\). In IB, desirable representations \emph{compress} the input (small \(I(X;Z)\)) while \emph{preserving} task\mbox{-}relevant content (large \(I(Z;Y)\)). 

Our claim is that, at comparable capacity, a bidirectional layer retains more information about the input and transmits more information relevant to predicting the target than a unidirectional layer; formally, for corresponding layers \(\ell\) : $ I(X;Z^{\leftrightarrow}_\ell) \;\ge\; I(X;Z^{\rightarrow}_\ell), I(Z^{\leftrightarrow}_\ell;Y) \;\ge\; I(Z^{\rightarrow}_\ell;Y)$
with strict inequalities under mild conditions (e.g., when future context reduces input uncertainty or contributes predictive signal). Intuitively, the bidirectional representation \(Z^{\leftrightarrow}_\ell\) conditions on both past and future tokens, whereas the unidirectional representation \(Z^{\rightarrow}_\ell\) conditions only on the past. Since conditioning reduces entropy \citep{madiman2010information}, 
\(H(X \mid Z^{\leftrightarrow}_\ell) \le H(X \mid Z^{\rightarrow}_\ell)\), and therefore 
\(I(X;Z^{\leftrightarrow}_\ell) \ge I(X;Z^{\rightarrow}_\ell)\).
 To make the IB analysis applicable to LMs, we formalize the following:

\begin{definition}[A valid information plane (post hoc)]
\label{def:information_plane}
Let a language model (LM) have $L$ hidden layers with layer-$\ell$ output $Z_\ell$ for $\ell=1,\dots,L$, input $X$, and target $Y$ under data distribution $p(x,y)$.
Let $\{ I^{(t)}\}_{t\ge0}$ denote a mutual information estimator family (e.g., MINE, InfoNCE) obtained by training the estimator for $t$ internal steps on $(X,Z_\ell)$ and $(Z_\ell,Y)$ while the LM is frozen. Define the epoch-$t$ information plane as $
\mathcal{I}^{[t]}
\;:=\;
\big\{\,\big( I^{(t)}(X;Z_\ell),\  I^{(t)}(Z_\ell;Y)\big)
\;:\; \ell=1,\dots,L \big\}\subset\mathbb{R}^2 .
 $
We say $\mathcal{I}^{[t]}$ is \emph{well-defined} if, for all $\ell$: (i) \textbf{Finite-valuedness:} $ I^{(t)}(X;Z_\ell)$ and $ I^{(t)}(Z_\ell;Y)$ are finite.\footnote{For deterministic real-valued networks, avoid infinite MI by injecting small noise into $Z_\ell$ or applying a fixed quantizer.} (ii) \textbf{Layerwise indexability:} Each point is associated with its layer index $\ell$ (ties in coordinates are allowed). (iii) \textbf{Temporal consistency:} Across $t$, the same estimator architecture/hyperparameters and the same $p(x,y)$ are used, so $\{\mathcal{I}^{[t]}\}_{t\ge0}$ is a well-defined sequence. (iv)  \textbf{Differentiability:} The maps driving $ I^{(t)}$ are a.e.\ differentiable in their inputs so that gradients exist when backpropagating through $Z_\ell$.

\end{definition}

\begin{remark}[Dynamics]
Empirical “fitting” (both $ I(X;Z_\ell)$ and $ I(Z_\ell;Y)$ rise) and “compression” ( $ I(X;Z_\ell)$ decreases while $ I(Z_\ell;Y)$ continues to rise) patterns are diagnostic and not required for well-definedness.
\end{remark}

Recent work has used the IB to improve training
\citep{alemi2016deep,nguyen2017layer,achille2018information} and to visualize
training dynamics \citep{shwartz2017opening,cheng2019utilizing}.
Applying IB to language models remains challenging: layer
representations are high-dimensional, MI estimation is expensive. Very recent work applies IB to LMs but is largely descriptive such as explaining the model behavior  \citep{wang2025rethinking,wu2025interpreting}, attribution-focused studies \citep{jiang2020inserting}, in\mbox{-}context learning \citep{yang2025exploring}, and pruning-oriented work \citep{fan2021layer} which limits to estimate empirical MI of a layer between input-layer and layer-output pairs. However, to test our claim empirically, we require a \emph{joint} empirical assessment that captures a layer’s information\mbox{-}carrying capacity—how much information it preserves from the input and how much it conveys to the target at a time which helps to show bidirectional layers exhibit higher joint information capacity than unidirectional layers.

We estimate mutual information with MINE \citep{belghazi2018mutual}, which
provides a \emph{lower bound} on the true MI.\footnote{MINE learns a critic; with finite data and limited capacity it underestimates MI.} For a layer $Z_\ell$, MINE can compute either $I(X;Z_\ell)$ or $I(Z_\ell;Y)$. But we are interested in finding both information \emph{simultaneously} so that we can determine the capacity of information carried by $Z_\ell$ of both $X$ and $Y$. This estimation helps us estimate MI to input and target by a bidirectional or unidirectional layer and enables easy comparison. To make this happen,  we introduce \textbf{FlowNIB}, a simple modification of MINE that jointly approximates \(I(X;Z_\ell)\) and \(I(Z_\ell;Y)\) within a single objective. FlowNIB trains two critics with a schedule \(\alpha(t)\) that initially emphasizes \(I(X;Z_\ell)\) and gradually shifts toward \(I(Z_\ell;Y)\) over \(T\) epochs, tracing the layer’s information–flow trajectory: $
\big\{\,\big( I^{(t)}(X;Z_\ell),\  I^{(t)}(Z_\ell;Y)\big)
\;:\; t=1,\dots,T \big\}\subset\mathbb{R}^2 $
(See Sec.~\ref{se:flownib} for details). Our interest is to choose one point where both $I(X;Z_\ell)$ and $I(Z_\ell;Y)$ maximize at a $t \in T$; we call it the Optimal Information Coordinate (OIC).

\begin{definition}[Optimal Information Coordinate (OIC)]
Let each epoch \(t\in\{0,\dots,T\}\) yield
\(x_t= I^{(t)}(X;Z_\ell)\) and \(y_t= I^{(t)}(Z_\ell;Y)\).
For a trade-off weight \(\gamma \in[0,1]\), we define OIC for layer $\ell \in L$
\[
t^\ast(\gamma)\in\arg\max_{t} \ \gamma\,x_t + (1-\gamma)\,y_t,
\qquad
\mathrm{OIC}_\gamma := \big(x_{t^\ast(\gamma)},\,y_{t^\ast(\gamma)}\big).
\]
A scale-balanced choice is \(\gamma^\star=\tfrac{R_y}{R_x+R_y}\), where
\(R_x=\max_t x_t-\min_t x_t\) and \(R_y=\max_t y_t-\min_t y_t\).
\end{definition}

We then compare OICs after fine-tuning on the same dataset between bidirectional and unidirectional LMs to see which carries more information for both input and output. In Figure~\ref{fig:info_tra}, we see the bidirectional LM has a higher OIC than the unidirectional LM. Beyond the theoretical explanation, we empirically compare OICs using \emph{FlowNIB} across diverse datasets and show clear benefits for downstream tasks. In particular, on standard benchmarks such as GLUE, commonsense reasoning, and regression tasks, a small bidirectional model outperforms a larger unidirectional model.

\textbf{Contributions.} (i) We provide a theoretical explanation for why bidirectional language models achieve better context understanding, showing that they can carry higher mutual information than unidirectional models. (ii) To estimate mutual information in high\mbox{-}dimensional LLM representations, we propose \emph{FlowNIB}, a simple and testable framework that jointly estimates $I(X;Z_\ell)$ and $I(Z_\ell;Y)$, quantifying the information capacity of $Z_\ell$. (iii) Empirically, on NLU benchmarks, bidirectional models outperform unidirectional models, and \emph{FlowNIB} confirms that they attain higher $I(X;Z_\ell)$ and $I(Z_\ell;Y)$ across layers.

\section{Methodology}

Unidirectional language models, such as GPT, construct each hidden representation using only left-to-right context \citep{allal2024SmolLM}. In contrast, bidirectional models like BERT encode each token using both past and future context \citep{he2020deberta,liu2019roberta}. This architectural asymmetry raises a natural question: can bidirectional representations carry more information?

Let $X=(x_1,\dots,x_n)$ denote the input sequence. For layer $\ell$, let
$Z_\ell^{\rightarrow}=(z_1^{\rightarrow},\dots,z_n^{\rightarrow})$ be the
forward (causal) representations, where $z_t^{\rightarrow}$ depends only on
$x_{\le t}$. Let $Z_\ell^{\leftarrow}=(z_1^{\leftarrow},\dots,z_n^{\leftarrow})$
be the backward (anti\mbox{-}causal) representations, where $z_t^{\leftarrow}$
depends only on $x_{\ge t}$. A unidirectional model uses $Z_\ell^{\rightarrow}$, whereas a bidirectional model augments this with $Z_\ell^{\leftarrow}$ and forms the full bidirectional representation $
Z_\ell^{\leftrightarrow} \;=\; \big(Z_\ell^{\rightarrow},\,Z_\ell^{\leftarrow}\big)
\quad\text{(e.g., by concatenation or another fusion).}$
We measure representational quality via mutual information: $
I(X; Z) \;=\; H(X) - H(X \mid Z),$
where $H(X \mid Z)$ is the conditional entropy of the input given $Z$. Because $Z_\ell^{\leftrightarrow}$ includes strictly more context than $Z_\ell^{\rightarrow}$, it can reduce uncertainty about $X$ more effectively. This follows from the monotonicity of conditional entropy: conditioning on more information reduces entropy (Theorem~\ref{thm:monotonicity_conditional_entropy}). Therefore, bidirectional models produce latent representations that retain at least as much (often strictly more) information about the input sequence.

\begin{theorem}[Full version in Appendix~\ref{thm:bidirectional_mi}]
\label{thm:bi-repr-mi}
Bidirectional representations preserve more mutual information about the input and the output: $
I(X; Z_\ell^{\leftrightarrow}) \;\ge\; I(X; Z_\ell^{\rightarrow})
\text{ and }
I(Z_\ell^{\leftrightarrow}; Y) \;\ge\; I(Z_\ell^{\rightarrow}; Y).$
\end{theorem}

While mutual information quantifies how much information a representation $Z_\ell$ preserves about the input or the target, it does not describe the internal structure or complexity of that representation. To complement MI, we analyze the spectral properties of $Z_\ell$ via \emph{effective dimensionality}, which captures how many orthogonal directions in representation space carry significant variance. This helps characterize how richly each layer encodes information.

\begin{definition}[Generalized Effective Dimensionality]
\label{def:generalized_deff}
Let $\Sigma_{Z_\ell}=\mathrm{Cov}(Z_\ell)$ and let $\lambda_1,\dots,\lambda_n$
be its nonzero eigenvalues, where $n=\mathrm{rank}(\Sigma_{Z_\ell})$.
Define the normalized spectrum $p_i := \lambda_i / \sum_{j=1}^n \lambda_j$.
The generalized effective dimensionality of $Z_\ell$ under a measure $\mathcal{M}(p)$ is $
d_{\mathrm{eff}}(Z_\ell;\mathcal{M}) := \exp\big(\mathcal{M}(p)\big), $
where $\mathcal{M}(p)$ satisfies:
(i) \textbf{nonnegativity:} $\mathcal{M}(p)\ge 0$;
(ii) \textbf{maximality:} $\mathcal{M}(p)\le \log n$, with equality iff $p_i=1/n$;
(iii) \textbf{Schur\mbox{-}concavity:} if $p' \succ p$ then $\mathcal{M}(p') \le \mathcal{M}(p)$.
\end{definition}

\textit{Examples.}
(1) \textbf{Shannon entropy:} $\mathcal{M}(p)= -\sum_{i=1}^n p_i \log p_i$ yields
$d_{\mathrm{eff}}(Z_\ell)=\exp(H(p))$ \citep{roy2007effective}.
(2) \textbf{$\ell_2$ participation ratio:} $\mathcal{M}(p)=\log\!\big(1/\sum_{i=1}^n p_i^2\big)$
gives $d_{\mathrm{eff}}(Z_\ell)=(\sum_{i=1}^n \lambda_i)^2/\sum_{i=1}^n \lambda_i^2$.
Unless otherwise stated, we adopt the $\ell_2$ version as the default.
The effect of alternative measures is explored in Appendix~\ref{app:ablation_effective}.

\begin{lemma}[Bidirectional Representations Exhibit Higher Spectral Complexity]
\label{lem:effective_dim}
Let $Z_\ell^{\rightarrow}\!\in\mathbb{R}^D$ denote the unidirectional representation and $Z_\ell^{\leftrightarrow}:=(Z_\ell^{\rightarrow},Z_\ell^{\leftarrow})\in\mathbb{R}^{2D}$ the concatenated bidirectional representation of an input $X$.
If $\mathrm{Cov}(Z_\ell^{\leftarrow},Z_\ell^{\rightarrow})$ is nonsingular, then $
d_{\mathrm{eff}}(Z_\ell^{\leftrightarrow};\mathcal{M}) \;\ge\; d_{\mathrm{eff}}(Z_\ell^{\rightarrow};\mathcal{M}), $
with equality iff $Z_\ell^{\leftarrow}$ is conditionally redundant given $Z_\ell^{\rightarrow}$, i.e., $\mathrm{Cov}(Z_\ell^{\leftarrow}\mid Z_\ell^{\rightarrow})=0$.
\end{lemma}

\noindent
See Appendix~\ref{app:lemma_deff} for the proof and Appendix~\ref{ab:deff} for an ablation.

\begin{tcolorbox}[
    enhanced,
    breakable,
    colback=orange!10,
    colframe=orange!30,
    boxrule=0.5pt,
    arc=1mm,
    left=4pt,right=4pt,top=4pt,bottom=4pt,
    title={\faLightbulb\hspace{4pt}Key Finding},
    fonttitle=\bfseries\sffamily,
    coltitle=black
]
Bidirectional representations retain at least as much (and typically strictly more) mutual information about the input than unidirectional representations. They also exhibit higher effective dimensionality throughout depth, reflecting richer and more expressive latent spaces.
\end{tcolorbox}

\paragraph{FlowNIB.}
\label{se:flownib}
For empirical validation of Theorem~\ref{thm:bidirectional_mi}, we use \textbf{FlowNIB}. After fine\mbox{-}tuning the LM on a dataset, we approximate the mutual information of every layer, quantifying how much information a layer carries about the input and the target. FlowNIB is simple: it trains two MINE critics under a single objective with a time\mbox{-}varying weight:
\begin{equation}
\mathcal{L}_\ell(t)
\;=\;
-\Big(\alpha(t)\,I(X;Z_\ell)\;+\;\big(1-\alpha(t)\big)\,I(Z_\ell;Y)\Big).
\label{eq:flownib}
\end{equation}

Here \(\alpha(t):\{0,\dots,T\}\!\to\![0,1]\) is a discrete, monotonically non\mbox{-}increasing schedule. We use \(\alpha(0)=1\) and $
\alpha(t{+}1)\;=\;\max\!\big\{0,\ \alpha(t)-\delta\big\}, $
where \(\delta>0\) is a small step (e.g., \(\delta=0.001\)); if \(T\) is small, a larger \(\delta\) ensures the schedule traverses \([1,0]\) within \(T\) steps (see Appendix~\ref{ab:delta} for an ablation on the effect of \(\delta\)). Early in training (\(\alpha\!\approx\!1\)) the loss emphasizes \(I(X;Z_\ell)\); as \(\alpha(t)\) decreases, the emphasis shifts toward \(I(Z_\ell;Y)\). At each step \(t\), we record the information\mbox{-}plane coordinate \(\big(I^{(t)}(X;Z_\ell),\,I^{(t)}(Z_\ell;Y)\big)\). During training, we optionally normalize \(I(X;Z_\ell)\) by the per\mbox{-}layer effective dimension \(d_{\mathrm{eff}}(Z_\ell)\) and \(I(Z_\ell;Y)\) by \(d_{\mathrm{eff}}(Y)\) to reduce scale effects. This normalization is used only for optimization, not for reporting.
Figure~\ref{fig:layer_dff}(a) shows a simple pattern: the \emph{effective dimension} depends on how large the output space is. If the input is fixed and the label \(Y\) has only a few possible values (low dimensional), then \(d_{\mathrm{eff}}(Z_\ell)\) starts at a moderate level and usually \emph{drops} as we go deeper—because the task does not need much information. When \(Y\) has many possible values (high dimensional), the network needs to keep more information, so \(d_{\mathrm{eff}}(Z_\ell)\) increases accordingly.

The same trend appears in Figure~\ref{fig:compare_params} for mutual information. With low\mbox{-}dimensional \(Y\), \(I(X;Z)\) typically \emph{decreases} across layers (the model throws away input details that are not needed), while \(I(Z;Y)\) increases only \emph{slightly}. As the dimensionality of \(Y\) grows, \(I(X;Z)\) still tends to decrease with depth (often from a higher starting point), but \(I(Z;Y)\) \emph{rises more strongly} and may saturate later, reflecting the harder alignment with a larger label space.

These observations clarify the scale imbalance in Figure~\ref{fig:MI_result}. On GLUE (labels 1–3), \(I(X;Z)\) often looks much larger than \(I(Z;Y)\) simply because the label space is small. Without any rescaling, the larger\mbox{-}magnitude term can dominate the FlowNIB objective. Since effective dimension correlates with how much mutual information is attainable, dividing by \(d_{\mathrm{eff}}(\cdot)\) provides a simple, task\mbox{-}aware normalization that balances the two terms during optimization (Details in Proposition~\ref{prop:effdim}, Ablation~\ref{ab:deff}, \ref{ab:dimVsout}, and \ref{ab:impact_dim}).

Over all epochs \(t=0,\dots,T\), we then select the OIC for each layer, which summarizes the layer’s capacity to jointly capture information about the input and the target.

In Practice. (i) Fine\mbox{-}tune the LM on a dataset with inputs \(X\) and targets \(Y\). (ii) Run the model once to cache \((X,Y,Z_\ell)\) for all layers \(\ell\). (iii) For each \(\ell\), fit two critics on this fixed cache—one for \(I(X;Z_\ell)\) and one for \(I(Z_\ell;Y)\)—using the same neural MI setup (iv) Train the critics by minimizing~\eqref{eq:flownib} with the schedule \(\alpha(t)\). (v) Compute the OICs. We report these as \emph{relative} measurements (e.g., for OIC selection) rather than absolute MI values.\footnote{All MI numbers are neural lower\mbox{-}bound estimates with fixed hyperparameters across layers and models; no additional noise or quantization is added.} Full details are in Appendix~\ref{app:FlowNIB}.

\section{Experiments} \label{sec:expr}

This section presents empirical evidence for our theoretical finding. We conduct two complementary evaluations. First, after fine-tuning each model on a dataset, we apply FlowNIB to every layer \(\ell\) to obtain the per-epoch coordinates \(\big(I^{(t)}(X;Z_\ell),\,I^{(t)}(Z_\ell;Y)\big)\). For each layer we then select the OIC to summarize its joint ability to retain input information and align with the target; comparing OICs across layers, we want to show that bidirectional LMs consistently achieve higher information than unidirectional LMs. Second, because large bidirectional LMs are limited, we perform downstream fine-tuning under a matched parameter budget (\(\le\)600M parameters) on both classification and regression benchmarks, and compare task performance to test whether the information advantage translates into end-task gains. To ensure a fair comparison, all models use identical data splits, training budgets, and a common PEFT recipe, RoCoFT \citep{kowsher2024rocoft}, which updates a small subset of existing weight rows without introducing new adapter parameters (we update three rows per linear layer). This setup is closer to full fine\mbox{-}tuning in parameterization while preserving pretrained information and keeping the fine\mbox{-}tuning footprint comparable across architectures. In contrast, adapter\mbox{-}based PEFT methods add new parameters that can confound comparisons. Additional results with LoRA appear in Appendix~Table~\ref{tab:performance_lora}. For FlowNIB, we report relative MI quantities (for OIC selection and comparison) using the same estimator architecture, batch size, negative sampling scheme, optimizer, and training steps across layers and models; absolute MI numbers are not the focus.

\begin{figure}[htbp]
\begin{center}
    \includegraphics[width=1.00\linewidth]{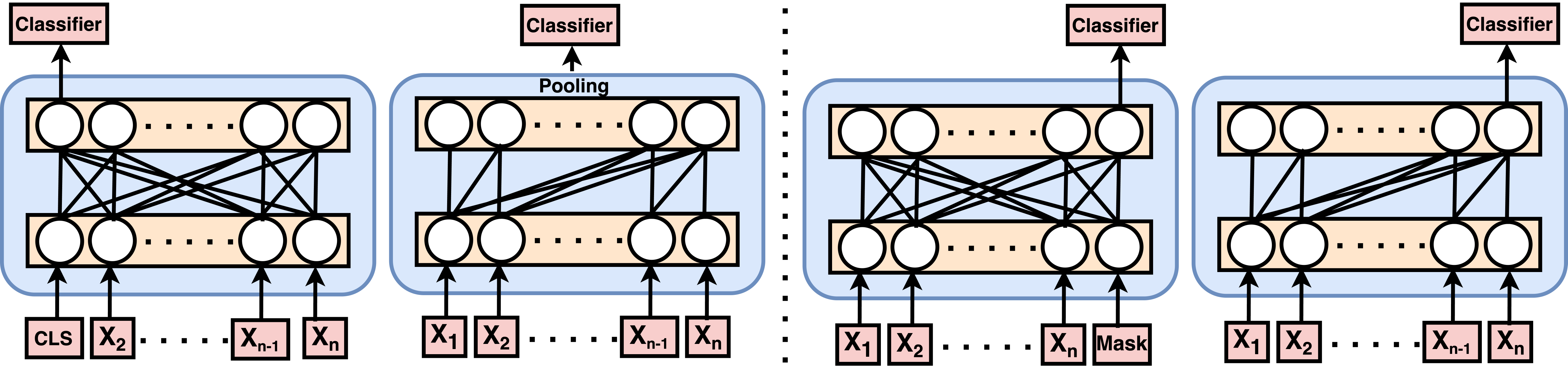}
   \end{center}
\caption{
Illustration of representation extraction methods: (a) prediction from CLS-token (bidirectional), (b) prediction from pooled embedding (unidirectional), (c) prediction from masked token (bidirectional), and (d) prediction from next-token generation (unidirectional).
}\label{fig:method_cmp}
\end{figure}

\paragraph{Model framework.} \label{sec:bidirectional-mi}
While standard approaches apply a pooling operation over the final hidden states followed by a classifier, we adopt an alternative strategy inspired by the PredGen framework \citep{kowsher2025predicting}. Instead of pooling, PredGen follows the native behavior of LMs—e.g., masked prediction or next\mbox{-}token generation—for prediction tasks. PredGen demonstrates that leveraging the model’s generative or masking capability, rather than relying solely on pooled representations, retains higher mutual information with the input and improves prediction quality. However, a key limitation of PredGen is the increased computational cost of multi\mbox{-}token generation, especially for regression\mbox{-}type tasks.

To address this, we modify this framework into a \emph{single\mbox{-}token generation or masked prediction}, as illustrated in Figure~\ref{fig:method_cmp} (right). Specifically, the model predicts a single masked token at a designated position, from which we extract the corresponding final hidden state. This representation is then passed through a lightweight MLP classifier. In Table~\ref{tab:predgen_regression}, we compare single\mbox{-}token prediction with PredGen across diverse datasets; see Appendix~\ref{sec:predgenvsone} for details.

In short, we focus on answering the following three research questions: (i) Do bidirectional models preserve more useful information than unidirectional models? (ii) Does higher mutual information lead to better context modeling? (iii) Does predicting a single token (e.g., masked token or next token) lead to better performance than traditional methods?

\begin{tcolorbox}[
    enhanced,
    breakable,
    colback=orange!10,       
    colframe=orange!30,      
    boxrule=0.5pt,           
    arc=1mm,                 
    left=4pt,right=4pt,top=4pt,bottom=4pt, 
    title={\faLightbulb\hspace{4pt}Key Finding},
    fonttitle=\bfseries\sffamily,
    coltitle=black
]
We illustrate a simplified variant of the PredGen framework that replaces multi-token generation with single-token generation or masked prediction. This approach achieves comparable performance to PredGen while substantially reducing inference cost and training complexity. See Appendix~Table\ref{tab:predgen_regression} for the comparison between single token-based prediction and PredGen.
\end{tcolorbox}

\begin{table}[htbp]
\centering
\resizebox{\textwidth}{!}{%
\begin{tabular}{l|l|l|l|l|l|l|l|l}
\hline
\rowcolor{gray!20} 
\textbf{Model} & \textbf{Layer} & \textbf{\#Heads} & \textbf{Embedding Dim} & \textbf{Max Length} & \textbf{Vocab Size} & \textbf{Total Params} & \textbf{FLOPs} & \textbf{MACs} \\
\hline
\rowcolor[HTML]{EAF3FA} \textbf{ModernBERT-base} & 22 & 12 & 768 & 8192 & 50368 & 149M & 28.258 & 14.118 \\
\rowcolor[HTML]{EAF3FA} \textbf{ModernBERT-large} & 28 & 16 & 1024 & 8192 & 50368 & 395M & 87.883 & 43.923 \\
\rowcolor[HTML]{EAF3FA} \textbf{RoBERTa-base} & 12 & 12 & 768 & 514 & 50265 & 125M & 21.760 & 10.870 \\
\rowcolor[HTML]{EAF3FA} \textbf{RoBERTa-large} & 24 & 16 & 1024 & 514 & 50265 & 355M & 77.344 & 38.656 \\
\rowcolor[HTML]{EAF3FA} \textbf{DeBERTa-v3-base} & 12 & 12 & 768 & 512 & 128100 & 184M & 39.275 & 19.629 \\
\rowcolor[HTML]{EAF3FA} \textbf{DeBERTa-v3-large} & 24 & 16 & 1024 & 512 & 128100 & 435M & 136.943 & 68.451 \\
\rowcolor[HTML]{F8DADA} \textbf{GPT2-small} & 12 & 12 & 768 & 1024 & 50257 & 117M & 21.756 & 10.872 \\
\rowcolor[HTML]{F8DADA} \textbf{GPT2-medium} & 24 & 16 & 1024 & 1024 & 50257 & 345M & 77.342 & 38.655 \\
\rowcolor[HTML]{F8DADA} \textbf{GPT2-large} & 36 & 20 & 1280 & 1024 & 50257 & 762M & 181.254 & 90.597 \\
\rowcolor[HTML]{F8DADA} \textbf{SmolLM-135M} & 30 & 9 & 576 & 2048 & 49152 & 135M & 27.185 & 13.590 \\
\rowcolor[HTML]{F8DADA} \textbf{SmolLM-360M} & 32 & 15 & 960 & 2048 & 49152 & 360M & 80.541 & 40.265 \\
\rowcolor[HTML]{F8DADA} \textbf{MobileLLM-125M} & 30 & 9 & 576 & 2048 & 32000 & 125M & 31.900 & 15.950 \\
\rowcolor[HTML]{F8DADA} \textbf{MobileLLM-600M} & 40 & 18 & 1152 & 2048 & 32000 & 600M & 154.408 & 77.196 \\
\hline
\end{tabular}%
}
\caption{Overview of bidirectional (top) and unidirectional (bottom) model architectures evaluated in our experiments, including FLOPs and MACs.}
\label{tab:models_cmp}
\end{table}
\begin{figure}[htbp]

\begin{center}
\includegraphics[width=1.00\linewidth]{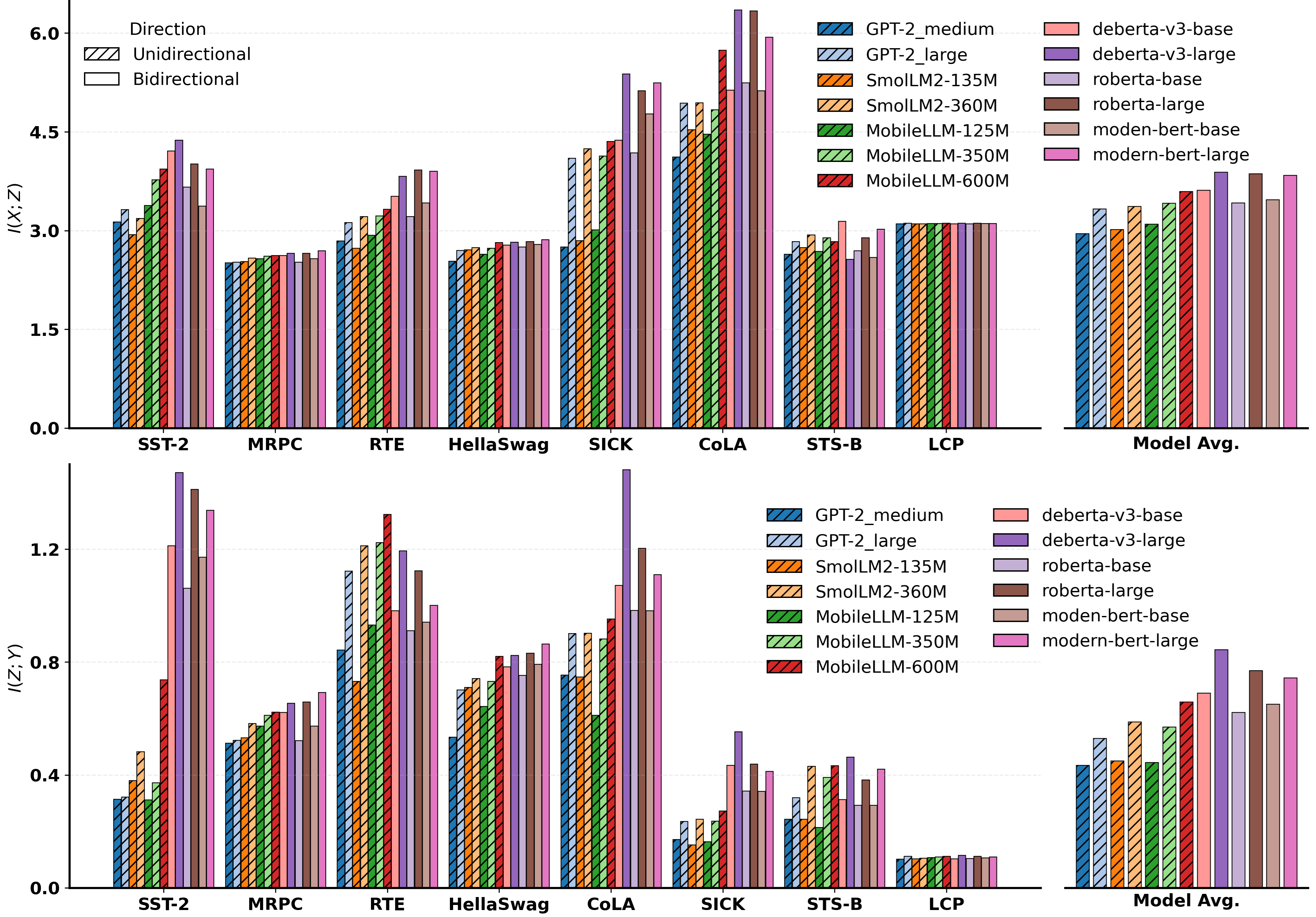}
   \end{center}
\caption{
Average OIC \(I(X;Z)\) (top) and \(I(Z;Y)\) (bottom) across all layers for unidirectional  and bidirectional LMs over multiple datasets. Bars show dataset-wise and average values, comparing information flow differences between architectures.
}
\label{fig:MI_result}
\end{figure}

\begin{figure}[htbp]
\begin{center}
    \includegraphics[width=1.0\linewidth]{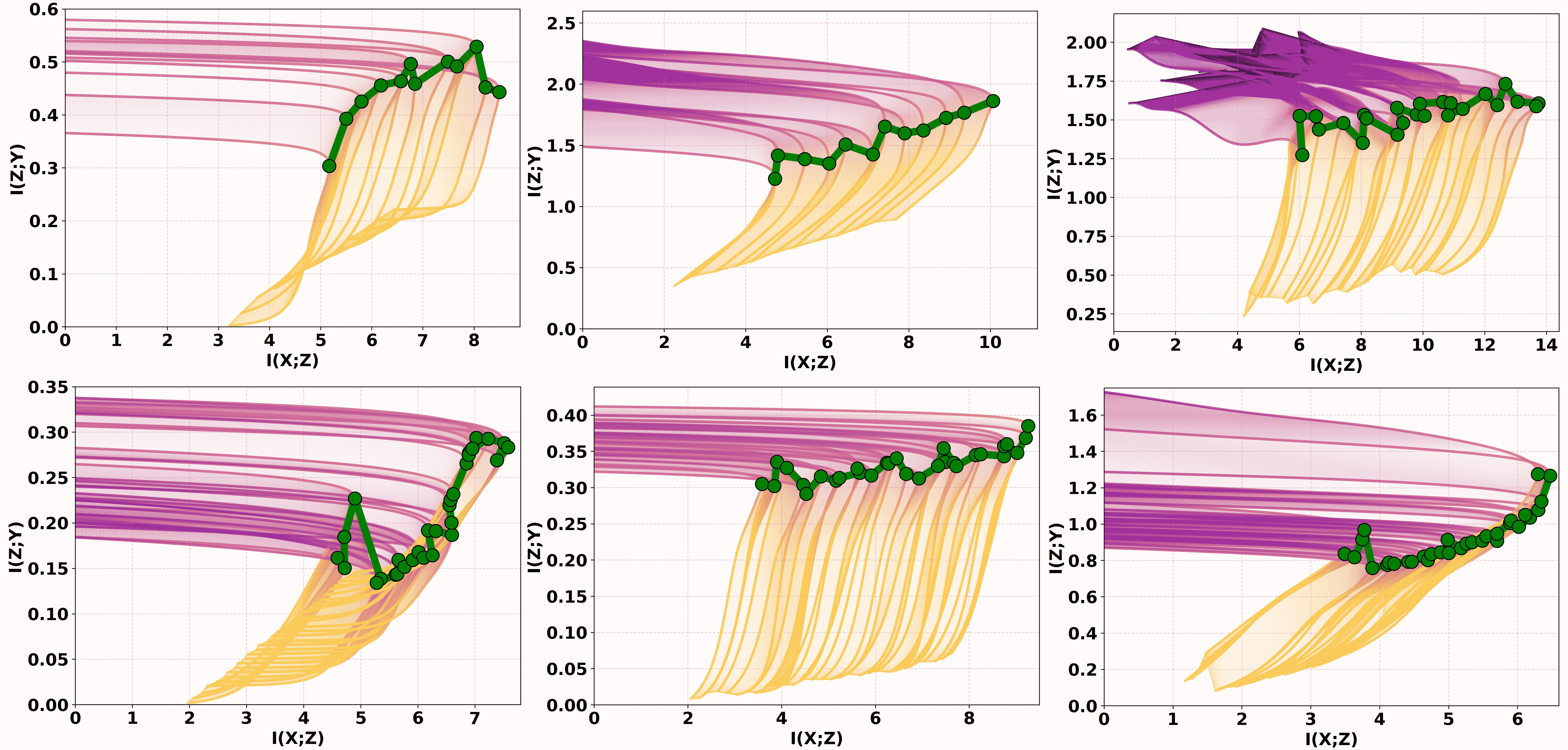}
\end{center}
\caption{
Mutual information flow comparison between bidirectional (top) and unidirectional (bottom) models across three datasets. The first column shows results on the SICK dataset using DeBERTa-base and MobileLLM-350M. The second column shows SST-2 results using RoBERTa-base and MobileLLM-350M. The third column presents results on the CoLA dataset using DeBERTa-v3-Large and MobileLLM-600M.
}
\label{fig:flownib_example}
\end{figure}
\paragraph{Datasets:} We evaluate our models across 16 diverse NLP datasets spanning classification and regression tasks to ensure a comprehensive analysis of representational learning under the information bottleneck framework. For classification, we include \textbf{SST-2}, \textbf{MRPC}, \textbf{QNLI}, \textbf{RTE}, \textbf{MNLI}, and \textbf{CoLA} from the GLUE benchmark~\citep{wang2018glue}, as well as \textbf{BoolQ}~\citep{clark2019boolq}, \textbf{HellaSwag}~\citep{zellers2019hellaswag}, and \textbf{SocialIQA}~\citep{sap2019socialiqa}, covering a range of linguistic challenges such as sentiment analysis, natural language inference, grammatical acceptability, question answering, and commonsense reasoning. The regression tasks comprise \textbf{STS-B}~\citep{cer2017semeval}, \textbf{SICK}~\citep{marelli-etal-2014-sick}, \textbf{WASSA}~\citep{vinayakumar2017deepcybernet}, \textbf{LCP}~\citep{shardlow2020complex}, \textbf{CRP}~\citep{shardlow2020complex}, and \textbf{Humicroedit}~\citep{hossain2019president}, addressing semantic textual similarity, lexical complexity prediction, and humor detection. Dataset sizes range from approximately 2,500 to 400,000 examples, with either binary or multi-class classification labels, or continuous-valued targets for regression. We exclude generation-based tasks because bidirectional language models are not designed for auto-regressive generation; instead, we focus on tasks requiring strong contextual representations to assess representational sufficiency under the information bottleneck. Additional dataset statistics are provided in Table~\ref{tab:datasets} in the Appendix.
In addition, the details of used models architecture, hyperparameters, evaluation metrics, and environment setup are provided in Appendix~\ref{app:model_description}, Appendix~\ref{app:hyper}, Appendix~\ref{app:evl_metric}, and Appendix~\ref{app:env_set}, respectively.

\begin{table*}[htbp]
\centering
\resizebox{\textwidth}{!}{%
\begin{tabular}{l|l|c|c|c|c|c|c|c|c|c|c}
\hline
\rowcolor{gray!20}
\textbf{Model} & \textbf{Method} & \textbf{SST-2} & \textbf{MRPC} & \textbf{QNLI} & \textbf{RTE} & \textbf{CoLA} & \textbf{MNLI} & \textbf{BoolQ} & \textbf{HellaSwag} & \textbf{SIQA} & \textbf{Avg.} \\
\hline
\rowcolor[HTML]{EAF3FA} DeBERTa-v3-Base & Pooling & 95.52 & 89.21 & 92.43 & 83.48 & 86.23 & 86.43 & 64.23 & 56.00 & 47.54 & 77.90 \\
\rowcolor[HTML]{EAF3FA}  & Masking & 95.75 & 91.17 & 92.48 & 84.98 & 87.44 & 87.22 & 64.23 & 69.49 & 60.90 & 81.52\\
\rowcolor[HTML]{EAF3FA} DeBERTa-v3-Large & Pooling & 95.67 & 93.45 & 93.58 & 88.38 & 93.34 & 90.76 & 64.73 & 57.34 & 51.43 & 80.96\\
\rowcolor[HTML]{EAF3FA}  & Masking & 96.11 & 94.04 & 94.14 & 89.93 & 92.95 & 91.43 & 64.98 & 73.43 & 65.53 & 84.73 \\
\rowcolor[HTML]{EAF3FA} RoBERTa-Base & Pooling & 94.24 & 84.53 & 91.96 & 83.45 & 86.34 & 86.34 & 63.82 & 52.43 & 45.64 &76.53\\
\rowcolor[HTML]{EAF3FA}  & Masking & 95.14 & 85.13 & 92.27 & 84.58 & 87.44 & 86.38 & 63.96 & 64.53 & 60.16 & 79.95\\
\rowcolor[HTML]{EAF3FA} RoBERTa-Large & Pooling & 95.68 & 89.54 & 94.17 & 86.32 & 93.85 & 90.87 & 64.82 & 57.35 & 48.69 & 80.14\\
\rowcolor[HTML]{EAF3FA}  & Masking & 96.23 & 91.25 & 94.38 & 87.84 & 95.83 & 91.13 & 63.82 & 71.43 & 63.67 & 83.95\\
\rowcolor[HTML]{EAF3FA} ModernBERT-Base & Pooling & 94.35 & 83.33 & 91.98 & 82.81 & 84.92 & 87.44 & 63.70 & 55.32 & 46.81 & 76.74\\
\rowcolor[HTML]{EAF3FA}  & Masking & 95.38 & 85.43 & 92.43 & 84.12 & 84.43 & 88.21 & 62.17 & 63.54 & 61.86 & 79.73\\
\rowcolor[HTML]{EAF3FA} ModernBERT-Large & Pooling & 95.37 & 89.43 & 94.22 & 86.74 & 89.95 & 93.23 & 64.22 & 60.32 & 49.67 & 80.35\\
\rowcolor[HTML]{EAF3FA}  & Masking & 95.89 & 89.93 & 94.57 & 87.78 & 90.79 & 92.98 & 64.72 & 73.18 & 64.68 & 83.84\\
\rowcolor[HTML]{F8DADA} GPT-2 Medium & Pooling & 93.80 & 85.78 & 91.17 & 69.67 & 80.24 & 78.81 & 63.43 & 37.83 & 38.45 & 71.02\\
\rowcolor[HTML]{F8DADA}  & Generation & 94.14 & 85.93 & 91.93 & 69.83 & 81.43 & 80.18 & 63.54 & 37.93 & 43.45 & 72.04 \\
\rowcolor[HTML]{F8DADA} GPT-2 Large & Pooling & 93.97 & 86.27 & 84.01 & 66.78 & 83.89 & 80.06 & 64.13 & 40.32 & 41.91 & 71.26\\
\rowcolor[HTML]{F8DADA}  & Generation & 94.24 & 87.23 & 84.56 & 67.34 & 83.87 &82.34 & 64.16 & 39.53 & 45.34 & 72.07\\
\rowcolor[HTML]{F8DADA} SmolLM2-135M & Pooling & 92.58 & 84.59 & 90.56 & 68.12 & 81.48 & 82.83 & 62.43 & 38.34 & 41.41 & 71.37\\
\rowcolor[HTML]{F8DADA}  & Generation & 93.00 & 84.83 & 90.68 & 68.93 & 82.48 & 83.58 & 62.27 & 41.78 & 47.86 & 72.82\\
\rowcolor[HTML]{F8DADA} SmolLM2-360M & Pooling & 94.26 & 84.80 & 91.61 & 70.70 & 82.07 & 85.12 & 63.13 & 42.45 & 42.43 & 72.95\\
\rowcolor[HTML]{F8DADA}  & Generation & 94.65 & 85.32 & 92.32 & 71.11 & 84.53 & 84.89 & 62.92 & 43.69 & 50.20 & 74.40\\
\rowcolor[HTML]{F8DADA} MobileLLM-125M & Pooling & 93.05 & 82.43 & 90.58 & 69.32 & 80.29 & 82.98 & 60.73 & 33.45 & 41.45 & 70.48\\
\rowcolor[HTML]{F8DADA}  & Generation & 93.15 & 83.35 & 90.54 & 69.53 & 80.53 & 83.24 & 61.26 & 37.42 & 48.23 & 71.92\\
\rowcolor[HTML]{F8DADA} MobileLLM-350M & Pooling & 93.85 & 83.68 & 90.85 & 70.33 & 82.38 & 83.45 & 63.42 & 36.28 & 42.74 & 71.89\\
\rowcolor[HTML]{F8DADA}  & Generation & 94.68 & 83.57 & 91.09 & 71.43 & 82.87 & 84.58 & 63.71 & 40.13 & 51.54 & 73.73\\
\rowcolor[HTML]{F8DADA} MobileLLM-600M & Pooling & 94.86 & 87.34 & 91.34 & 72.45 & 84.56 & 84.93 & 64.18 & 45.32 & 45.54 & 74.50\\
\rowcolor[HTML]{F8DADA}  & Generation & 95.14 & 87.87 & 91.37 & 72.29 & 86.30 & 84.79 & 64.12 & 48.53 & 58.54 & 76.55\\
\hline
\end{tabular}%
}
\caption{Accuracy(\%) results across nine NLP classification tasks comparing bidirectional and unidirectional models under pooling, masking, and generation inference strategies.}
\label{tab:performance}
\end{table*}

\paragraph{MI results.}
To measure layerwise information, we first fine\mbox{-}tune each model on the target dataset, then run a single pass to cache triplets \((X,Y,Z_\ell)\) for every layer \(\ell\in L\), where \(Z_\ell\) denotes the layer’s activations on \(X\). Given this fixed cache, we instantiate two identical two\mbox{-} fully connected layer (nn.Linear() in pytorch) estimator networks (same widths, nonlinearity, and initialization): one estimates \(I(X;Z_\ell)\) and the other estimates \(I(Z_\ell;Y)\). Both estimators are trained jointly under the common FlowNIB objective in Eq.~\eqref{eq:flownib} with a discrete schedule \(\alpha(t)\) that linearly decays from \(1\) to \(0\): $
\alpha(0)=1,\qquad \alpha(t{+}1)=\max\{0,\alpha(t)-\delta\},\quad \delta=0.001.$
Unless noted otherwise, we use batch size \(128\), \(T=2000\) training steps, and the same optimizer and negative\mbox{-}sampling scheme across all layers and models. At each step \(t\) we record the information\mbox{-}plane coordinate \(\big(I^{(t)}(X;Z_\ell),\,I^{(t)}(Z_\ell;Y)\big)\). After training, for each layer \(\ell\) we select its OIC from these coordinates; the OIC summarizes the layer’s capacity to jointly capture input and target information. We apply the \emph{same} estimator architecture, schedule, and hyperparameters to all bidirectional and unidirectional models, enabling a like\mbox{-}for\mbox{-}like comparison. The full procedure is given in Algorithm~\ref{alg:flownib}.

Figure~\ref{fig:MI_result} compares the \emph{average OIC} across all layers between bidirectional and unidirectional LMs. We observe that bidirectional models consistently retain higher mutual information for both \(I(X;Z)\) and \(I(Z;Y)\). Notably, even smaller bidirectional models (e.g., RoBERTa\mbox{-}base, 125M) surpass larger unidirectional models (e.g., MobileLLM\mbox{-}600M, SmolLM2\mbox{-}360M) in OIC on many datasets. To further elucidate this behavior, Figure~\ref{fig:flownib_example} visualizes the \emph{information\mbox{-}plane trajectories} layer by layer over the estimator training horizon \(T\), contrasting bidirectional and unidirectional models on multiple datasets. Across layers and epochs, bidirectional models trace trajectories with systematically higher \(I(X;Z)\) and \(I(Z;Y)\), aligning with their larger OICs. Complementarily, Figure~\ref{fig:MI_example} shows a token\mbox{-}level MI analysis from the final layer (after fine\mbox{-}tuning on SST\mbox{-}2), which further highlights the representational advantage of bidirectional models.

A common assumption is that bidirectional models are inherently more expensive—roughly twice the cost of unidirectional models. In practice, we find that \emph{smaller} bidirectional models can achieve \emph{higher OIC} while matching or even reducing compute. For example, Table~\ref{tab:models_cmp} reports that RoBERTa\mbox{-}base\mbox{-}125M requires only 21.76\,GFLOPs and 10.87\,GMACs, whereas MobileLLM\mbox{-}125M requires 31.90\,GFLOPs and 15.95\,GMACs, despite being unidirectional. Additional CPU profiling in Appendix~\ref{sec:model_profile_info} shows comparable end\mbox{-}to\mbox{-}end runtime characteristics between the two families, reinforcing that the observed information advantage of bidirectional models need not come with prohibitive compute overhead.

\begin{table*}[htbp]
\centering
\resizebox{\textwidth}{!}{%
\begin{tabular}{l|l|c|c|c|c|c|c|c}
\hline
\rowcolor{gray!20}
\textbf{Model} & \textbf{Method} & \textbf{WASSA} & \textbf{SICK} & \textbf{STSB} & \textbf{LCP} & \textbf{CRP} & \textbf{Humicroedit} & \textbf{Avg.}\\
\hline
\rowcolor[HTML]{DFFFD6} DeBERTa-v3-Base & Pooling & 0.017/0.107 & 0.163/0.297 & 0.363/0.455 & 0.007/0.076 & 0.429/0.518 & 0.278/0.432 & 0.209/0.314 \\
\rowcolor[HTML]{DFFFD6}  & Masking & 0.013/0.091 & 0.135/0.277 & 0.373/0.462 & 0.006/0.060 & 0.385/0.478 & 0.274/0.423 & 0.197/0.298\\
\rowcolor[HTML]{DFFFD6} DeBERTa-v3-Large & Pooling & 0.016/0.102 & 0.140/0.281 & 0.353/0.442& 0.007/0.073 & 0.345/0.457 & 0.263/0.419 & 0.187/0.295\\
\rowcolor[HTML]{DFFFD6}  & Masking & 0.012/0.075 & 0.132/0.274 & 0.348/0.414 & 0.005/0.051 & 0.340/0.459 & 0.268/0.421 & 0.184/0.282\\
\rowcolor[HTML]{DFFFD6} RoBERTa-Base & Pooling & 0.016/0.097 & 0.168/0.300 & 0.364/0.452 & 0.007/0.066 & 0.465/0.535 & 0.293/0.438 & 0.218/0.314\\
\rowcolor[HTML]{DFFFD6}  & Masking & 0.015/0.094 & 0.145/0.294 & 0.353/0.448 & 0.007/0.065 & 0.431/0.517 & 0.289/0.431 &  0.206/0.308\\
\rowcolor[HTML]{DFFFD6} RoBERTa-Large & Pooling & 0.015/0.097 & 0.153/0.291 & 0.351/0.439 & 0.006/0.060 & 0.376/0.469 & 0.283/0.432 & 0.197/0.298\\
\rowcolor[HTML]{DFFFD6}  & Masking & 0.016/0.099 & 0.152/0.291 & 0.350/0.429 & 0.603/0.059 & 0.366/0.475 & 0.281/0.431 & 0.294/0.297\\
\rowcolor[HTML]{DFFFD6} ModernBERT-Base & Pooling & 0.016/0.092 & 0.207/0.350 & 0.469/0.517 & 0.006/0.069 & 0.376/0.469 & 0.302/0.447 & 0.229/0.324\\
\rowcolor[HTML]{DFFFD6}  & Masking & 0.015/0.093 & 0.173/0.328 & 0.482/0.536 & 0.006/0.067 & 0.364/0.471 & 0.281/0.430 & 0.220/0.320 \\
\rowcolor[HTML]{DFFFD6} ModernBERT-Large & Pooling & 0.016/0.093 & 0.160/0.307 & 0.378/0.468 & 0.006/0.060 & 0.341/0.453 & 0.302/0.449 & 0.200/0.305\\
\rowcolor[HTML]{DFFFD6}  & Masking & 0.150/0.294 & 0.150/0.292 & 0.371/0.462 & 0.006/0.005 & 0.344/0.457 & 0.293/0.441 & 0.219/0.325\\
\rowcolor[HTML]{FFF5CC} GPT-2 Medium & Pooling & 0.019/0.112 & 0.662/0.619 & 0.427/0.499 & 0.008/0.084 & 0.369/0.476 & 0.394/0.535 & 0.313/0.387\\
\rowcolor[HTML]{FFF5CC}  & Generation & 0.018/0.111 & 0.673/0.620 & 0.412/0.490 & 0.008/0.083 & 0.345/0.457& 0.347/0.493 & 0.300/0.375 \\
\rowcolor[HTML]{FFF5CC} GPT-2 Large & Pooling & 0.018/0.105 & 0.623/0.583 & 0.442/0.522 & 0.007/0.080 & 0.324/0.443 & 0.318/0.463 & 0.288/0.366 \\
\rowcolor[HTML]{FFF5CC}  & Generation & 0.017/0.107 & 0.583/0.523 & 0.423/0.499 & 0.007/0.078 & 0.326/0.446 & 0.323/0.473 & 0.279/0.354 \\
\rowcolor[HTML]{FFF5CC} SmolLM2-135M & Pooling & 0.017/0.105 & 0.192/0.336 & 0.424/0.489 & 0.007/0.076 & 0.369/0.476 & 0.304/0.450 & 0.218/0.322\\
\rowcolor[HTML]{FFF5CC}  & Generation & 0.017/0.106 & 0.175/0.319 & 0.403/0.484 & 0.007/0.076 & 0.366/0.475 & 0.295/0.442 & 0.210/0.317 \\
\rowcolor[HTML]{FFF5CC} SmolLM2-350M & Pooling & 0.017/0.104 & 0.173/0.310 & 0.407/0.488 & 0.006/0.061 & 0.340/0.459 & 0.338/0.463 & 0.213/0.314\\
\rowcolor[HTML]{FFF5CC}  & Generation & 0.017/0.105 & 0.170/0.298 & 0.394/0.481 & 0.006/0.060 & 0.332/0.454 & 0.323/0.462 & 0.207/0.310\\
\rowcolor[HTML]{FFF5CC} MobileLLM-125M & Pooling & 0.020/0.111 & 0.197/0.354 & 0.419/0.492 & 0.006/0.070 & 0.323/0.446 & 0.302/0.451 & 0.211/0.320\\
\rowcolor[HTML]{FFF5CC}  & Generation & 0.019/0.113 & 0.192/0.324 & 0.410/0.491 & 0.006/0.068 & 0.312/0.448 & 0.293/0.442 & 0.205/0.314\\
\rowcolor[HTML]{FFF5CC} MobileLLM-350M & Pooling & 0.018/0.104 & 0.191/0.336 & 0.394/0.482 & 0.006/0.063& 0.310/0.436 & 0.282/0.431 & 0.200/0.308\\
\rowcolor[HTML]{FFF5CC}  & Generation & 0.017/0.105 & 0.187/0.320 & 0.391/0.478 & 0.006/0.063 & 0.309/0.437 & 0.278/0.421 & 0.198/0.304\\
\rowcolor[HTML]{FFF5CC} MobileLLM-600M & Pooling & 0.017/0.105& 0.181/0.320 & 0.384/0.474 & 0.006/0.063 & 0.301/0.432 & 0.274/0.421 & 0.193/0.302\\
\rowcolor[HTML]{FFF5CC}  & Generation & 0.017/0.105 & 0.172/0.318 & 0.381/0.472 & 0.006/0.063 & 0.308/0.419 & 0.278/0.438 & 0.193/0.302\\
\hline
\end{tabular}%
}
\caption{Regression results (MSE/MAE) across six NLP regression tasks comparing bidirectional and unidirectional models under pooling, masking, and generation inference strategies.}
\label{tab:regression_results}
\end{table*}

\paragraph{Main Results:} 
Our results show that bidirectional models consistently outperform unidirectional models across both classification and regression tasks (Table~\ref{tab:performance}, Table~\ref{tab:regression_results}). For example, in classification, DeBERTa-v3-Large achieves the highest average accuracy of 84.73\% using masked token prediction, improving by +3.77\% over its pooling-based variant. Furthermore, we observe that even RoBERTa-base outperforms MobileLLM-600M in several tasks, highlighting a consistent trend with mutual information (MI): better MI is correlated with improved context modeling and task performance.

Overall, these findings highlight that masking inference yields stronger gains in bidirectional models, while generation provides modest improvements for unidirectional models but fails to close the accuracy and error gap, reinforcing the advantage of bidirectional context and masking for both classification and regression.

\begin{tcolorbox}[
    enhanced,
    breakable,
    colback=orange!10,
    colframe=orange!30,
    boxrule=0.5pt,
    arc=1mm,
    left=4pt,right=4pt,top=4pt,bottom=4pt,
    title={\faLightbulb\hspace{4pt}Key Finding},
    fonttitle=\bfseries\sffamily,
    coltitle=black
]
\textbf{OIC} is strongly correlated with model performance: representations with higher OIC values—i.e., high mutual information with both the input and the output—consistently yield better downstream task accuracy.
\end{tcolorbox}

\section{Related work}
\paragraph{Information bottleneck in deep learning} 
The IB principle has been studied from both practical and theoretical perspectives in deep learning.
On the practical side, \citep{alemi2016deep, higgins2017beta, achille2018information} formulated the IB problem as a deep learning objective and introduced variational approximations to enable optimization via gradient descent. 
On the theoretical side, \citep{tishby2015deep, shwartz2017opening} provided an information-theoretic framework for understanding deep learning, establishing the IB as a foundational tool for analyzing representation learning and generalization in deep learning. These fundamental ideas have inspired a wide range of follow-up works \citep{goldfeld2020information, saxe2019information, shwartz2022information} that further investigate deep learning dynamics through the lens of information theory. 

\paragraph{Mutual information estimation}
Mutual information quantifies the statistical dependence between two random variables and plays an important role in the IB principle. However, the mutual information is notoriously difficult to estimate between continuous high-dimensional random variables. Traditional nonparametric approaches \citep{fraser1986independent,moon1995estimation,darbellay1999estimation,suzuki2008approximating,kwak2002input,kraskov2004estimating} typically are not scalable with dimension and sample size. To achieve an efficient estimator, recent work \citep{nguyen2010estimating, nowozin2016f} characterized the mutual information of two random variables with the Kullback-Leibler (KL-) divergence \citep{kullback1997information} between their joint distribution and the product of the marginals and used a dual representations to cast the KL divergence. The Mutual Information Neural Estimator (MINE) \citep{belghazi2018mutual} utilized the dual representation of the KL divergence and estimated mutual information via gradient descent over neural networks and thus scaled well. 


\section{Conclusion}
This work investigates why bidirectional models outperform unidirectional ones in natural language understanding and context modeling, combining theory with empirical evidence. We introduce \textbf{FlowNIB}, a dynamic, IB\mbox{-}based framework that tracks layer\mbox{-}wise mutual information over training. Our results show that bidirectional models retain more input information and more predictive information, yielding stronger representations and better downstream performance. FlowNIB offers a principled explanation for this advantage and suggests new directions for analyzing and improving deep language models.


\bibliography{iclr2026_conference}
\bibliographystyle{iclr2026_conference}

\appendix

\onecolumn
\tableofcontents

\clearpage

\section*{Use of Large Language Models}
We used a large language model (GPT) solely for minor writing assistance, such as grammar checking, language polishing, and improving readability. No content generation, ideation, experimental design, data analysis, or result interpretation was performed by the LLM. All research contributions, technical content, and results in this paper are entirely the work of the authors.

\section{Bidirectional vs Unidirectional Representation} \label{app:bidirectional_mi}

\begin{theorem}[Conditioning Reduces Entropy]\label{app_sec:proof_conditioning_entropy}
Let \( X \) and \( Y \) be continuous random variables with joint density \( f_{X,Y}(x,y) \), marginal densities \( f_X(x) \), \( f_Y(y) \), and conditional density \( f_{X|Y}(x|y) \). The differential entropy satisfies:

\[
H(X) \geq H(X|Y),
\]

where \( H(X) \) and \( H(X|Y) \) denote the marginal and conditional differential entropy, respectively. \citep{cover2006elements}

\end{theorem} 
\begin{proof}

For continuous random variables, differential entropy is defined as:

\begin{multline*}
H(X) = - \int f_X(x) \log f_X(x) dx, H(X|Y) = - \iint f_{X,Y}(x,y) \log f_{X|Y}(x|y) dx dy.
\end{multline*}

Substituting \( f_{X|Y}(x|y) = \frac{f_{X,Y}(x,y)}{f_Y(y)} \) into \( H(X|Y) \), we derive:

\[
H(X|Y) = - \iint f_{X,Y}(x,y) \log \frac{f_{X,Y}(x,y)}{f_Y(y)} dx dy
\]

Expanding the logarithm:
\small
\begin{multline*}
H(X|Y) = - \underbrace{\iint f_{X,Y}(x,y) \log f_{X,Y}(x,y) \,dxdy}_{H(X,Y)} + \iint f_{X,Y}(x,y) \log f_Y(y) \,dxdy.
\end{multline*}
\normalsize

The second term simplifies using the marginal \( \int f_{X,Y}(x,y) dx = f_Y(y) \):

\begin{multline*}
\iint f_{X,Y}(x,y) \log f_Y(y) dx dy = \int f_Y(y) \log f_Y(y) dy = - H(Y).
\end{multline*}

Thus,

\[
H(X|Y) = H(X,Y) - H(Y).
\]

To show \( H(X) \geq H(X|Y) \), we invoke the non-negativity of the Kullback-Leibler (KL) divergence:

\begin{multline*}
D_{\text{KL}}(f_{X,Y} \| f_X f_Y) =  f_{X,Y}(x,y) \log \frac{f_{X,Y}(x,y)}{f_X(x) f_Y(y)} dx dy \geq 0.
\end{multline*}

Expanding the integrand:

\begin{multline*}
D_{\text{KL}} = f_{X,Y}(x,y) \log f_{X,Y}(x,y) dx dy
-  f_{X,Y}(x,y) \log f_X(x) dx dy - f_{X,Y}(x,y) \log f_Y(y) dx dy.
\end{multline*}

Recognizing the entropy terms:

\begin{multline*}
D_{\text{KL}} = - H(X,Y) + H(X) + H(Y) \geq 0 
\implies H(X) + H(Y) \geq H(X,Y).
\end{multline*}

Substituting \( H(X,Y) = H(X|Y) + H(Y) \) into the inequality:

\[
H(X) \geq H(X|Y).
\]

\end{proof}

\begin{theorem}[Monotonicity of Conditional Entropy]
\label{thm:monotonicity_conditional_entropy}
Let \( X, Y, Z \) be continuous random variables. Then the differential entropy satisfies:
\[
H(X \mid Y) \geq H(X \mid Y, Z),
\]
with equality if and only if \( X \perp Z \mid Y \). More generally, for any sequence \( Y_1, \dots, Y_n \),
\[
H(X \mid Y_1) \geq H(X \mid Y_1, Y_2) \geq \cdots \geq H(X \mid Y_1, \dots, Y_n).
\]
\end{theorem}

\begin{proof}
We begin with the definition of conditional differential entropy:
\[
H(X \mid Y) = - \iint f_{X,Y}(x,y) \log f_{X \mid Y}(x \mid y) \, dx\,dy,
\]
\[
H(X \mid Y,Z) = -  f_{X,Y,Z}(x,y,z) \log f_{X \mid Y,Z}(x \mid y,z) \, dx\,dy\,dz.
\]

Recall that:
\[
f_{X \mid Y}(x \mid y) = \int f_{X \mid Y,Z}(x \mid y, z) f_{Z \mid Y}(z \mid y) \, dz.
\]

Now apply Jensen’s inequality using the convexity of \( -\log(\cdot) \):
\[
-\log \left( \int f_{X \mid Y,Z}(x \mid y, z) f_{Z \mid Y}(z \mid y) \, dz \right)
\leq -\int f_{Z \mid Y}(z \mid y) \log f_{X \mid Y,Z}(x \mid y, z) \, dz.
\]

Multiplying both sides by \( f_{X \mid Y}(x \mid y) \) and integrating over \( x, y \), we obtain:
\begin{align*}
H(X \mid Y) &= - \iint f_{X,Y}(x,y) \log f_{X \mid Y}(x \mid y) \, dx\,dy \\
&\geq -  f_{X,Y,Z}(x,y,z) \log f_{X \mid Y,Z}(x \mid y, z) \, dx\,dy\,dz \\
&= H(X \mid Y,Z).
\end{align*}

Equality holds iff Jensen’s inequality becomes an equality, which occurs if and only if
\[
f_{X \mid Y,Z}(x \mid y, z) = f_{X \mid Y}(x \mid y)
\quad \text{a.e. in } z,
\]
i.e., \( X \perp Z \mid Y \).

For the generalization, apply this result inductively:
\[
H(X \mid Y_1) \geq H(X \mid Y_1, Y_2) \geq \cdots \geq H(X \mid Y_1, \dots, Y_n).
\]

\end{proof}

\begin{theorem}[Bidirectional Representations Preserve More Mutual Information]
\label{thm:bidirectional_mi}
Let \(X\) denote a sequence input \(x_1, x_2, \dots, x_n\).  
Let \(Z_\ell^{\rightarrow}\) denote the unidirectional hidden representation constructed of layer $\ell$ from the forward context:
\[
Z_\ell^{\rightarrow} = (z_1^{\rightarrow}, z_2^{\rightarrow}, \dots, z_n^{\rightarrow}) \quad \text{with } z_t^{\rightarrow} = f(x_1,\dots,x_t),
\]
and \(Z_\ell^{\leftarrow}\) the backward representation:
\[
Z_\ell^{\leftarrow} = (z_1^{\leftarrow}, z_2^{\leftarrow}, \dots, z_n^{\leftarrow}) \quad \text{with } z_t^{\leftarrow} = g(x_t,\dots,x_n).
\]
Let the bidirectional representation be:
\[
Z_\ell^{\leftrightarrow} = (Z_\ell^{\rightarrow}, Z_\ell^{\leftarrow}).
\]

Then the mutual information between \(X\) and the bidirectional representation satisfies:
\[
I(X; Z_\ell^{\leftrightarrow}) \geq I(X; Z_\ell^{\rightarrow}),
\]
with equality if and only if \(Z_\ell^{\leftarrow} \perp X \mid Z_\ell^{\rightarrow}\).

\end{theorem}

---

\begin{proof}

We begin with the identity:
\[
I(X; Z) = H(X) - H(X \mid Z).
\]

Apply this to both representations:
\[
I(X; Z_\ell^{\rightarrow}) = H(X) - H(X \mid Z_\ell^{\rightarrow}),
\]
\[
I(X; Z_\ell^{\leftrightarrow}) = H(X) - H(X \mid Z_\ell^{\rightarrow}, Z_\ell^{\leftarrow}).
\]

Since \(Z_\ell^{\leftrightarrow}\) contains strictly more information than \(Z_\ell^{\rightarrow}\), we can invoke the \textit{monotonicity of conditional entropy} \ref{thm:monotonicity_conditional_entropy}:

\[
H(X \mid Z_\ell^{\rightarrow}) \geq H(X \mid Z_\ell^{\rightarrow}, Z_\ell^{\leftarrow}),
\]
with equality iff \(X \perp Z_\ell^{\leftarrow} \mid Z_\ell^{\rightarrow}\).

Subtracting both sides from \(H(X)\) gives:
\[
I(X; Z_\ell^{\leftrightarrow}) = H(X) - H(X \mid Z_\ell^{\rightarrow}, Z_\ell^{\leftarrow}) 
\geq H(X) - H(X \mid Z_\ell^{\rightarrow}) = I(X; Z_\ell^{\rightarrow}).
\]

Thus:
\[
I(X; Z_\ell^{\leftrightarrow}) \geq I(X; Z_\ell^{\rightarrow}).
\]

Equality holds iff:
\[
H(X \mid Z_\ell^{\rightarrow}) = H(X \mid Z_\ell^{\rightarrow}, Z_\ell^{\leftarrow}),
\]
which by the equality condition of monotonicity of conditional entropy holds iff:
\[
X \perp Z_\ell^{\leftarrow} \mid Z_\ell^{\rightarrow}.
\]

Similarly with respect to output we can show:
\[
I( Z_\ell^{\leftrightarrow}; Y) \geq I( Z_\ell^{\rightarrow}; Y).
\]
This completes the proof.

\end{proof}

\begin{theorem}[General Bound on Representation Difference]
Let \( Z_\ell^{\leftrightarrow}, Z_\ell^{\rightarrow} \in \mathbb{R}^d \) denote the bidirectional and unidirectional representations of the same input token at a given layer, and define:
\[
\Delta_Z := Z_\ell^{\leftrightarrow} - Z_\ell^{\rightarrow}.
\]
Then the expected squared difference satisfies:
\begin{align*}
\mathbb{E} \|\Delta_Z\|^2 &= \operatorname{tr} \operatorname{Cov}(Z_\ell^{\leftrightarrow}) + \operatorname{tr} \operatorname{Cov}(Z_\ell^{\rightarrow}) - 2 \operatorname{tr} \operatorname{Cov}(Z_\ell^{\leftrightarrow}, Z_\ell^{\rightarrow}) + \|\mathbb{E}[\Delta_Z]\|^2.
\end{align*}
In particular, we have the following bound:
\begin{align*}
&\operatorname{tr} \operatorname{Cov}(Z_\ell^{\leftrightarrow}) + \operatorname{tr} \operatorname{Cov}(Z_\ell^{\rightarrow}) - 2 |\operatorname{tr} \operatorname{Cov}(Z_\ell^{\leftrightarrow}, Z_\ell^{\rightarrow})| \\
&\quad \leq \mathbb{E} \|\Delta_Z\|^2 - \|\mathbb{E}[\Delta_Z]\|^2 \\
&\quad \leq \operatorname{tr} \operatorname{Cov}(Z_\ell^{\leftrightarrow}) + \operatorname{tr} \operatorname{Cov}(Z_\ell^{\rightarrow}) + 2 |\operatorname{tr} \operatorname{Cov}(Z_\ell^{\leftrightarrow}, Z_\ell^{\rightarrow})|.
\end{align*}
\end{theorem}

\begin{proof}
By the covariance identity, we have:
\[
\operatorname{Cov}(\Delta_Z) = \operatorname{Cov}(Z_\ell^{\leftrightarrow}) + \operatorname{Cov}(Z_\ell^{\rightarrow}) - \operatorname{Cov}(Z_\ell^{\leftrightarrow}, Z_\ell^{\rightarrow}) - \operatorname{Cov}(Z_\ell^{\rightarrow}, Z_\ell^{\leftrightarrow}).
\]
Taking the trace and noting that \(\operatorname{tr}(A^\top) = \operatorname{tr}(A)\), we obtain:
\[
\operatorname{tr} \operatorname{Cov}(\Delta_Z) = \operatorname{tr} \operatorname{Cov}(Z_\ell^{\leftrightarrow}) + \operatorname{tr} \operatorname{Cov}(Z_\ell^{\rightarrow}) - 2 \operatorname{tr} \operatorname{Cov}(Z_\ell^{\leftrightarrow}, Z_\ell^{\rightarrow}).
\]
The expected squared norm decomposes as:
\[
\mathbb{E} \|\Delta_Z\|^2 = \operatorname{tr} \operatorname{Cov}(\Delta_Z) + \|\mathbb{E}[\Delta_Z]\|^2.
\]
Substituting the expression for \(\operatorname{Cov}(\Delta_Z)\) yields the stated identity.

Finally, since for any real scalar \(a\), we have \( -|a| \leq a \leq |a| \), it follows:
\[
-|\operatorname{tr} \operatorname{Cov}(Z_\ell^{\leftrightarrow}, Z_\ell^{\rightarrow})| \leq \operatorname{tr} \operatorname{Cov}(Z_\ell^{\leftrightarrow}, Z_\ell^{\rightarrow}) \leq |\operatorname{tr} \operatorname{Cov}(Z_\ell^{\leftrightarrow}, Z_\ell^{\rightarrow})|,
\]
which implies:

\begin{multline*}
\operatorname{tr}\,\operatorname{Cov}(\Delta_Z) \in \Bigl[
\operatorname{tr}\,\operatorname{Cov}(Z_\ell^{\leftrightarrow}) 
+ \operatorname{tr}\,\operatorname{Cov}(Z_\ell^{\rightarrow}) 
- 2\bigl|\operatorname{tr}\,\operatorname{Cov}(Z_\ell^{\leftrightarrow}, Z_\ell^{\rightarrow})\bigr|, \\[0.5em]
\operatorname{tr}\,\operatorname{Cov}(Z_\ell^{\leftrightarrow}) 
+ \operatorname{tr}\,\operatorname{Cov}(Z_\ell^{\rightarrow}) 
+ 2\bigl|\operatorname{tr}\,\operatorname{Cov}(Z_\ell^{\leftrightarrow}, Z_\ell^{\rightarrow})\bigr|
\Bigr].
\end{multline*}

Substitute into the expectation equation to complete the proof.
\end{proof}

\begin{lemma}[Effective Dimensionality of Bidirectional Representations] \label{app:lemma_deff}
Let \( Z_\ell^\rightarrow \in \mathbb{R}^D \) denote the unidirectional representation and \( Z_\ell^\leftrightarrow := (Z_\ell^\rightarrow, Z^\leftarrow) \in \mathbb{R}^{2D} \) the concatenated bidirectional representation of input \( X \). Define $\ell_2$-norm-based effective dimension as
\[
d_{\mathrm{eff}}(Z) := \frac{ \left( \sum_{i} \lambda_i \right)^2 }{ \sum_{i} \lambda_i^2 },
\]
where \( \lambda_i \) are eigenvalues of the covariance matrix of \( Z_\ell \).  
If \( \operatorname{Cov}(Z^\leftarrow, Z_\ell^\rightarrow) \) is non-singular, then:
\[
d_{\mathrm{eff}}(Z_\ell^\leftrightarrow) \geq d_{\mathrm{eff}}(Z_\ell^\rightarrow),
\]
with equality iff \( Z^\leftarrow \) is conditionally redundant given \( Z_\ell^\rightarrow \) (i.e., \( \operatorname{Cov}(Z^\leftarrow \mid Z_\ell^\rightarrow) = 0 \)).
\end{lemma}

\begin{proof}
Let \( \Sigma^\rightarrow := \operatorname{Cov}(Z_\ell^\rightarrow) \in \mathbb{R}^{D \times D} \) and \( \Sigma^\leftrightarrow := \operatorname{Cov}(Z_\ell^\leftrightarrow) \in \mathbb{R}^{2D \times 2D} \) denote the covariance matrices of unidirectional and bidirectional representations, respectively.

By block structure:
\[
\Sigma^\leftrightarrow = 
\begin{bmatrix}
\Sigma^\rightarrow & C \\
C^\top & \Sigma^\leftarrow
\end{bmatrix},
\]
where \( C := \operatorname{Cov}(Z_\ell^\rightarrow, Z^\leftarrow) \).

Let \( \{\lambda^\rightarrow_i\}_{i=1}^{D} \) be eigenvalues of \( \Sigma^\rightarrow \), and \( \{\lambda^\leftrightarrow_j\}_{j=1}^{2D} \) eigenvalues of \( \Sigma^\leftrightarrow \).

Since \( \Sigma^\leftrightarrow \) augments \( \Sigma^\rightarrow \) with additional variables \( Z^\leftarrow \) and cross-covariance \( C \), by eigenvalue interlacing theorem (Cauchy’s interlacing), we have:
\[
\sum_{j=1}^{2D} \lambda^\leftrightarrow_j \geq \sum_{i=1}^{D} \lambda^\rightarrow_i,
\]
and
\[
\sum_{j=1}^{2D} (\lambda^\leftrightarrow_j)^2 \geq \sum_{i=1}^{D} (\lambda^\rightarrow_i)^2,
\]
with strict inequality if \( C \) or \( \Sigma^\leftarrow \) is nonzero.

Applying definition:
\[
d_{\mathrm{eff}}(Z_\ell^\leftrightarrow) = \frac{ (\sum_j \lambda^\leftrightarrow_j)^2 }{ \sum_j (\lambda^\leftrightarrow_j)^2 }.
\]
Since numerator and denominator both increase under positive-definite augmentation,  
and quadratic-over-linear ratio increases under positive additive terms (Jensen's inequality),  
we conclude:
\[
d_{\mathrm{eff}}(Z_\ell^\leftrightarrow) \geq d_{\mathrm{eff}}(Z_\ell^\rightarrow).
\]
Equality holds iff \( \Sigma^\leftarrow = 0 \) and \( C = 0 \), implying \( Z^\leftarrow \) carries no additional variance or covariance beyond \( Z_\ell^\rightarrow \).
\end{proof}

\section{FlowNIB: Flow Neural Information Bottleneck} \label{app:FlowNIB}

We consider, for each layer $\ell$, the Markov chain
\[
X \;\longrightarrow\; Z_\ell \;\longrightarrow\; Y,
\]
where $X$ denotes the input, $Z_\ell$ the layer-$\ell$ representation (induced by an encoder $p_\theta(z_\ell\mid x)$), and $Y$ the target variable.

Our goal is to learn a representation $Z_\ell$ that:
\begin{itemize}
    \item compresses the input information by minimizing $I(X;Z_\ell)$,
    \item preserves predictive information by maximizing $I(Z_\ell;Y)$.
\end{itemize}

The classical \textbf{Information Bottleneck} (IB) principle~\citep{tishby2000information,tishby2015deep} formalizes this trade-off as
\[
\min_{p(z_\ell\mid x)} \; I(X;Z_\ell) \;-\; \beta\, I(Z_\ell;Y),
\]
where $\beta>0$ controls the balance between compression and prediction.

MI requires high-dimensional density ratios over $p(x,z_\ell)$ vs.\ $p(x)p(z_\ell)$ and $p(z_\ell,y)$ vs.\ $p(z_\ell)p(y)$, which are intractable to compute exactly when $X,Z_\ell$ are high-dimensional. The KL divergence
\[
D_{\mathrm{KL}}\!\bigl(p(x,z_\ell)\,\|\,p(x)p(z_\ell)\bigr)
\]
is especially problematic because neither joint nor marginals are known in practice and must be estimated \citep{belghazi2018mutual}. In deep networks, deterministic real-valued layers can also lead to unbounded $I(X;Z_\ell)$ in the continuous setting; in practice, one uses variational lower bounds and careful estimator training. These issues make vanilla IB difficult to apply directly to large models.

\paragraph{FlowNIB approach.}
To address these challenges, we introduce \textbf{FlowNIB}, which gradually shifts emphasis from input preservation to target prediction during training or post-hoc estimation. We use a time-dependent trade-off $\alpha:\mathbb{N}\to[0,1]$ that monotonically decays from $1$ to $0$ as the estimator training step $t$ increases (the model can be frozen). The FlowNIB loss at step $t$ for layer $\ell$ is
\[
\mathcal{L}_\ell(\theta, t) \;=\; -\Bigl(\alpha(t)\, I(X;Z_\ell) \;+\; \bigl(1-\alpha(t)\bigr)\, I(Z_\ell;Y)\Bigr),
\]
so early steps ($\alpha\!\approx\!1$) emphasize $I(X;Z_\ell)$, while later steps ($\alpha\!\approx\!0$) emphasize $I(Z_\ell;Y)$.

Each mutual information term is
\[
I(X;Z_\ell) \;=\; D_{\mathrm{KL}}\!\bigl(p(x,z_\ell)\,\|\,p(x)p(z_\ell)\bigr),
\qquad
I(Z_\ell;Y) \;=\; D_{\mathrm{KL}}\!\bigl(p(z_\ell,y)\,\|\,p(z_\ell)p(y)\bigr),
\]
with $D_{\mathrm{KL}}$ the Kullback–Leibler divergence. Since exact KLs are infeasible in high dimensions, we use variational lower bounds (MINE-style) \citep{belghazi2018mutual}:
\[
I(X;Z_\ell) \;\ge\; \mathbb{E}_{p(x,z_\ell)}\!\bigl[T_{xz,\ell}(x,z_\ell)\bigr] \;-\; \log \mathbb{E}_{p(x)p(z_\ell)}\!\bigl[e^{T_{xz,\ell}(x,z_\ell)}\bigr],
\]
\[
I(Z_\ell;Y) \;\ge\; \mathbb{E}_{p(z_\ell,y)}\!\bigl[T_{zy,\ell}(z_\ell,y)\bigr] \;-\; \log \mathbb{E}_{p(z_\ell)p(y)}\!\bigl[e^{T_{zy,\ell}(z_\ell,y)}\bigr],
\]
where $T_{xz,\ell}$ and $T_{zy,\ell}$ are learned scalar-valued critics (small neural networks) trained on joint pairs and product-of-marginals pairs (implemented by shuffling). Expectations are estimated with minibatches; we use the same critic architecture, batch size, negative sampling, optimizer, and steps across layers and models for comparability.

Because $X,Z_\ell,Y$ can have different scales and dimensions, we normalize MI estimates using the effective dimension (participation-ratio effective rank) \citep{roy2007effective}:
\[
d_{\mathrm{eff}}(Z_\ell)\;=\;\frac{\bigl(\sum_i \lambda_i\bigr)^2}{\sum_i \lambda_i^2},
\]
where $\{\lambda_i\}$ are the eigenvalues of $\mathrm{Cov}(Z_\ell)$ (estimated via PCA). The normalized MI estimates are
\[
\hat{I}(X;Z_\ell)
\;=\;
\frac{\;\mathbb{E}_{p(x,z_\ell)}[T_{xz,\ell}(x,z_\ell)] - \log \mathbb{E}_{p(x)p(z_\ell)}\!\bigl[e^{T_{xz,\ell}(x,z_\ell)}\bigr]\;}{\,d_{\mathrm{eff}}(Z_\ell)^{\,2}},
\]
\[
\hat{I}(Z_\ell;Y)
\;=\;
\frac{\;\mathbb{E}_{p(z_\ell,y)}[T_{zy,\ell}(z_\ell,y)] - \log \mathbb{E}_{p(z_\ell)p(y)}\!\bigl[e^{T_{zy,\ell}(z_\ell,y)}\bigr]\;}{\,d_{\mathrm{eff}}(Y)^{\,2}}.
\]
\emph{Remark.} The $d_{\mathrm{eff}}(\cdot)^{2}$ factor is a practical normalization for scale-matching across layers/models; it does not change the fact that the estimates are variational lower bounds.

Thus, the final loss optimized during FlowNIB training is

\[
\mathcal{L}_\ell(\theta, t)
\;=\;
-\Bigl( \alpha(t)\, \hat{I}(X;Z_\ell) \;+\; \bigl(1-\alpha(t)\bigr)\, \hat{I}(Z_\ell;Y) \Bigr),
\]

which, expanded, becomes
\begin{align*}
\mathcal{L}_\ell(\theta, t) \;=\; -\Biggl(&
\alpha(t)\, \frac{ \mathbb{E}_{p(x,z_\ell)}[T_{xz,\ell}(x,z_\ell)] 
- \log \mathbb{E}_{p(x)p(z_\ell)}\!\bigl[e^{T_{xz,\ell}(x,z_\ell)}\bigr] }
{ d_{\mathrm{eff}}(Z_\ell)^{2} } \nonumber \\[0.6em]
&+\; \bigl(1-\alpha(t)\bigr)\, \frac{ \mathbb{E}_{p(z_\ell,y)}[T_{zy,\ell}(z_\ell,y)] 
- \log \mathbb{E}_{p(z_\ell)p(y)}\!\bigl[e^{T_{zy,\ell}(z_\ell,y)}\bigr] }
{ d_{\mathrm{eff}}(Y)^{2} }
\Biggr).
\end{align*}

Here, $\theta$ denotes the parameters of the encoder $p_\theta(z_\ell\mid x)$ (if trained end-to-end) and of the critics $T_{xz,\ell},T_{zy,\ell}$. In our post-hoc setting, the encoder is frozen and $\theta$ refers to the critic parameters; $\alpha(t)$ is the estimator step index. All MI values are neural \emph{lower bounds} and are used for \emph{relative} comparisons across layers (e.g., for OIC selection), not as absolute MI.

\begin{theorem}[Consistency under optimal critics (per layer)]
\label{app:thm:consistency}
Fix a layer $\ell$ and let $(X,Z_\ell)\sim p(x,z_\ell)$ and $(Z_\ell,Y)\sim p(z_\ell,y)$
with the Markov chain $X \to Z_\ell \to Y$. Assume $p(x,z_\ell)\ll p(x)p(z_\ell)$ and
$p(z_\ell,y)\ll p(z_\ell)p(y)$, and that the relevant expectations are finite.
Suppose the Donsker--Varadhan optima (unique up to an additive constant) are attained:
\[
T_{xz,\ell}^*(x,z_\ell)=\log\frac{p(x,z_\ell)}{p(x)p(z_\ell)}+c_{xz,\ell},
\qquad
T_{zy,\ell}^*(z_\ell,y)=\log\frac{p(z_\ell,y)}{p(z_\ell)p(y)}+c_{zy,\ell}.
\]
Let the dimension-normalized estimators be
\[
\hat I(X;Z_\ell)\;=\;
\frac{\mathbb{E}_{p(x,z_\ell)}[T_{xz,\ell}(x,z_\ell)]-\log \mathbb{E}_{p(x)p(z_\ell)}[e^{T_{xz,\ell}(x,z_\ell)}]}
     {d_{\mathrm{eff}}(Z_\ell)^{2}},
\]
\[
\hat I(Z_\ell;Y)\;=\;
\frac{\mathbb{E}_{p(z_\ell,y)}[T_{zy,\ell}(z_\ell,y)]-\log \mathbb{E}_{p(z_\ell)p(y)}[e^{T_{zy,\ell}(z_\ell,y)}]}
     {d_{\mathrm{eff}}(Y)^{2}},
\]
where $d_{\mathrm{eff}}(\cdot)\in(0,\infty)$ are fixed scale factors (e.g., participation-ratio effective ranks).
Then
\[
\hat I(X;Z_\ell)\xrightarrow{\,T_{xz,\ell}\to T_{xz,\ell}^*\,}\frac{I(X;Z_\ell)}{d_{\mathrm{eff}}(Z_\ell)^{2}},
\qquad
\hat I(Z_\ell;Y)\xrightarrow{\,T_{zy,\ell}\to T_{zy,\ell}^*\,}\frac{I(Z_\ell;Y)}{d_{\mathrm{eff}}(Y)^{2}}.
\]
\end{theorem}

\begin{proof}
We show the claim for $(X,Z_\ell)$; the $(Z_\ell,Y)$ case is identical. By the DV representation,
\[
I(X;Z_\ell)\;=\;\sup_{T}\Big\{\mathbb{E}_{p(x,z_\ell)}[T(x,z_\ell)]-\log \mathbb{E}_{p(x)p(z_\ell)}[e^{T(x,z_\ell)}]\Big\}.
\]
Under the stated assumptions the supremum is achieved at
$T_{xz,\ell}^*(x,z_\ell)=\log\frac{p(x,z_\ell)}{p(x)p(z_\ell)}+c$ for any constant $c$,
and the objective is invariant to $c$:
\[
\mathbb{E}[T+c]-\log \mathbb{E}[e^{T+c}]=\mathbb{E}[T]-\log \mathbb{E}[e^{T}].
\]
Substituting $T_{xz,\ell}^*$ gives
\[
\mathbb{E}_{p(x,z_\ell)}\!\Big[\log\tfrac{p(x,z_\ell)}{p(x)p(z_\ell)}\Big]
-\log \mathbb{E}_{p(x)p(z_\ell)}\!\Big[\tfrac{p(x,z_\ell)}{p(x)p(z_\ell)}\Big]
= I(X;Z_\ell)-\log 1
= I(X;Z_\ell).
\]
By definition, the normalized estimator satisfies
\[
\hat I(X;Z_\ell)=\frac{\mathbb{E}_{p(x,z_\ell)}[T_{xz,\ell}]-\log \mathbb{E}_{p(x)p(z_\ell)}[e^{T_{xz,\ell}}]}
{d_{\mathrm{eff}}(Z_\ell)^{2}}.
\]
Hence, as $T_{xz,\ell}\to T_{xz,\ell}^*$ in function space, the numerator converges to $I(X;Z_\ell)$,
so $\hat I(X;Z_\ell)\to I(X;Z_\ell)/d_{\mathrm{eff}}(Z_\ell)^{2}$.
\end{proof}

\noindent\textit{Remark.} If $Y$ is discrete (e.g., class labels), one may set $d_{\mathrm{eff}}(Y)=1$
or compute it from a fixed embedding of $Y$; the theorem holds for any finite, positive normalizer.

\begin{lemma}[Non-Monotonic Dependence of Mutual Information on Output Dimension]
\label{lemma:mi_nonmonotonic}
Let \( X \in \mathbb{R}^{d_X} \), \( Z \in \mathbb{R}^{d_Z} \), and \( Y \in \mathbb{R}^{d_Y} \) denote input, latent, and output variables, respectively, with \( d_X, d_Z \) fixed and \( d_Y \) variable.

Then under FlowNIB optimization, the mutual information \( I(X;Z) \) and \( I(Z;Y) \) are non-monotonic functions of \( d_Y \), satisfying:

\[
\frac{\partial I(X;Z)}{\partial d_Y} > 0 \quad \text{for } d_Y < k,\quad
\frac{\partial I(X;Z)}{\partial d_Y} < 0 \quad \text{for } d_Y > k
\]
and similarly for \( I(Z;Y) \), for some critical threshold \( k \approx d_X \).

\end{lemma}

\begin{proof}[Proof Sketch]
FlowNIB optimizes a tradeoff between \( I(X;Z) \) and \( I(Z;Y) \), constrained by the model's representational capacity \( d_Z \) and data complexity.

When \( d_Y \) is small (\( d_Y \ll d_X \)), the predictive target contains limited information; thus \( I(Z;Y) \) is small and the latent representation does not need high complexity.

As \( d_Y \) increases toward \( d_X \), the predictive task demands richer information; both \( I(X;Z) \) and \( I(Z;Y) \) increase to capture relevant features.

However, once \( d_Y > d_X \), the output space exceeds the input manifold's capacity; the latent representation \( Z_\ell \) cannot fully carry the increased predictive information due to fixed \( d_Z \), leading to saturation and eventual decline in both \( I(X;Z) \) and \( I(Z;Y) \) as redundant or noisy output components exceed representational limits.

This yields a non-monotonic dependency of mutual information on \( d_Y \), peaking around \( d_Y \approx d_X \), then declining as \( d_Y \) further increases.

\end{proof}

\begin{proposition}[Effective Dimensionality Adaptation under FlowNIB]
\label{prop:effdim}
Let \( X \in \mathbb{R}^{d_X} \) and \( Y \in \mathbb{R}^{d_Y} \) be input and output random variables with dimensions \(d_X, d_Y\).  
Let \( Z_\ell \) denote the latent representation at layer \(\ell\) produced by a model trained under FlowNIB.

Then, under optimal critic approximation and continuous optimization, the effective dimension \(d_{\mathrm{eff}}(Z_\ell)\) exhibits the following dependence on \(d_Y\) (with \(d_X\) fixed):

\[
\frac{\partial d_{\mathrm{eff}}(Z_\ell)}{\partial d_Y}
\begin{cases}
< 0 & \text{if } d_Y \ll d_X \\
\approx 0 & \text{if } d_Y \approx d_X \\
> 0 & \text{if } d_Y \gg d_X
\end{cases}
\]
i.e., the effective dimension \(d_{\mathrm{eff}}(Z_\ell)\) decreases with \(d_Y\) when \(d_Y\) is small, plateaus when \(d_Y \approx d_X\), and increases when \(d_Y\) exceeds \(d_X\).
\end{proposition}

\begin{proof}[Proof Sketch]
Under FlowNIB, the latent representation \(Z_\ell\) is optimized to balance information preservation \(I(X; Z_\ell)\) and predictive sufficiency \(I(Z_\ell; Y)\), modulated dynamically by \(\alpha(t)\).

When \(d_Y \ll d_X\), the predictive information \(I(Z_\ell; Y)\) is small; the model prioritizes compressing irrelevant input variance, resulting in reduced \(d_{\mathrm{eff}}(Z_\ell)\).

When \(d_Y \approx d_X\), the predictive complexity of \(Y\) matches the input complexity; the model maintains \(d_{\mathrm{eff}}(Z_\ell)\) to balance preserving input and predictive information.

When \(d_Y \gg d_X\), the model must expand \(Z_\ell\) to capture sufficient predictive capacity, increasing \(d_{\mathrm{eff}}(Z_\ell)\) to span a higher-dimensional output manifold.

Empirical observations support this trend, where \(d_{\mathrm{eff}}(Z_\ell)\) traces a non-monotonic dependency on \(d_Y\), reflecting an intrinsic adaptation of latent geometry to output complexity.

\end{proof}

\begin{algorithm}[t]
\caption{FlowNIB: Flow Neural Information Bottleneck}
\label{alg:flownib}
\begin{algorithmic}[1]
\Require Dataset $\mathcal{D} = \{(x_i, y_i)\}_{i=1}^N$, pretrained model $f_\theta$, MI critics $T_{xz}$ and $T_{zy}$, scheduler $\alpha(t)$, number of training steps $T$
\State Initialize FlowNIB parameters and critics
\For{$t = 1$ to $T$}
    \State Sample mini-batch $\{(x, y)\}$ from $\mathcal{D}$
    \State Compute hidden representation $Z = f_\theta(x)$
    \State Estimate $I(X;Z)$ using MINE:
    \Statex \hspace{1em} $\hat{I}(X;Z) \leftarrow \mathbb{E}_{p(x,z)}[T_{xz}(x,z)] - \log \mathbb{E}_{p(x)p(z)}[e^{T_{xz}(x,z)}]$
    \State Estimate $I(Z;Y)$ using MINE:
    \Statex \hspace{1em} $\hat{I}(Z;Y) \leftarrow \mathbb{E}_{p(z,y)}[T_{zy}(z,y)] - \log \mathbb{E}_{p(z)p(y)}[e^{T_{zy}(z,y)}]$
    \State Normalize MI by effective dimensions:
    \Statex \hspace{1em} $\hat{I}_n(X;Z) \leftarrow \frac{\hat{I}(X;Z)}{d_{\text{eff}}(Z)^2}, \quad
    \hat{I}_n(Z;Y) \leftarrow \frac{\hat{I}(Z;Y)}{d_{\text{eff}}(Y)^2}$
    \State Compute dynamic loss:
    \Statex \hspace{1em} $\mathcal{L}_{\text{FlowNIB}} \leftarrow -\left( \alpha(t) \cdot \hat{I}_n(X;Z) + (1 - \alpha(t)) \cdot \hat{I}_n(Z;Y) \right)$
    \State Update schedule: $\alpha(t+1) \gets \max(0, \alpha(t) - \delta)$
    \State Backpropagate and update $\theta$, $T_{xz}$, $T_{zy}$
\EndFor
\end{algorithmic}
\end{algorithm}

\section{Ablation Study}

\begin{figure}[!t]
\begin{center}
    \includegraphics[width=1.00\linewidth]{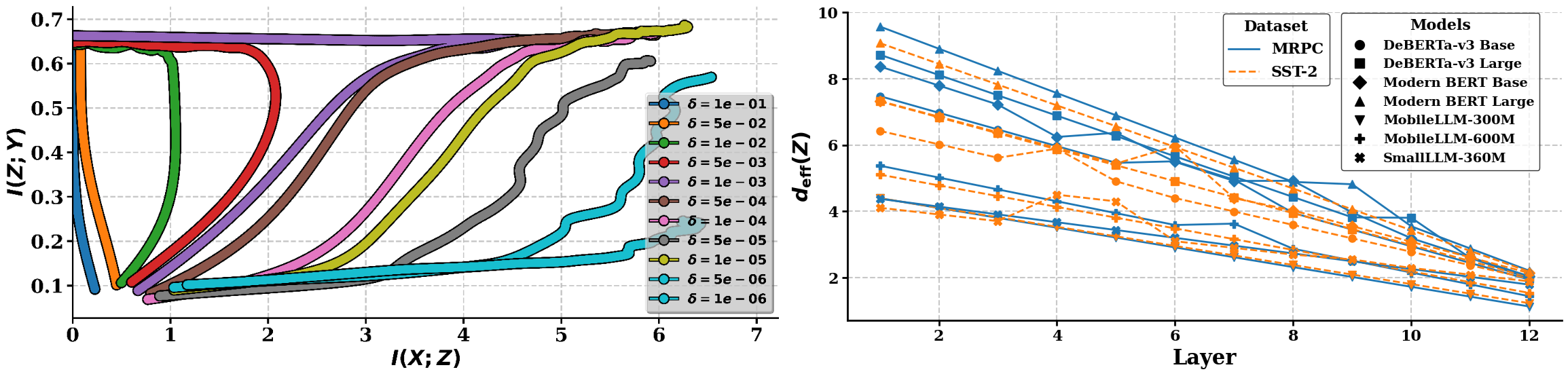}
   \end{center}
\caption{
(Left)Information plane trajectories under varying step sizes \(\delta\) for \(\alpha(t)\) in FlowNIB. Each curve shows the progression of mutual information \(I(X;Z)\) and \(I(Z;Y)\) across 2000 training epochs.  
(Right) Effective dimensionality \(d_{\mathrm{eff}}(Z)\) across layers for different models on MRPC and SST-2. Bidirectional models show higher \(d_{\mathrm{eff}}(Z)\) than unidirectional models at every layer.
}
\label{fig:info_plane_2}
\end{figure}
\begin{figure*}[!t]
    \centering
    \includegraphics[width=1.00\linewidth]{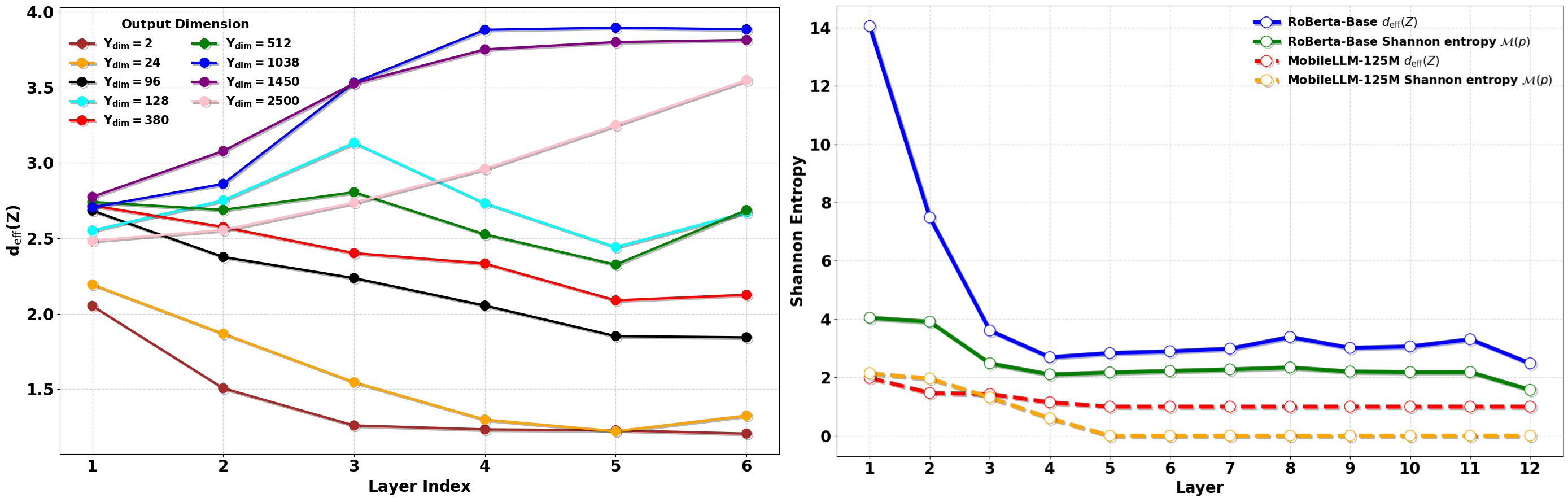}
    \caption{Effective dimension and Shannon entropy across network layers. \textbf{Left:} Effective dimension $d_{\mathrm{eff}}(Z)$ across layers for different output dimensions $Y_{\mathrm{dim}}$. \textbf{Right:} Shannon entropy $\mathcal{M}(p)$ across layers for RoBERTa-Base and MobileLLM-125M. Both plots use bold markers and shadows to emphasize trends in representation capacity and information compression.}
    \label{fig:layer_dff}
\end{figure*}

\begin{figure*}[!t]
    \centering
    \includegraphics[width=1.00\linewidth]{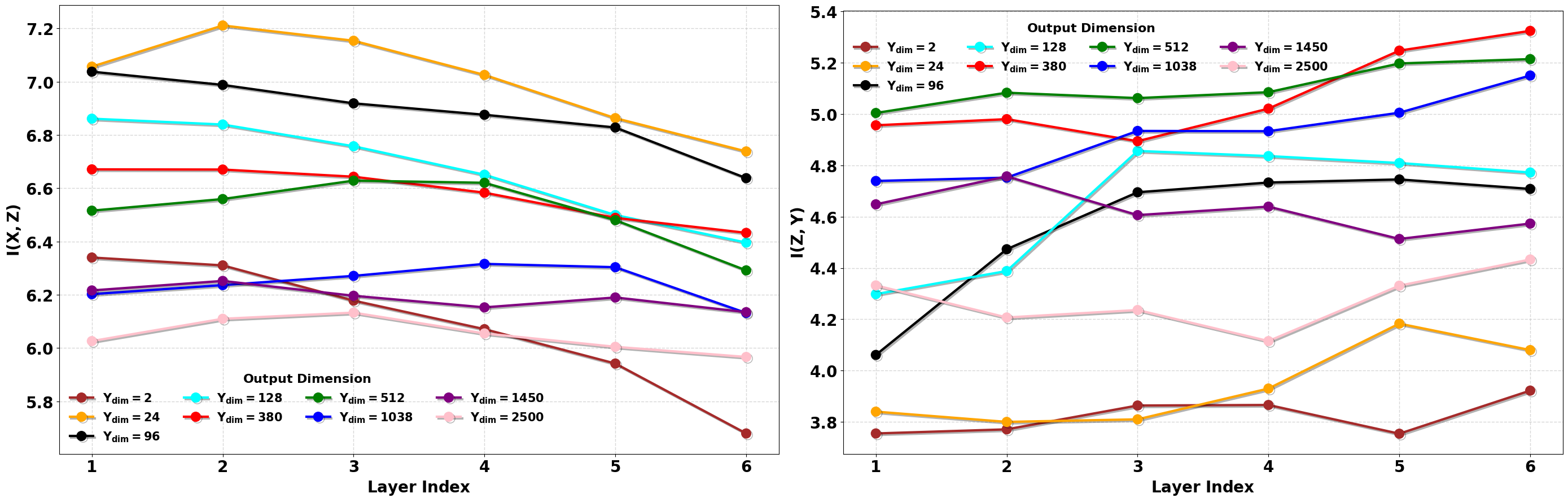}
    \caption{Visualization of mutual information across layers for different output dimensions. The left plot shows $I(X;Z)$ and the right plot shows $I(Z;Y)$ for various output dimensions $Y_{\mathrm{dim}}$. Each curve represents a specific output dimension, with bold markers and shadows to highlight the trends. This analysis provides insights into the evolution of representation capacity and target alignment across network layers as the output dimension increases.}
    \label{fig:compare_params}
\end{figure*}

\subsection{Effect of step size \(\delta\) on FlowNIB dynamics}\label{ab:delta}
We conducted an ablation study on the MRPC dataset to analyze the influence of the step size \(\delta\) controlling the decay of \(\alpha(t)\) in FlowNIB . Specifically, we varied \(\delta\) logarithmically from \(10^{-1}\) to \(10^{-11}\) and measured the evolution of mutual information \(I(X;Z)\) and \(I(Z;Y)\) throughout training. Figure~\ref{fig:info_plane_2}(left) shows the corresponding trajectories in the Information Plane. We observe that large step sizes (e.g., \(\delta=10^{-1}\)) induce rapid compression, sharply reducing \(I(X;Z)\) early in training but failing to preserve sufficient predictive information \(I(Z;Y)\), likely due to premature information loss. Conversely, very small step sizes (e.g., \(\delta=10^{-6}\)) cause negligible decay of \(\alpha(t)\), leading to nearly static representations that retain high \(I(X;Z)\) but fail to increase \(I(Z;Y)\). Intermediate step sizes (e.g., \(\delta=10^{-3}\) to \(\delta=10^{-4}\)) achieve the most desirable balance, gradually reducing \(I(X;Z)\) while increasing \(I(Z;Y)\), effectively steering the model toward the information bottleneck frontier. These findings empirically validate our theoretical insight that \(\delta\) serves as a critical control knob governing the speed and quality of information compression in FlowNIB.

\subsection{Effective Dimensionality Across Models} \label{ab:deff}
We measure effective dimensionality \( d_{\mathrm{eff}}(Z) \) across layers for DeBERTaV3 (base, large), ModernBERT (base, large), MobileLLM (300M, 600M), and SmallLLM (360M) on MRPC and SST-2. To ensure fair comparison across models with different depths, we normalize layer indices to a common scale of 1 to 12. Figure~\ref{fig:info_plane_2}(right) shows that \( d_{\mathrm{eff}}(Z) \) decreases monotonically with depth for all models, reflecting progressive compression (reasons of decreasing in Ablation Study~\ref{ab:impact_dim}).

Importantly, bidirectional models consistently exhibit higher \( d_{\mathrm{eff}}(Z) \) than unidirectional models at every layer. For example, on MRPC, DeBERTaV3-Large starts at \( 8.73 \) and compresses to \( 1.98 \), while MobileLLM-600M starts at \( 5.38 \) and compresses to \( 1.44 \). Similar trends appear on SST-2. These findings empirically support Lemma~\ref{lem:effective_dim}, confirming that bidirectional representations retain richer and more expressive features throughout depth.

\subsection{Effective Dimensionality vs. Output Complexity:} \label{ab:dimVsout}
We study how the effective dimensionality \( d_{\mathrm{eff}}(Z) \) of the latent representations changes with different output dimensions using the time-series forecasting dataset \texttt{ETTh1} \citep{zhou2021informer} by following Proposition \ref{prop:effdim}. We use a fixed 6-layer network with each layer having 128 units and keep the input dimension fixed at \( d_X = 380 \). We vary the output dimension \( d_Y \) from very small (\( d_Y = 2 \)) to much larger than the input (\( d_Y = 2500 \)). As shown in Figure~\ref{fig:layer_dff}, when the output dimension is much smaller than the input (\( d_Y \ll d_X \)), the effective dimension \( d_{\mathrm{eff}}(Z) \) decreases across layers, showing that the representation becomes more compressed. As \( d_Y \) grows closer to or larger than \( d_X \), we observe a non-monotonic trend: the dimension first compresses, then expands. When \( d_Y \gg d_X \), the effective dimension increases across layers, suggesting that the model adjusts the complexity of its representations to match the complexity of the prediction task. This behavior occurs even without directly optimizing for it in FlowNIB, showing that the shape of the output affects how the model organizes its internal representations.

\subsection{ Mutual Information Dynamics Across Output Dimensions and Layers:} \label{ab:impact_dim}
We explore how changing the output dimension \( Y_{\dim} \) affects mutual information and model performance by following Lemma~\ref{lemma:mi_nonmonotonic}. We trained the same model with different output sizes: \( Y_{\dim} \in \{2, 24, 96, 128, 380, 512, 1038, 1450, 2500\} \), and measured the mutual information between inputs and hidden layers \( I(X;Z) \), and between hidden layers and outputs \( I(Z;Y) \), after training. As shown in Figure~\ref{fig:compare_params}, \( I(X;Z) \) generally decreases across layers, especially for larger \( Y_{\dim} \), meaning more information is lost as the network gets deeper. At the same time, \( I(Z;Y) \) increases with depth, but for large \( Y_{\dim} \), it saturates early—suggesting it's harder for the model to align with very high-dimensional outputs. Interestingly, models with intermediate output dimensions (like \( Y_{\dim}=96 \) or \( 128 \)) show a better balance: they retain useful input information and achieve strong alignment with the output. This balance leads to better performance. Overall, we find that output dimensionality plays a key role in controlling how well the model balances input compression and predictive accuracy, making it an important hyperparameter to tune.

\subsection{Validating Generalized Effective Dimensionality} \label{app:ablation_effective}
To validate our definition of generalized effective dimensionality, we compare the layerwise trends of \( d_{\mathrm{eff}}(Z) \) (based on the \(\ell_2\)-norm participation ratio) and the Shannon entropy \( \mathcal{M}(p) \) across two models: RoBERTa-Base and MobileLLM-125M. As shown in Figure~\ref{fig:layer_dff} (Right), both metrics follow similar trends across layers—confirming that higher entropy leads to higher effective dimension, consistent with our definition \( d_{\mathrm{eff}}(Z; \mathcal{M}) := \exp(\mathcal{M}(p)) \). Notably, RoBERTa-Base maintains higher entropy and effective dimension than MobileLLM-125M at every layer, reflecting its richer representational capacity. The first few layers show a sharp drop in entropy, followed by a stable regime, aligning with the known compression phase in transformer representations. This empirical behavior confirms that both the entropy and \( d_{\mathrm{eff}} \) satisfy the expected monotonicity and boundedness properties outlined in Definition~\ref{def:generalized_deff}, including non-negativity and the Schur-concavity property.

\section{LoRA Based Performance Comparison}
Table~\ref{tab:performance_lora} shows the performance comparison between bidirectional and unidirectional models using LoRA.

\begin{table*}[htbp]
\centering
\resizebox{\textwidth}{!}{%
\begin{tabular}{l|l|c|c|c|c|c|c|c|c|c|c}
\hline
\rowcolor{gray!20}
\textbf{Model} & \textbf{Method} & \textbf{SST-2} & \textbf{MRPC} & \textbf{QNLI} & \textbf{RTE} & \textbf{CoLA} & \textbf{MNLI} & \textbf{BoolQ} & \textbf{HellaSwag} & \textbf{SIQA} & \textbf{Avg.} \\
\hline
\rowcolor[HTML]{EAF3FA} DeBERTa-v3-Base & Pooling & 95.12 & 88.75 & 91.75 & 82.85 & 85.43 & 85.96 & 63.55 & 55.22 & 46.74 & 77.15 \\
\rowcolor[HTML]{EAF3FA}  & Masking & 96.22 & 90.03 & 93.10 & 85.92 & 88.55 & 88.10 & 65.05 & 68.33 & 61.92 & 81.81\\
\rowcolor[HTML]{EAF3FA} DeBERTa-v3-Large & Pooling & 96.25 & 92.88 & 94.67 & 88.90 & 94.12 & 91.92 & 65.48 & 58.15 & 52.04 & 81.82\\
\rowcolor[HTML]{EAF3FA}  & Masking & 96.94 & 94.95 & 95.35 & 90.85 & 93.05 & 91.96 & 65.12 & 74.10 & 66.41 & 85.30 \\
\rowcolor[HTML]{EAF3FA} RoBERTa-Base & Pooling & 93.80 & 83.40 & 91.13 & 82.20 & 85.45 & 85.95 & 62.10 & 51.78 & 44.63 & 75.72\\
\rowcolor[HTML]{EAF3FA}  & Masking & 94.80 & 86.10 & 93.42 & 86.02 & 88.25 & 87.20 & 63.80 & 65.33 & 61.12 & 80.45\\
\rowcolor[HTML]{EAF3FA} RoBERTa-Large & Pooling & 95.12 & 88.40 & 93.76 & 86.10 & 93.02 & 90.14 & 64.00 & 56.23 & 47.15 & 79.66\\
\rowcolor[HTML]{EAF3FA}  & Masking & 96.67 & 91.98 & 95.10 & 88.45 & 95.33 & 90.92 & 64.25 & 70.35 & 62.45 & 83.83\\
\rowcolor[HTML]{EAF3FA} ModernBERT-Base & Pooling & 93.70 & 82.40 & 90.25 & 81.52 & 84.22 & 86.02 & 62.00 & 54.18 & 45.70 & 75.78\\
\rowcolor[HTML]{EAF3FA}  & Masking & 94.92 & 84.05 & 92.88 & 85.00 & 85.80 & 88.55 & 61.35 & 62.00 & 60.00 & 78.95\\
\rowcolor[HTML]{EAF3FA} ModernBERT-Large & Pooling & 95.00 & 88.55 & 93.50 & 87.32 & 90.25 & 92.80 & 63.50 & 59.00 & 48.50 & 79.82\\
\rowcolor[HTML]{EAF3FA}  & Masking & 96.32 & 91.10 & 95.12 & 88.50 & 91.02 & 92.10 & 63.90 & 72.42 & 64.33 & 83.42\\
\rowcolor[HTML]{F8DADA} GPT-2 Medium & Pooling & 92.70 & 84.32 & 90.42 & 68.50 & 79.15 & 78.02 & 62.33 & 36.80 & 37.42 & 69.96\\
\rowcolor[HTML]{F8DADA}  & Generation & 93.40 & 85.72 & 91.65 & 69.02 & 80.10 & 79.43 & 63.00 & 36.55 & 42.12 & 71.00 \\
\rowcolor[HTML]{F8DADA} GPT-2 Large & Pooling & 93.75 & 85.50 & 83.35 & 65.90 & 82.85 & 79.55 & 63.50 & 39.20 & 40.50 & 70.68\\
\rowcolor[HTML]{F8DADA}  & Generation & 94.05 & 87.05 & 85.12 & 67.88 & 84.23 &81.72 & 64.05 & 39.70 & 45.02 & 71.98\\
\rowcolor[HTML]{F8DADA} SmolLM2-360M & Pooling & 93.80 & 84.20 & 90.92 & 69.90 & 81.22 & 84.10 & 62.75 & 41.20 & 41.55 & 72.18\\
\rowcolor[HTML]{F8DADA}  & Generation & 94.52 & 85.85 & 91.93 & 70.50 & 83.80 & 85.10 & 62.60 & 42.40 & 49.45 & 73.68\\
\rowcolor[HTML]{F8DADA} SmolLM2-135M & Pooling & 91.90 & 83.05 & 89.43 & 67.55 & 80.15 & 81.52 & 61.35 & 37.00 & 40.25 & 70.13\\
\rowcolor[HTML]{F8DADA}  & Generation & 92.80 & 83.85 & 90.05 & 68.12 & 81.82 & 82.78 & 61.70 & 40.00 & 46.20 & 71.59\\
\rowcolor[HTML]{F8DADA} MobileLLM-125M & Pooling & 92.25 & 81.42 & 89.82 & 68.42 & 79.12 & 81.35 & 59.50 & 32.30 & 40.40 & 69.07\\
\rowcolor[HTML]{F8DADA}  & Generation & 92.98 & 82.35 & 90.22 & 68.92 & 80.42 & 82.20 & 60.25 & 36.12 & 47.33 & 70.53\\
\rowcolor[HTML]{F8DADA} MobileLLM-350M & Pooling & 93.00 & 82.65 & 90.32 & 69.55 & 81.58 & 82.55 & 62.05 & 35.42 & 41.50 & 70.73\\
\rowcolor[HTML]{F8DADA}  & Generation & 94.10 & 82.98 & 90.85 & 70.25 & 82.62 & 83.40 & 62.85 & 39.20 & 50.05 & 72.15\\
\rowcolor[HTML]{F8DADA} MobileLLM-600M & Pooling & 94.25 & 86.80 & 90.92 & 71.32 & 83.92 & 84.12 & 63.50 & 44.50 & 44.20 & 73.06\\
\rowcolor[HTML]{F8DADA}  & Generation & 94.95 & 87.55 & 91.50 & 72.02 & 85.92 & 84.30 & 63.75 & 47.80 & 57.32 & 75.68\\
\hline
\end{tabular}%
}
\caption{Accuracy results across nine NLP classification tasks comparing bidirectional and unidirectional models under pooling, masking, and generation inference strategies using LoRA fine-tuning.}
\label{tab:performance_lora}
\end{table*}

\begin{figure*}[htbp]
\begin{center}
    \includegraphics[width=1.0\linewidth]{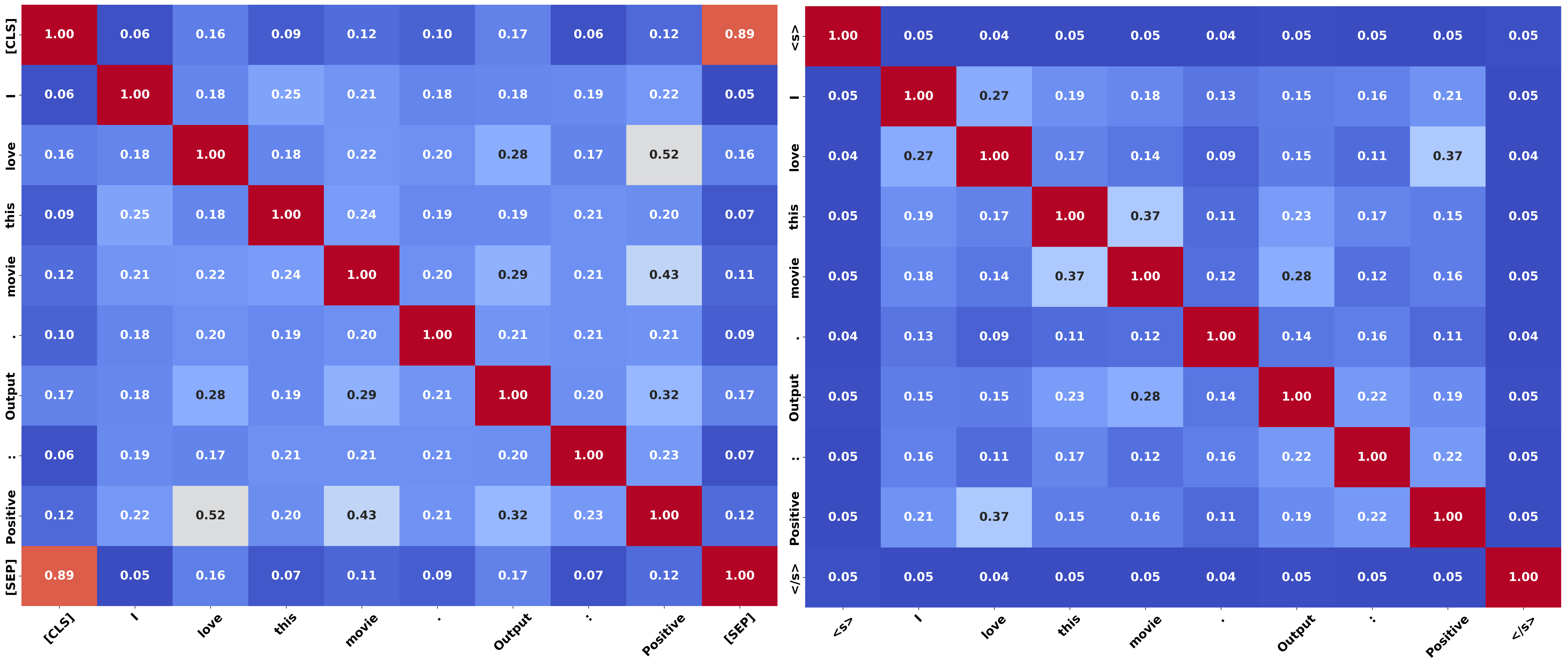}
   \end{center}
\caption{Token-level mutual information matrix on the SST-2 dataset for sentiment classification, computed from the final hidden layer representations. (Left) RoBERTa-base; (Right) SmallM2-360}
\label{fig:MI_example}
\end{figure*}

\section{Dataset}
The details of datasets are described in Table~\ref{tab:datasets}
\begin{table*}[htbp]
\centering
\resizebox{\textwidth}{!}{%
\begin{tabular}{l|l|l|p{9cm}}
\hline
\rowcolor{gray!20} 
\textbf{Dataset} & \textbf{Task Type} & \textbf{Domain} & \textbf{Description} \\
\hline
\rowcolor[HTML]{EAF3FA} \textbf{SST-2} \citep{wang2018glue} & Classification & Sentiment Analysis & The Stanford Sentiment Treebank, a binary sentiment classification dataset labeling sentences as positive or negative. \\
\rowcolor[HTML]{EAF3FA} \textbf{MRPC} \citep{wang2018glue} & Classification & Paraphrase Detection & The Microsoft Research Paraphrase Corpus for detecting whether two sentences are semantically equivalent. \\
\rowcolor[HTML]{EAF3FA} \textbf{QNLI} \citep{wang2018glue} & Classification & Question Answering / NLI & A question natural language inference dataset built from SQuAD, determining if a context sentence contains the answer. \\
\rowcolor[HTML]{EAF3FA} \textbf{RTE} \citep{wang2018glue} & Classification & Natural Language Inference & The Recognizing Textual Entailment dataset for determining if a hypothesis is entailed by a premise. \\
\rowcolor[HTML]{EAF3FA} \textbf{MNLI} \citep{wang2018glue} & Classification & Natural Language Inference & Multi-Genre Natural Language Inference dataset covering entailment, neutral, and contradiction relations across multiple genres. \\
\rowcolor[HTML]{EAF3FA} \textbf{CoLA} \citep{wang2018glue} & Classification & Grammatical Acceptability & Corpus of Linguistic Acceptability, evaluating whether sentences conform to English grammatical rules. \\
\rowcolor[HTML]{EAF3FA} \textbf{BoolQ} \citep{clark2019boolq} & Classification & Reading Comprehension & Boolean Questions dataset with yes/no questions based on Wikipedia passages requiring reading comprehension. \\
\rowcolor[HTML]{EAF3FA} \textbf{HellaSwag} \citep{zellers2019hellaswag} & Classification & Commonsense Reasoning & Tests commonsense reasoning by selecting the most plausible continuation of a given scenario. \\
\rowcolor[HTML]{EAF3FA} \textbf{SIQA} \citep{sap2019socialiqa} & Classification & Social Intelligence & Social IQa dataset evaluating models' understanding of social situations, emotions, and intentions. \\
\rowcolor[HTML]{F8DADA} \textbf{WASSA} \citep{vinayakumar2017deepcybernet} & Regression & Emotion Intensity & WASSA-2017 dataset for predicting emotion intensity scores for tweets across multiple emotions. \\
\rowcolor[HTML]{F8DADA} \textbf{SICK} \citep{marelli-etal-2014-sick} & Regression & Semantic Similarity & Sentences Involving Compositional Knowledge dataset for measuring sentence similarity and entailment. \\
\rowcolor[HTML]{F8DADA} \textbf{STSB-regression} \citep{cer2017semeval} & Regression & Semantic Similarity & Semantic Textual Similarity Benchmark scored on a continuous scale from 0 to 5. \\
\rowcolor[HTML]{F8DADA} \textbf{LCP} \citep{shardlow2020complex} & Regression & Lexical Complexity & Lexical Complexity Prediction dataset for predicting the complexity of words within their context. \\
\rowcolor[HTML]{F8DADA} \textbf{CRP} \citep{shardlow2020complex} & Regression & Complex Word Identification & Complex Word Identification dataset from SemEval, labeling words as simple or complex in context. \\
\rowcolor[HTML]{F8DADA} \textbf{Humicroedit} \citep{hossain2019president} & Regression & Humor Perception & SemEval humor dataset evaluating the impact of small text edits (micro-edits) on humor perception. \\
\hline
\end{tabular}%
}
\caption{Overview of the 16 benchmark datasets used in our experiments across classification and regression tasks.}
\label{tab:datasets}
\end{table*}

\section{Environment Setup} \label{app:env_set}

All experiments are conducted using PyTorch 2.0 and Hugging Face Transformers version 4.50. Training and evaluation are performed on a single NVIDIA A100 GPU with 80GB of memory. We use Python 3.10 within an Anaconda virtual environment configured with CUDA 12.1. Key dependencies include NumPy, SciPy, scikit-learn, and tqdm for data processing and evaluation. Random seeds are fixed across all runs to ensure reproducibility.

\section{Evaluation Metrics} \label{app:evl_metric}

We evaluate our models using task-specific metrics selected for their interpretability, relevance, and comparability to prior work. For \textbf{classification tasks}, we adopt \textit{accuracy} as the primary metric, defined as the ratio of correct predictions to the total number of predictions:
\[
\text{Accuracy} = \frac{\text{Number of correct predictions}}{\text{Total number of predictions}}.
\]
Accuracy provides a straightforward measure of model correctness and aligns with standard practices in classification benchmarks~\citep{wang2018glue}.

For \textbf{regression tasks}, we report both \textit{mean squared error (MSE)} and \textit{mean absolute error (MAE)} to capture complementary aspects of prediction error. MSE emphasizes larger errors due to the squared term, while MAE reflects the average magnitude of errors:
\[
\text{MSE} = \frac{1}{N} \sum_{i=1}^{N} (y_i - \hat{y}_i)^2,\quad
\text{MAE} = \frac{1}{N} \sum_{i=1}^{N} |y_i - \hat{y}_i|,
\]
where \( N \) is the number of samples, \( y_i \) is the ground-truth label, and \( \hat{y}_i \) is the predicted value. These metrics ensure a robust evaluation of both typical and extreme prediction errors~\citep{cer2017semeval, marelli2014sick}.

In addition to task performance metrics, we measure the \textit{mutual information} between the input \( X \) and the learned representation \( Z_\ell \), denoted \( I(X;Z) \). Mutual information quantifies how much information about the input is preserved in \( Z_\ell \), providing insight into the information bottleneck trade-off~\citep{tishby2015deep}. We estimate \( I(X;Z) \) using a variational lower bound based on Mutual Information Neural Estimation~\citep{belghazi2018mutual}, following prior work in information-theoretic analyses of neural networks.

All metrics are computed using scikit-learn and official benchmark evaluation scripts. Model selection is performed based on validation set performance, with final metrics reported on the held-out test sets.

\section{Model Description} \label{app:model_description}

We compare our method with a range of pretrained language models covering both bidirectional and unidirectional architectures. The bidirectional baselines include \textbf{DeBERTaV3-Base}~\citep{he2020deberta}, \textbf{DeBERTaV3-Large}~\citep{he2020deberta}, \textbf{RoBERTa-Base}~\citep{liu2019roberta}, \textbf{RoBERTa-Large}~\citep{liu2019roberta}, \textbf{ModernBERT-Base}~\citep{warner2024smarter}, and \textbf{ModernBERT-Large}~\citep{warner2024smarter}. The unidirectional baselines include \textbf{GPT-2 Medium}~\citep{radford2019language}, \textbf{GPT-2 Large}~\citep{radford2019language}, \textbf{MobileLLM-125M}~\citep{liu2024mobilellm}, \textbf{MobileLLM-350M}~\citep{liu2024mobilellm}, \textbf{MobileLLM-630M}~\citep{liu2024mobilellm}, \textbf{SmolLM-135M}~\citep{allal2024SmolLM}, and \textbf{SmolLM-360M}~\citep{allal2024SmolLM}. These models are selected to cover a range of sizes and architectures, enabling a fair and broad evaluation of representational learning. We focus on smaller model sizes to allow fair comparisons since large bidirectional models are not readily available. All baseline models are fine-tuned using RoCoFT adapters with an adapter rank of \( r = 3 \), enabling efficient fine-tuning without modifying the main model parameters. We use a cosine learning rate schedule for training.

\section{Hyperparameters} \label{app:hyper}
We select hyperparameters systematically to ensure consistent and balanced evaluation across all tasks and models. For classification tasks, we set the learning rate to \( 1 \times 10^{-4} \) with batch sizes between 8 and 16. For regression tasks, we increase the learning rate to \( 1 \times 10^{-3} \) with batch sizes ranging from 8 to 32. All models are fine-tuned using the AdamW optimizer with a cosine learning rate schedule, weight decay values in the range of 0.1 to 0.2, and a warmup ratio of 0.1. Gradient accumulation steps are varied between 1 and 8 depending on GPU memory capacity. To improve training stability, gradients are clipped at a maximum norm of 1.0, and label smoothing with a factor of 0.1 is applied where applicable. Each model is trained for 2 to 30 epochs, with warmup steps selected between 100 and 500. These hyperparameter settings are held consistent across experimental runs to ensure fair comparisons and reproducibility. This finding aligns with earlier work showing the benefits of bidirectional models for non-autoregressive NLP tasks. A detailed breakdown of the hyperparameters used for each dataset and model is provided in Appendix, including Table~\ref{tab:humicroedit_allmodels} (Humicroedit), Table~\ref{tab:wassa_hyperparams} (WASSA), Table~\ref{tab:sick_hyperparams} (SICK), Table~\ref{tab:stsb_hyperparams} (STS-B), Table~\ref{tab:lcp_hyperparams} (LCP), Table~\ref{tab:sst2_hyperparams} (SST-2), Table~\ref{tab:mrpc_hyperparams} (MRPC), Table~\ref{tab:qnli_hyperparams} (QNLI), Table~\ref{tab:rte_hyperparams} (RTE), Table~\ref{tab:cola_hyperparams} (CoLA), Table~\ref{tab:mnli_hyperparams} (MNLI), Table~\ref{tab:boolqa_hyperparams} (BoolQ), Table~\ref{tab:hellaswag_hyperparams} (HellaSwag), and Table~\ref{tab:siqa_hyperparams} (SIQA).

\begin{table*}[htbp]
\centering
\resizebox{\textwidth}{!}{%
\begin{tabular}{l|l|l|l|l|l|l|l|l}
\hline
\rowcolor{gray!20}
\textbf{Model} & \textbf{Learning Rate} & \textbf{Batch Size} & \textbf{Grad Accum} & \textbf{Weight Decay} & \textbf{LR Scheduler} & \textbf{Rank} & \textbf{Max Length} & \textbf{Epochs / Warmup Steps} \\
\hline
MobileLLM-350M & 6e-4 & 16 & 1 & 0.2 & Cosine & 3 & 512 & 10 / 100 \\
SmolLM-360M    & 6e-4 & 16 & 1 & 0.2 & Cosine & 3 & 512 & 10 / 100 \\
SmolLM-135M    & 6e-4 & 16 & 1 & 0.2 & Cosine & 3 & 512 & 10 / 100 \\
ModernBERT-base & 6e-4 & 16 & 1 & 0.2 & Cosine & 3 & 512 & 10 / 100 \\
GPT2-medium     & 6e-4 & 16 & 1 & 0.2 & Cosine & 3 & 512 & 10 / 100 \\
GPT2-large      & 6e-4 & 16 & 1 & 0.2 & Cosine & 3 & 512 & 10 / 100 \\
deberta-v3-base & 6e-4 & 16 & 1 & 0.2 & Cosine & 3 & 512 & 10 / 100 \\
roberta-base    & 6e-4 & 16 & 1 & 0.2 & Cosine & 3 & 512 & 10 / 100 \\
roberta-large   & 6e-4 & 16 & 1 & 0.2 & Cosine & 3 & 512 & 10 / 100 \\
deberta-v3-large& 6e-4 & 16 & 1 & 0.2 & Cosine & 3 & 512 & 10 / 100 \\
Mobile-llm-125  & 6e-4 & 16 & 1 & 0.2 & Cosine & 3 & 512 & 10 / 100 \\
Mobile-llm-630  & 6e-4 & 16 & 1 & 0.2 & Cosine & 3 & 512 & 10 / 100 \\
moden-bert-large& 6e-4 & 16 & 1 & 0.2 & Cosine & 3 & 512 & 10 / 100 \\
\hline
\end{tabular}%
}
\caption{Hyperparameter settings for the Humicroedit dataset for each evaluated model.}
\label{tab:humicroedit_allmodels}
\end{table*}

\begin{table*}[htbp]
\centering
\resizebox{\textwidth}{!}{%
\begin{tabular}{l|l|l|l|l|l|l|l|l}
\hline
\rowcolor{gray!20}
\textbf{Model} & \textbf{Learning Rate} & \textbf{Batch Size} & \textbf{Grad Accum} & \textbf{Weight Decay} & \textbf{LR Scheduler} & \textbf{Rank} & \textbf{Max Length} & \textbf{Epochs / Warmup Steps} \\
\hline
SmolLM-135M     & 5e-4 & 14 & 1 & 0.2 & Cosine & 3 & 512 & 10 / 100 \\
MobileLLM-350M  & 5e-4 & 14 & 1 & 0.2 & Cosine & 3 & 512 & 10 / 100 \\
SmolLM-360M     & 5e-4 & 14 & 1 & 0.2 & Cosine & 3 & 512 & 10 / 100 \\
GPT2-medium      & 5e-4 & 14 & 1 & 0.2 & Cosine & 3 & 512 & 10 / 100 \\
GPT2-large       & 5e-4 & 14 & 1 & 0.2 & Cosine & 3 & 512 & 10 / 100 \\
ModernBERT-base  & 5e-4 & 14 & 1 & 0.2 & Cosine & 3 & 512 & 10 / 100 \\
deberta-v3-base  & 6e-4 & 16 & 1 & 0.2 & Cosine & 3 & 512 & 10 / 100 \\
roberta-base     & 6e-4 & 16 & 1 & 0.2 & Cosine & 3 & 512 & 10 / 100 \\
roberta-large    & 6e-4 & 16 & 1 & 0.2 & Cosine & 3 & 512 & 10 / 100 \\
deberta-v3-large & 6e-4 & 16 & 1 & 0.2 & Cosine & 3 & 512 & 10 / 100 \\
Mobile-llm-125   & 6e-4 & 16 & 1 & 0.2 & Cosine & 3 & 512 & 10 / 100 \\
Mobile-llm-630   & 6e-4 & 16 & 1 & 0.2 & Cosine & 3 & 512 & 10 / 100 \\
moden-bert-large & 6e-4 & 16 & 1 & 0.2 & Cosine & 3 & 512 & 10 / 100 \\
\hline
\end{tabular}%
}
\caption{Hyperparameter settings for the WASSA dataset for each evaluated model.}
\label{tab:wassa_hyperparams}
\end{table*}

\begin{table*}[htbp]
\centering
\resizebox{\textwidth}{!}{%
\begin{tabular}{l|l|l|l|l|l|l|l|l}
\hline
\rowcolor{gray!20}
\textbf{Model} & \textbf{Learning Rate} & \textbf{Batch Size} & \textbf{Grad Accum} & \textbf{Weight Decay} & \textbf{LR Scheduler} & \textbf{Rank} & \textbf{Max Length} & \textbf{Epochs / Warmup Steps} \\
\hline
SmolLM-360M     & 1e-3 & 14 & 1 & 0.2 & Cosine & 3 & 512 & 20 / 100 \\
SmolLM-135M     & 1e-3 & 14 & 1 & 0.2 & Cosine & 3 & 512 & 20 / 100 \\
ModernBERT-base  & 1e-3 & 8  & 2 & 0.2 & Cosine & 3 & 512 & 20 / 100 \\
deberta-v3-base  & 1e-3 & 8  & 2 & 0.2 & Cosine & 3 & 512 & 20 / 100 \\
GPT2-medium      & 1e-3 & 14 & 1 & 0.2 & Cosine & 3 & 512 & 20 / 100 \\
GPT2-large       & 1e-3 & 14 & 1 & 0.2 & Cosine & 3 & 512 & 20 / 100 \\
roberta-base     & 1e-3 & 8  & 2 & 0.2 & Cosine & 3 & 512 & 20 / 100 \\
roberta-large    & 1e-3 & 8  & 2 & 0.2 & Cosine & 3 & 512 & 20 / 100 \\
deberta-v3-large & 1e-3 & 8  & 2 & 0.2 & Cosine & 3 & 512 & 20 / 100 \\
Mobile-llm-125   & 1e-3 & 8  & 2 & 0.2 & Cosine & 3 & 512 & 20 / 100 \\
Mobile-llm-630   & 1e-3 & 8  & 2 & 0.2 & Cosine & 3 & 512 & 20 / 100 \\
moden-bert-large & 1e-3 & 8  & 2 & 0.2 & Cosine & 3 & 512 & 20 / 100 \\
\hline
\end{tabular}%
}
\caption{Hyperparameter settings for the SICK dataset for each evaluated model.}
\label{tab:sick_hyperparams}
\end{table*}

\begin{table*}[htbp]
\centering
\resizebox{\textwidth}{!}{%
\begin{tabular}{l|l|l|l|l|l|l|l|l|l}
\hline
\rowcolor{gray!20}
\textbf{Model} & \textbf{Learning Rate} & \textbf{Batch Size} & \textbf{Grad Accum} & \textbf{Weight Decay} & \textbf{LR Scheduler} & \textbf{Rank} & \textbf{Max Length} & \textbf{Epochs / Warmup Steps} & \textbf{Max Grad Norm} \\
\hline
SmolLM-360M     & 2e-4 & 8  & 1 & 0.1  & Cosine & 3 & 512 & 10 / 100 & 1 \\
MobileLLM-350M  & 2e-4 & 8  & 1 & 0.1  & Cosine & 3 & 512 & 10 / 100 & 1 \\
SmolLM-135M     & 2e-4 & 8  & 1 & 0.1  & Cosine & 3 & 512 & 10 / 100 & 1 \\
deberta-v3-base & 6e-4 & 16 & 1 & 0.2  & Cosine & 3 & 512 & 20 / 100 & 1 \\
roberta-base     & 6e-4 & 16 & 1 & 0.2  & Cosine & 3 & 512 & 20 / 100 & 1 \\
roberta-large    & 6e-4 & 16 & 1 & 0.2  & Cosine & 3 & 512 & 20 / 100 & 1 \\
deberta-v3-large & 6e-4 & 16 & 1 & 0.2  & Cosine & 3 & 512 & 20 / 100 & 1 \\
Mobile-llm-125   & 6e-4 & 16 & 1 & 0.2  & Cosine & 3 & 512 & 20 / 100 & 1 \\
Mobile-llm-630   & 6e-4 & 16 & 1 & 0.2  & Cosine & 3 & 512 & 20 / 100 & 1 \\
moden-bert-large & 6e-4 & 16 & 1 & 0.2  & Cosine & 3 & 512 & 20 / 100 & 1 \\
GPT2-medium      & 1e-4 & 16 & 4 & 0.0  & Cosine & 3 & 512 & 10 / 100 & 1 \\
GPT2-large       & 1e-4 & 16 & 4 & 0.0  & Cosine & 3 & 512 & 10 / 100 & 1 \\
ModernBERT-base  & 1e-4 & 16 & 4 & 0.0  & Cosine & 3 & 512 & 10 / 100 & 1 \\
\hline
\end{tabular}%
}
\caption{Hyperparameter settings for the STSB dataset for each evaluated model.}
\label{tab:stsb_hyperparams}
\end{table*}

\begin{table*}[htbp]
\centering
\resizebox{\textwidth}{!}{%
\begin{tabular}{l|l|l|l|l|l|l|l|l}
\hline
\rowcolor{gray!20}
\textbf{Model} & \textbf{Learning Rate} & \textbf{Batch Size} & \textbf{Grad Accum} & \textbf{Weight Decay} & \textbf{LR Scheduler} & \textbf{Rank} & \textbf{Max Length} & \textbf{Epochs / Warmup Steps} \\
\hline
SmolLM-360M     & 5e-4 & 4  & 4 & 0.2 & Cosine & 3 & 512 & 10 / 100 \\
MobileLLM-350M  & 5e-4 & 4  & 4 & 0.2 & Cosine & 3 & 512 & 10 / 100 \\
SmolLM-135M     & 5e-4 & 4  & 4 & 0.2 & Cosine & 3 & 512 & 10 / 100 \\
ModernBERT-base & 5e-4 & 4  & 4 & 0.2 & Cosine & 3 & 512 & 10 / 100 \\
GPT2-medium     & 5e-4 & 4  & 4 & 0.2 & Cosine & 3 & 512 & 10 / 100 \\
GPT2-large      & 5e-4 & 4  & 4 & 0.2 & Cosine & 3 & 512 & 10 / 100 \\
roberta-base     & 1e-3 & 10 & 1 & 0.2 & Cosine & 3 & 512 & 10 / 100 \\
roberta-large    & 1e-3 & 10 & 1 & 0.2 & Cosine & 3 & 512 & 10 / 100 \\
deberta-v3-large & 1e-3 & 10 & 1 & 0.2 & Cosine & 3 & 512 & 10 / 100 \\
Mobile-llm-125   & 1e-3 & 10 & 1 & 0.2 & Cosine & 3 & 512 & 10 / 100 \\
Mobile-llm-630   & 1e-3 & 10 & 1 & 0.2 & Cosine & 3 & 512 & 10 / 100 \\
moden-bert-large & 1e-3 & 10 & 1 & 0.2 & Cosine & 3 & 512 & 10 / 100 \\
deberta-v3-base  & 2e-3 & 32 & 1 & 0.2 & Cosine & 3 & 512 & 10 / 100 \\
\hline
\end{tabular}%
}
\caption{Hyperparameter settings for the LCP dataset for each evaluated model.}
\label{tab:lcp_hyperparams}
\end{table*}

\begin{table*}[htbp]
\centering
\resizebox{\textwidth}{!}{%
\begin{tabular}{l|l|l|l|l|l|l|l|l}
\hline
\rowcolor{gray!20}
\textbf{Model} & \textbf{Learning Rate} & \textbf{Batch Size} & \textbf{Grad Accum} & \textbf{Weight Decay} & \textbf{LR Scheduler} & \textbf{Rank} & \textbf{Max Length} & \textbf{Epochs / Warmup Steps} \\
\hline
SmolLM-360M     & 1e-4 & 8  & 2 & 0.1  & Cosine & 3 & 512 & 3 / 500 \\
MobileLLM-350M  & 1e-4 & 8  & 2 & 0.1  & Cosine & 3 & 512 & 3 / 500 \\
SmolLM-135M     & 1e-4 & 8  & 2 & 0.1  & Cosine & 3 & 512 & 3 / 500 \\
ModernBERT-base & 1e-4 & 8  & 2 & 0.1  & Cosine & 3 & 512 & 3 / 500 \\
deberta-v3-base  & 1e-4 & 16 & 4 & 0.00 & Cosine & 3 & 512 & 3 / 100 \\
roberta-base     & 1e-4 & 16 & 4 & 0.00 & Cosine & 3 & 512 & 3 / 100 \\
roberta-large    & 1e-4 & 16 & 4 & 0.00 & Cosine & 3 & 512 & 3 / 100 \\
deberta-v3-large & 1e-4 & 16 & 4 & 0.00 & Cosine & 3 & 512 & 3 / 100 \\
Mobile-llm-125   & 1e-4 & 16 & 4 & 0.00 & Cosine & 3 & 512 & 3 / 100 \\
Mobile-llm-630   & 1e-4 & 16 & 4 & 0.00 & Cosine & 3 & 512 & 3 / 100 \\
moden-bert-large & 1e-4 & 16 & 4 & 0.00 & Cosine & 3 & 512 & 3 / 100 \\
GPT2-medium      & 1e-4 & 8  & 2 & 0.1  & Cosine & 3 & 512 & 3 / 500 \\
GPT2-large       & 3e-3 & 32 & 1 & 0.00 & Cosine & 3 & 512 & 2 / 100 \\
\hline
\end{tabular}%
}
\caption{Hyperparameter settings for the SST-2 dataset for each evaluated model.}
\label{tab:sst2_hyperparams}
\end{table*}

\begin{table*}[htbp]
\centering
\resizebox{\textwidth}{!}{%
\begin{tabular}{l|l|l|l|l|l|l|l|l}
\hline
\rowcolor{gray!20}
\textbf{Model} & \textbf{Learning Rate} & \textbf{Batch Size} & \textbf{Grad Accum} & \textbf{Weight Decay} & \textbf{LR Scheduler} & \textbf{Rank} & \textbf{Max Length} & \textbf{Epochs / Warmup Steps} \\
\hline
SmolLM-360M     & 5e-4 & 4  & 4 & 0.1  & Cosine & 3 & 512 & 10 / 100 \\
MobileLLM-350M  & 5e-4 & 4  & 4 & 0.1  & Cosine & 3 & 512 & 10 / 100 \\
SmolLM-135M     & 5e-4 & 4  & 4 & 0.1  & Cosine & 3 & 512 & 10 / 100 \\
ModernBERT-base & 5e-4 & 4  & 4 & 0.1  & Cosine & 3 & 512 & 10 / 100 \\
deberta-v3-base  & 1e-3 & 64 & 1 & 0.00 & Cosine & 3 & 512 & 10 / 100 \\
roberta-base     & 1e-3 & 64 & 1 & 0.00 & Cosine & 3 & 512 & 10 / 100 \\
roberta-large    & 1e-3 & 64 & 1 & 0.00 & Cosine & 3 & 512 & 10 / 100 \\
deberta-v3-large & 1e-3 & 64 & 1 & 0.00 & Cosine & 3 & 512 & 10 / 100 \\
GPT2-medium      & 5e-4 & 4  & 4 & 0.1  & Cosine & 3 & 512 & 10 / 100 \\
GPT2-large       & 1e-4 & 16 & 2 & 0.00 & Cosine & 3 & 512 & 10 / 100 \\
Mobile-llm-125   & 3e-3 & 16 & 1 & 0.00 & Cosine & 3 & 512 & 5 / 100  \\
Mobile-llm-630   & 3e-3 & 16 & 1 & 0.00 & Cosine & 3 & 512 & 5 / 100  \\
moden-bert-large & 5e-4 & 4  & 4 & 0.1  & Cosine & 3 & 512 & 10 / 100 \\
\hline
\end{tabular}%
}
\caption{Hyperparameter settings for the MRPC dataset for each evaluated model.}
\label{tab:mrpc_hyperparams}
\end{table*}

\begin{table*}[htbp]
\centering
\resizebox{\textwidth}{!}{%
\begin{tabular}{l|l|l|l|l|l|l|l|l}
\hline
\rowcolor{gray!20}
\textbf{Model} & \textbf{Learning Rate} & \textbf{Batch Size} & \textbf{Grad Accum} & \textbf{Weight Decay} & \textbf{LR Scheduler} & \textbf{Rank} & \textbf{Max Length} & \textbf{Epochs / Warmup Steps} \\
\hline
SmolLM-360M     & 2e-4 & 8  & 2 & 0.1  & Cosine & 3 & 512 & 2 / 500 \\
MobileLLM-350M  & 2e-4 & 8  & 2 & 0.1  & Cosine & 3 & 512 & 2 / 500 \\
SmolLM-135M     & 2e-4 & 8  & 2 & 0.1  & Cosine & 3 & 512 & 2 / 500 \\
ModernBERT-base & 2e-4 & 8  & 2 & 0.1  & Cosine & 3 & 512 & 2 / 500 \\
GPT2-medium      & 2e-4 & 8  & 2 & 0.1  & Cosine & 3 & 512 & 2 / 500 \\
GPT2-large       & 1e-4 & 12 & 4 & 0.00 & Cosine & 3 & 512 & 2 / 100 \\
deberta-v3-base  & 1e-4 & 12 & 4 & 0.00 & Cosine & 3 & 512 & 2 / 100 \\
roberta-base     & 1e-4 & 12 & 4 & 0.00 & Cosine & 3 & 512 & 2 / 100 \\
roberta-large    & 1e-4 & 12 & 4 & 0.00 & Cosine & 3 & 512 & 2 / 100 \\
deberta-v3-large & 1e-4 & 12 & 4 & 0.00 & Cosine & 3 & 512 & 2 / 100 \\
Mobile-llm-125   & 1e-4 & 12 & 4 & 0.00 & Cosine & 3 & 512 & 2 / 100 \\
Mobile-llm-630   & 1e-4 & 12 & 4 & 0.00 & Cosine & 3 & 512 & 2 / 100 \\
moden-bert-large & 1e-4 & 12 & 4 & 0.00 & Cosine & 3 & 512 & 2 / 100 \\
\hline
\end{tabular}%
}
\caption{Hyperparameter settings for the QNLI dataset for each evaluated model.}
\label{tab:qnli_hyperparams}
\end{table*}

\begin{table*}[htbp]
\centering
\resizebox{\textwidth}{!}{%
\begin{tabular}{l|l|l|l|l|l|l|l|l}
\hline
\rowcolor{gray!20}
\textbf{Model} & \textbf{Learning Rate} & \textbf{Batch Size} & \textbf{Grad Accum} & \textbf{Weight Decay} & \textbf{LR Scheduler} & \textbf{Rank} & \textbf{Max Length} & \textbf{Epochs / Warmup Steps} \\
\hline
SmolLM-360M     & 1e-4 & 4  & 8 & 0.00 & Cosine & 3 & 512 & 30 / 100 \\
MobileLLM-350M  & 1e-4 & 4  & 8 & 0.00 & Cosine & 3 & 512 & 30 / 100 \\
SmolLM-135M     & 1e-4 & 4  & 8 & 0.00 & Cosine & 3 & 512 & 30 / 100 \\
ModernBERT-base & 1e-4 & 4  & 8 & 0.00 & Cosine & 3 & 512 & 30 / 100 \\
GPT2-medium      & 1e-4 & 4  & 8 & 0.00 & Cosine & 3 & 512 & 30 / 100 \\
GPT2-large       & 1e-3 & 16 & 2 & 0.00 & Cosine & 3 & 512 & 30 / 100 \\
deberta-v3-base  & 1e-4 & 16 & 8 & 0.00 & Cosine & 3 & 512 & 30 / 100 \\
roberta-base     & 1e-4 & 16 & 8 & 0.00 & Cosine & 3 & 512 & 30 / 100 \\
roberta-large    & 1e-4 & 16 & 8 & 0.00 & Cosine & 3 & 512 & 30 / 100 \\
deberta-v3-large & 1e-4 & 16 & 8 & 0.00 & Cosine & 3 & 512 & 30 / 100 \\
Mobile-llm-125   & 1e-4 & 16 & 8 & 0.00 & Cosine & 3 & 512 & 30 / 100 \\
Mobile-llm-630   & 1e-4 & 16 & 8 & 0.00 & Cosine & 3 & 512 & 30 / 100 \\
moden-bert-large & 1e-4 & 16 & 8 & 0.00 & Cosine & 3 & 512 & 30 / 100 \\
\hline
\end{tabular}%
}
\caption{Hyperparameter settings for the RTE dataset for each evaluated model.}
\label{tab:rte_hyperparams}
\end{table*}

\begin{table*}[htbp]
\centering
\resizebox{\textwidth}{!}{%
\begin{tabular}{l|l|l|l|l|l|l|l|l}
\hline
\rowcolor{gray!20}
\textbf{Model} & \textbf{Learning Rate} & \textbf{Batch Size} & \textbf{Grad Accum} & \textbf{Weight Decay} & \textbf{LR Scheduler} & \textbf{Rank} & \textbf{Max Length} & \textbf{Epochs / Warmup Steps} \\
\hline
SmolLM-360M     & 2e-5 & 8  & 1 & 0.1  & Cosine & 3 & 512 & 10 / 500 \\
MobileLLM-350M  & 2e-5 & 8  & 1 & 0.1  & Cosine & 3 & 512 & 10 / 500 \\
SmolLM-135M     & 2e-5 & 8  & 1 & 0.1  & Cosine & 3 & 512 & 10 / 500 \\
ModernBERT-base & 2e-5 & 8  & 1 & 0.1  & Cosine & 3 & 512 & 10 / 500 \\
GPT2-medium      & 2e-5 & 8  & 1 & 0.1  & Cosine & 3 & 512 & 10 / 500 \\
GPT2-large       & 1e-3 & 64 & 1 & 0.00 & Cosine & 3 & 512 & 10 / 100 \\
deberta-v3-base  & 2e-5 & 4  & 8 & 0.00 & Cosine & 3 & 512 & 10 / 100 \\
roberta-base     & 2e-5 & 4  & 8 & 0.00 & Cosine & 3 & 512 & 10 / 100 \\
roberta-large    & 2e-5 & 4  & 8 & 0.00 & Cosine & 3 & 512 & 10 / 100 \\
deberta-v3-large & 2e-5 & 4  & 8 & 0.00 & Cosine & 3 & 512 & 10 / 100 \\
Mobile-llm-125   & 5e-4 & 4  & 4 & 0.1  & Cosine & 3 & 512 & 10 / 100 \\
Mobile-llm-630   & 5e-4 & 4  & 4 & 0.1  & Cosine & 3 & 512 & 10 / 100 \\
moden-bert-large & 5e-4 & 4  & 4 & 0.1  & Cosine & 3 & 512 & 10 / 100 \\
\hline
\end{tabular}%
}
\caption{Hyperparameter settings for the COLA dataset for each evaluated model.}
\label{tab:cola_hyperparams}
\end{table*}

\begin{table*}[htbp]
\centering
\resizebox{\textwidth}{!}{%
\begin{tabular}{l|l|l|l|l|l|l|l|l}
\hline
\rowcolor{gray!20}
\textbf{Model} & \textbf{Learning Rate} & \textbf{Batch Size} & \textbf{Grad Accum} & \textbf{Weight Decay} & \textbf{LR Scheduler} & \textbf{Rank} & \textbf{Max Length} & \textbf{Epochs / Warmup Steps} \\
\hline
SmolLM-360M     & 2e-4 & 8  & 4 & 0.00 & Cosine & 3 & 512 & 2 / 500 \\
MobileLLM-350M  & 2e-4 & 8  & 4 & 0.00 & Cosine & 3 & 512 & 2 / 500 \\
SmolLM-135M     & 2e-4 & 8  & 4 & 0.00 & Cosine & 3 & 512 & 2 / 500 \\
ModernBERT-base & 2e-4 & 8  & 4 & 0.00 & Cosine & 3 & 512 & 2 / 500 \\
GPT2-medium      & 2e-4 & 8  & 4 & 0.00 & Cosine & 3 & 512 & 2 / 500 \\
GPT2-large       & 1e-3 & 32 & 1 & 0.00 & Cosine & 3 & 512 & 2 / 100 \\
deberta-v3-base  & 1e-3 & 14 & 1 & 0.00 & Cosine & 3 & 512 & 2 / 100 \\
roberta-base     & 1e-3 & 14 & 1 & 0.00 & Cosine & 3 & 512 & 2 / 100 \\
roberta-large    & 1e-3 & 14 & 1 & 0.00 & Cosine & 3 & 512 & 2 / 100 \\
deberta-v3-large & 1e-3 & 14 & 1 & 0.00 & Cosine & 3 & 512 & 2 / 100 \\
Mobile-llm-125   & 1e-3 & 14 & 1 & 0.00 & Cosine & 3 & 512 & 2 / 100 \\
Mobile-llm-630   & 1e-3 & 14 & 1 & 0.00 & Cosine & 3 & 512 & 2 / 100 \\
moden-bert-large & 2e-4 & 8  & 4 & 0.00 & Cosine & 3 & 512 & 2 / 500 \\
\hline
\end{tabular}%
}
\caption{Hyperparameter settings for the MNLI dataset for each evaluated model.}
\label{tab:mnli_hyperparams}
\end{table*}

\begin{table*}[htbp]
\centering
\resizebox{\textwidth}{!}{%
\begin{tabular}{l|l|l|l|l|l|l|l|l}
\hline
\rowcolor{gray!20}
\textbf{Model} & \textbf{Learning Rate} & \textbf{Batch Size} & \textbf{Grad Accum} & \textbf{Weight Decay} & \textbf{LR Scheduler} & \textbf{Rank} & \textbf{Max Length} & \textbf{Epochs / Warmup Steps} \\
\hline
ModernBERT-base  & 3e-4 & 128 & 1 & 0.00 & Cosine & 3 & 512 & 100 / 100 \\
MobileLLM-350M   & 3e-4 & 128 & 1 & 0.00 & Cosine & 3 & 512 & 100 / 100 \\
SmolLM-360M      & 3e-4 & 128 & 1 & 0.00 & Cosine & 3 & 512 & 100 / 100 \\
SmolLM-135M      & 3e-4 & 128 & 1 & 0.00 & Cosine & 3 & 512 & 100 / 100 \\
GPT2-medium       & 3e-4 & 128 & 1 & 0.00 & Cosine & 3 & 512 & 100 / 100 \\
GPT2-large        & 3e-4 & 128 & 1 & 0.00 & Cosine & 3 & 512 & 100 / 100 \\
deberta-v3-base   & 3e-4 & 128 & 1 & 0.00 & Cosine & 3 & 512 & 100 / 100 \\
roberta-base      & 3e-4 & 128 & 1 & 0.00 & Cosine & 3 & 512 & 100 / 100 \\
roberta-large     & 3e-4 & 128 & 1 & 0.00 & Cosine & 3 & 512 & 100 / 100 \\
deberta-v3-large  & 3e-4 & 128 & 1 & 0.00 & Cosine & 3 & 512 & 100 / 100 \\
Mobile-llm-125    & 3e-4 & 128 & 1 & 0.00 & Cosine & 3 & 512 & 100 / 100 \\
Mobile-llm-630    & 3e-4 & 128 & 1 & 0.00 & Cosine & 3 & 512 & 100 / 100 \\
moden-bert-large  & 3e-4 & 128 & 1 & 0.00 & Cosine & 3 & 512 & 100 / 100 \\
\hline
\end{tabular}%
}
\caption{Hyperparameter settings for the BoolQ dataset for each evaluated model.}
\label{tab:boolqa_hyperparams}
\end{table*}

\begin{table*}[htbp]
\centering
\resizebox{\textwidth}{!}{%
\begin{tabular}{l|l|l|l|l|l|l|l|l}
\hline
\rowcolor{gray!20}
\textbf{Model} & \textbf{Learning Rate} & \textbf{Batch Size} & \textbf{Grad Accum} & \textbf{Weight Decay} & \textbf{LR Scheduler} & \textbf{Rank} & \textbf{Max Length} & \textbf{Epochs / Warmup Steps} \\
\hline
deberta-v3-base   & 1e-4 & 16 & 1 & 0.00 & Cosine & 3 & 512 & 12 / 100 \\
mobilellm-350M    & 1e-4 & 16 & 1 & 0.00 & Cosine & 3 & 512 & 12 / 100 \\
SmolLM-360M       & 1e-4 & 16 & 1 & 0.00 & Cosine & 3 & 512 & 12 / 100 \\
SmolLM-135M       & 1e-4 & 16 & 1 & 0.00 & Cosine & 3 & 512 & 12 / 100 \\
ModernBERT-base   & 1e-4 & 16 & 1 & 0.00 & Cosine & 3 & 512 & 12 / 100 \\
GPT2-medium       & 1e-4 & 16 & 1 & 0.00 & Cosine & 3 & 512 & 12 / 100 \\
GPT2-large        & 1e-4 & 16 & 1 & 0.00 & Cosine & 3 & 512 & 12 / 100 \\
roberta-base      & 1e-4 & 16 & 1 & 0.00 & Cosine & 3 & 512 & 12 / 100 \\
roberta-large     & 1e-4 & 16 & 1 & 0.00 & Cosine & 3 & 512 & 12 / 100 \\
deberta-v3-large  & 1e-4 & 16 & 1 & 0.00 & Cosine & 3 & 512 & 12 / 100 \\
Mobile-llm-125    & 1e-4 & 16 & 1 & 0.00 & Cosine & 3 & 512 & 12 / 100 \\
Mobile-llm-630    & 1e-4 & 16 & 1 & 0.00 & Cosine & 3 & 512 & 12 / 100 \\
moden-bert-large  & 1e-4 & 16 & 1 & 0.00 & Cosine & 3 & 512 & 12 / 100 \\
\hline
\end{tabular}%
}
\caption{Hyperparameter settings for the HellaSwag dataset for each evaluated model.}
\label{tab:hellaswag_hyperparams}
\end{table*}

\begin{table*}[htbp]
\centering
\resizebox{\textwidth}{!}{%
\begin{tabular}{l|l|l|l|l|l|l|l|l}
\hline
\rowcolor{gray!20}
\textbf{Model} & \textbf{Learning Rate} & \textbf{Batch Size} & \textbf{Grad Accum} & \textbf{Weight Decay} & \textbf{LR Scheduler} & \textbf{Rank} & \textbf{Max Length} & \textbf{Epochs / Warmup Steps} \\
\hline
deberta-v3-base   & 3e-4 & 16 & 1 & 0.00 & Cosine & 3 & 512 & 4 / 100 \\
mobilellm-350M    & 3e-4 & 16 & 1 & 0.00 & Cosine & 3 & 512 & 4 / 100 \\
SmolLM-360M       & 3e-4 & 16 & 1 & 0.00 & Cosine & 3 & 512 & 4 / 100 \\
SmolLM-135M       & 3e-4 & 16 & 1 & 0.00 & Cosine & 3 & 512 & 4 / 100 \\
ModernBERT-base   & 3e-4 & 16 & 1 & 0.00 & Cosine & 3 & 512 & 4 / 100 \\
GPT2-medium       & 3e-4 & 16 & 1 & 0.00 & Cosine & 3 & 512 & 4 / 100 \\
GPT2-large        & 3e-4 & 16 & 1 & 0.00 & Cosine & 3 & 512 & 4 / 100 \\
roberta-base      & 3e-4 & 16 & 1 & 0.00 & Cosine & 3 & 512 & 4 / 100 \\
roberta-large     & 3e-4 & 16 & 1 & 0.00 & Cosine & 3 & 512 & 4 / 100 \\
deberta-v3-large  & 3e-4 & 16 & 1 & 0.00 & Cosine & 3 & 512 & 4 / 100 \\
Mobile-llm-125    & 3e-4 & 16 & 1 & 0.00 & Cosine & 3 & 512 & 4 / 100 \\
Mobile-llm-630    & 3e-4 & 16 & 1 & 0.00 & Cosine & 3 & 512 & 4 / 100 \\
moden-bert-large  & 3e-4 & 16 & 1 & 0.00 & Cosine & 3 & 512 & 4 / 100 \\
\hline
\end{tabular}%
}
\caption{Hyperparameter settings for the SIQA dataset for each evaluated model.}
\label{tab:siqa_hyperparams}
\end{table*}

\section{Model Profile Information}
\label{sec:model_profile_info}

We conduct a comprehensive CPU profiling analysis of twelve transformer models to understand the computational bottlenecks and runtime behavior that influence performance. The models we evaluate include DeBERTa-v3-Base~\autoref{tab:deberta_v3_base_profile}, DeBERTa-v3-Large~\autoref{tab:deberta_v3_large_profile}, RoBERTa-Base~\autoref{tab:roberta_base_profile}, RoBERTa-Large~\autoref{tab:roberta_large_profile}, ModernBERT-Base~\autoref{tab:modernbertbase_profile}, ModernBERT-Large~\autoref{tab:modernbertlarge_profile}, GPT-2 Medium~\autoref{tab:gpt2medium_profile}, GPT-2 Large~\autoref{tab:gpt2large_profile}, SmolLM-135M~\autoref{tab:smollm135m_profile}, SmolLM-360M~\autoref{tab:smollm360m_profile}, MobileLLM-125M~\autoref{tab:mobilellm125m_profile}, and MobileLLM-600M~\autoref{tab:mobilellm600m_profile}. Our CPU profiling shows that bidirectional models are often comparable to unidirectional models. For example, DeBERTa-v3-Base~\autoref{tab:deberta_v3_base_profile} and ModernBERT-Base~\autoref{tab:modernbertbase_profile} complete inference in 502ms and 347ms, respectively, while GPT-2 Medium~\autoref{tab:gpt2medium_profile} takes 1126ms—more than double the time. Larger bidirectional models like DeBERTa-v3-Large~\autoref{tab:deberta_v3_large_profile} and RoBERTa-Large~\autoref{tab:roberta_large_profile} have runtimes comparable to GPT-2 Large~\autoref{tab:gpt2large_profile} in total execution time and compute distribution. Bidirectional models spread CPU usage more evenly across attention, normalization, and embedding layers, whereas unidirectional models spend over 85\% of their time on \texttt{addmm}, suggesting less efficient resource utilization. Additionally, compact bidirectional models like SmolLM-135M~\autoref{tab:smollm135m_profile} and MobileLLM-125M~\autoref{tab:mobilellm125m_profile} show runtimes similar to GPT-2 Medium, indicating that this efficiency advantage holds even at smaller scales.

\begin{table}[h!]
\centering
\resizebox{\textwidth}{!}{%
\begin{tabular}{l|c|c|c|c|c|c}
\hline
\rowcolor{gray!20}
\textbf{Name} & \textbf{Self CPU \%} & \textbf{Self CPU} & \textbf{CPU total \%} & \textbf{CPU total} & \textbf{CPU time avg} & \textbf{\# of Calls} \\
\hline
\rowcolor[HTML]{DFFFD6} aten::linear & 0.51\% & 2.580ms & 77.29\% & 388.420ms & 4.046ms & 96 \\
\rowcolor[HTML]{DFFFD6} aten::addmm & 74.66\% & 375.212ms & 76.25\% & 383.177ms & 3.991ms & 96 \\
\rowcolor[HTML]{DFFFD6} aten::matmul & 0.27\% & 1.333ms & 8.83\% & 44.372ms & 924.422µs & 48 \\
\rowcolor[HTML]{DFFFD6} aten::bmm & 8.25\% & 41.477ms & 8.26\% & 41.502ms & 864.622µs & 48 \\
\rowcolor[HTML]{DFFFD6} aten::copy\_ & 4.84\% & 24.308ms & 4.84\% & 24.308ms & 79.180µs & 307 \\
\rowcolor[HTML]{DFFFD6} aten::gather & 2.73\% & 13.696ms & 2.73\% & 13.696ms & 570.650µs & 24 \\
\rowcolor[HTML]{DFFFD6} aten::clone & 0.12\% & 618.044µs & 2.26\% & 11.360ms & 135.242µs & 84 \\
\rowcolor[HTML]{DFFFD6} aten::contiguous & 0.04\% & 207.146µs & 2.08\% & 10.476ms & 145.499µs & 72 \\
\rowcolor[HTML]{DFFFD6} aten::repeat & 0.12\% & 586.012µs & 1.62\% & 8.156ms & 339.848µs & 24 \\
\rowcolor[HTML]{DFFFD6} aten::add & 1.17\% & 5.887ms & 1.22\% & 6.136ms & 84.054µs & 73 \\
\hline
\multicolumn{7}{l}{Self CPU time total: 502.528ms} \\
\hline
\end{tabular}
}
\caption{CPU profiling results for DeBERTa-v3-Base showing operation-wise breakdown of computation time.}
\label{tab:deberta_v3_base_profile}
\end{table}

\begin{table}[htbp]
\centering
\resizebox{\textwidth}{!}{%
\begin{tabular}{l|c|c|c|c|c|c}
\hline
\rowcolor{gray!20}
\textbf{Name} & \textbf{Self CPU \%} & \textbf{Self CPU} & \textbf{CPU total \%} & \textbf{CPU total} & \textbf{CPU time avg} & \textbf{\# of Calls} \\
\hline
\rowcolor[HTML]{DFFFD6} aten::linear & 0.30\% & 4.865ms & 82.66\% & 1.329s & 6.921ms & 192 \\
\rowcolor[HTML]{DFFFD6} aten::addmm & 80.79\% & 1.299s & 82.08\% & 1.319s & 6.872ms & 192 \\
\rowcolor[HTML]{DFFFD6} aten::matmul & 0.15\% & 2.466ms & 7.37\% & 118.530ms & 1.235ms & 96 \\
\rowcolor[HTML]{DFFFD6} aten::bmm & 7.03\% & 113.072ms & 7.04\% & 113.118ms & 1.178ms & 96 \\
\rowcolor[HTML]{DFFFD6} aten::copy\_ & 3.91\% & 62.848ms & 3.91\% & 62.848ms & 103.539µs & 607 \\
\rowcolor[HTML]{DFFFD6} aten::gather & 2.17\% & 34.856ms & 2.17\% & 34.856ms & 726.164µs & 48 \\
\rowcolor[HTML]{DFFFD6} aten::clone & 0.07\% & 1.160ms & 1.78\% & 28.664ms & 170.619µs & 168 \\
\rowcolor[HTML]{DFFFD6} aten::contiguous & 0.03\% & 443.678µs & 1.63\% & 26.265ms & 182.397µs & 144 \\
\rowcolor[HTML]{DFFFD6} aten::repeat & 0.08\% & 1.258ms & 1.23\% & 19.738ms & 411.214µs & 48 \\
\rowcolor[HTML]{DFFFD6} aten::add & 0.88\% & 14.152ms & 0.91\% & 14.626ms & 100.871µs & 145 \\
\hline
\multicolumn{7}{l}{Self CPU time total: 1608ms} \\
\hline
\end{tabular}
}
\caption{CPU profiling results for DeBERTa-v3-Large showing operation-wise breakdown of computation time.}
\label{tab:deberta_v3_large_profile}
\end{table}

\begin{table}[htbp]
\centering
\resizebox{\textwidth}{!}{%
\begin{tabular}{l|c|c|c|c|c|c}
\hline
\rowcolor{gray!20}
\textbf{Name} & \textbf{Self CPU \%} & \textbf{Self CPU} & \textbf{CPU total \%} & \textbf{CPU total} & \textbf{CPU time avg} & \textbf{\# of Calls} \\
\hline
\rowcolor[HTML]{DFFFD6} aten::linear & 0.22\% & 2.579ms & 92.35\% & 1.079s & 14.774ms & 73 \\
\rowcolor[HTML]{DFFFD6} aten::addmm & 91.46\% & 1.068s & 91.93\% & 1.074s & 14.706ms & 73 \\
\rowcolor[HTML]{DFFFD6} aten::scaled\_dot\_product\_attention & 0.02\% & 187.093µs & 5.13\% & 59.890ms & 4.991ms & 12 \\
\rowcolor[HTML]{DFFFD6} aten::\_scaled\_dot\_product\_flash\_attention\_for\_cpu & 5.04\% & 58.850ms & 5.11\% & 59.703ms & 4.975ms & 12 \\
\rowcolor[HTML]{DFFFD6} aten::gelu & 1.15\% & 13.426ms & 1.15\% & 13.426ms & 1.119ms & 12 \\
\rowcolor[HTML]{DFFFD6} aten::layer\_norm & 0.03\% & 356.267µs & 0.74\% & 8.673ms & 346.936µs & 25 \\
\rowcolor[HTML]{DFFFD6} aten::native\_layer\_norm & 0.67\% & 7.832ms & 0.71\% & 8.317ms & 332.685µs & 25 \\
\rowcolor[HTML]{DFFFD6} aten::copy\_ & 0.42\% & 4.888ms & 0.42\% & 4.888ms & 61.871µs & 79 \\
\rowcolor[HTML]{DFFFD6} aten::add & 0.25\% & 2.868ms & 0.25\% & 2.878ms & 106.586µs & 27 \\
\rowcolor[HTML]{DFFFD6} aten::ne & 0.14\% & 1.675ms & 0.14\% & 1.675ms & 1.675ms & 1 \\
\hline
\multicolumn{7}{l}{Self CPU time total: 1168ms} \\
\hline
\end{tabular}
}
\caption{CPU profiling results for RoBERTa-Base showing operation-wise breakdown of computation time.}
\label{tab:roberta_base_profile}

\end{table}

\begin{table}[htbp]
\centering
\resizebox{\textwidth}{!}{%
\begin{tabular}{l|c|c|c|c|c|c}
\hline
\rowcolor{gray!20}
\textbf{Name} & \textbf{Self CPU \%} & \textbf{Self CPU} & \textbf{CPU total \%} & \textbf{CPU total} & \textbf{CPU time avg} & \textbf{\# of Calls} \\
\hline
\rowcolor[HTML]{DFFFD6} aten::linear & 0.39\% & 4.022ms & 94.22\% & 982.099ms & 6.773ms & 145 \\
\rowcolor[HTML]{DFFFD6} aten::addmm & 92.45\% & 963.703ms & 93.46\% & 974.219ms & 6.719ms & 145 \\
\rowcolor[HTML]{DFFFD6} aten::scaled\_dot\_product\_attention & 0.03\% & 304.568µs & 3.29\% & 34.249ms & 1.427ms & 24 \\
\rowcolor[HTML]{DFFFD6} aten::\_scaled\_dot\_product\_flash\_attention\_for\_cpu & 3.13\% & 32.634ms & 3.26\% & 33.945ms & 1.414ms & 24 \\
\rowcolor[HTML]{DFFFD6} aten::gelu & 1.00\% & 10.469ms & 1.00\% & 10.469ms & 436.198µs & 24 \\
\rowcolor[HTML]{DFFFD6} aten::copy\_ & 0.93\% & 9.662ms & 0.93\% & 9.662ms & 63.987µs & 151 \\
\rowcolor[HTML]{DFFFD6} aten::layer\_norm & 0.04\% & 434.620µs & 0.75\% & 7.775ms & 158.670µs & 49 \\
\rowcolor[HTML]{DFFFD6} aten::native\_layer\_norm & 0.63\% & 6.605ms & 0.70\% & 7.340ms & 149.800µs & 49 \\
\rowcolor[HTML]{DFFFD6} aten::add & 0.45\% & 4.657ms & 0.45\% & 4.670ms & 91.559µs & 51 \\
\rowcolor[HTML]{DFFFD6} aten::view & 0.22\% & 2.325ms & 0.22\% & 2.325ms & 4.754µs & 489 \\
\hline
\multicolumn{7}{l}{Self CPU time total: 1042ms} \\
\hline
\end{tabular}
}
\caption{CPU profiling results for RoBERTa-Large showing operation-wise breakdown of computation time.}
\label{tab:roberta_large_profile}
\end{table}

\begin{table}[htbp]
\centering
\resizebox{\textwidth}{!}{%
\begin{tabular}{l|c|c|c|c|c|c}
\hline
\rowcolor{gray!20}
\textbf{Name} & \textbf{Self CPU \%} & \textbf{Self CPU} & \textbf{CPU total \%} & \textbf{CPU total} & \textbf{CPU time avg} & \textbf{\# of Calls} \\
\hline
\rowcolor[HTML]{DFFFD6} aten::linear & 0.15\% & 532.099µs & 81.11\% & 282.061ms & 3.205ms & 88 \\
\rowcolor[HTML]{DFFFD6} aten::matmul & 0.62\% & 2.164ms & 81.03\% & 281.778ms & 2.562ms & 110 \\
\rowcolor[HTML]{DFFFD6} aten::mm & 79.88\% & 277.768ms & 79.89\% & 277.814ms & 3.157ms & 88 \\
\rowcolor[HTML]{DFFFD6} aten::scaled\_dot\_product\_attention & 0.07\% & 230.328µs & 6.25\% & 21.748ms & 988.565µs & 22 \\
\rowcolor[HTML]{DFFFD6} aten::\_scaled\_dot\_product\_flash\_attention\_for\_cpu & 5.85\% & 20.351ms & 6.19\% & 21.518ms & 978.096µs & 22 \\
\rowcolor[HTML]{DFFFD6} aten::layer\_norm & 0.13\% & 462.996µs & 2.60\% & 9.037ms & 200.831µs & 45 \\
\rowcolor[HTML]{DFFFD6} aten::native\_layer\_norm & 2.28\% & 7.919ms & 2.47\% & 8.574ms & 190.542µs & 45 \\
\rowcolor[HTML]{DFFFD6} aten::mul & 2.17\% & 7.550ms & 2.35\% & 8.189ms & 53.177µs & 154 \\
\rowcolor[HTML]{DFFFD6} aten::add & 1.82\% & 6.327ms & 1.82\% & 6.327ms & 71.901µs & 88 \\
\rowcolor[HTML]{DFFFD6} aten::gelu & 1.40\% & 4.852ms & 1.40\% & 4.852ms & 220.545µs & 22 \\
\hline
\multicolumn{7}{l}{Self CPU time total: 347.749ms} \\
\hline
\end{tabular}
}
\caption{CPU profiling results for ModernBERT-Base showing operation-wise breakdown of computation time.}
\label{tab:modernbertbase_profile}
\end{table}

\begin{table}[htbp]
\centering
\resizebox{\textwidth}{!}{%
\begin{tabular}{l|c|c|c|c|c|c}
\hline
\rowcolor{gray!20}
\textbf{Name} & \textbf{Self CPU \%} & \textbf{Self CPU} & \textbf{CPU total \%} & \textbf{CPU total} & \textbf{CPU time avg} & \textbf{\# of Calls} \\
\hline
\rowcolor[HTML]{DFFFD6} aten::linear & 0.03\% & 818.323µs & 81.17\% & 2.223s & 19.850ms & 112 \\
\rowcolor[HTML]{DFFFD6} aten::matmul & 0.14\% & 3.970ms & 81.15\% & 2.223s & 15.876ms & 140 \\
\rowcolor[HTML]{DFFFD6} aten::mm & 80.90\% & 2.216s & 80.90\% & 2.216s & 19.785ms & 112 \\
\rowcolor[HTML]{DFFFD6} aten::embedding & 0.00\% & 61.446µs & 12.23\% & 335.032ms & 335.032ms & 1 \\
\rowcolor[HTML]{DFFFD6} aten::index\_select & 12.23\% & 334.935ms & 12.23\% & 334.953ms & 334.953ms & 1 \\
\rowcolor[HTML]{DFFFD6} aten::layer\_norm & 0.02\% & 470.737µs & 2.22\% & 60.931ms & 1.069ms & 57 \\
\rowcolor[HTML]{DFFFD6} aten::native\_layer\_norm & 2.18\% & 59.590ms & 2.21\% & 60.460ms & 1.061ms & 57 \\
\rowcolor[HTML]{DFFFD6} aten::scaled\_dot\_product\_attention & 0.02\% & 564.994µs & 1.45\% & 39.851ms & 1.423ms & 28 \\
\rowcolor[HTML]{DFFFD6} aten::\_scaled\_dot\_product\_flash\_attention\_for\_cpu & 1.38\% & 37.714ms & 1.43\% & 39.286ms & 1.403ms & 28 \\
\rowcolor[HTML]{DFFFD6} aten::gelu & 0.89\% & 24.332ms & 0.89\% & 24.332ms & 868.986µs & 28 \\
\hline
\multicolumn{7}{l}{Self CPU time total: 2739ms} \\
\hline
\end{tabular}
}
\caption{CPU profiling results for ModernBERT-large showing operation-wise breakdown of computation time.}
\label{tab:modernbertlarge_profile}
\end{table}

\begin{table}[htbp]
\centering
\resizebox{\textwidth}{!}{%
\begin{tabular}{l|c|c|c|c|c|c}
\hline
\rowcolor{gray!20}
\textbf{Name} & \textbf{Self CPU \%} & \textbf{Self CPU} & \textbf{CPU total \%} & \textbf{CPU total} & \textbf{CPU time avg} & \textbf{\# of Calls} \\
\hline
\rowcolor[HTML]{FFF5CC} aten::addmm & 86.77\% & 976.892ms & 88.05\% & 991.390ms & 10.327ms & 96 \\
\rowcolor[HTML]{FFF5CC} aten::mul & 3.18\% & 35.802ms & 3.35\% & 37.679ms & 392.489µs & 96 \\
\rowcolor[HTML]{FFF5CC} aten::scaled\_dot\_product\_attention & 0.04\% & 396.746µs & 2.76\% & 31.048ms & 1.294ms & 24 \\
\rowcolor[HTML]{FFF5CC} aten::\_scaled\_dot\_product\_flash\_attention\_for\_cpu & 2.60\% & 29.255ms & 2.72\% & 30.652ms & 1.277ms & 24 \\
\rowcolor[HTML]{FFF5CC} aten::copy\_ & 2.07\% & 23.295ms & 2.07\% & 23.295ms & 80.886µs & 288 \\
\rowcolor[HTML]{FFF5CC} aten::add & 1.95\% & 21.947ms & 1.99\% & 22.375ms & 230.671µs & 97 \\
\rowcolor[HTML]{FFF5CC} aten::contiguous & 0.03\% & 298.059µs & 1.01\% & 11.422ms & 118.983µs & 96 \\
\rowcolor[HTML]{FFF5CC} aten::clone & 0.07\% & 742.482µs & 0.99\% & 11.124ms & 115.879µs & 96 \\
\rowcolor[HTML]{FFF5CC} aten::pow & 0.87\% & 9.819ms & 0.88\% & 9.867ms & 411.125µs & 24 \\
\rowcolor[HTML]{FFF5CC} aten::tanh & 0.79\% & 8.921ms & 0.79\% & 8.921ms & 371.720µs & 24 \\
\hline
\multicolumn{7}{l}{Self CPU time total: 1126ms} \\
\hline
\end{tabular}
}
\caption{CPU profiling results for GPT-2 Medium showing operation-wise breakdown of computation time.}

\label{tab:gpt2medium_profile}
\end{table}

\begin{table}[htbp]
\centering
\resizebox{\textwidth}{!}{%
\begin{tabular}{l|c|c|c|c|c|c}
\hline
\rowcolor{gray!20}
\textbf{Name} & \textbf{Self CPU \%} & \textbf{Self CPU} & \textbf{CPU total \%} & \textbf{CPU total} & \textbf{CPU time avg} & \textbf{\# of Calls} \\
\hline
\rowcolor[HTML]{FFF5CC} aten::addmm & 87.92\% & 2.160s & 89.08\% & 2.188s & 15.196ms & 144 \\
\rowcolor[HTML]{FFF5CC} aten::mul & 2.84\% & 69.731ms & 2.98\% & 73.160ms & 508.058µs & 144 \\
\rowcolor[HTML]{FFF5CC} aten::scaled\_dot\_product\_attention & 0.02\% & 560.556µs & 2.74\% & 67.311ms & 1.870ms & 36 \\
\rowcolor[HTML]{FFF5CC} aten::\_scaled\_dot\_product\_flash\_attention\_for\_cpu & 2.63\% & 64.497ms & 2.72\% & 66.750ms & 1.854ms & 36 \\
\rowcolor[HTML]{FFF5CC} aten::copy\_ & 1.82\% & 44.776ms & 1.82\% & 44.776ms & 103.647µs & 432 \\
\rowcolor[HTML]{FFF5CC} aten::add & 1.77\% & 43.543ms & 1.80\% & 44.286ms & 305.422µs & 145 \\
\rowcolor[HTML]{FFF5CC} aten::contiguous & 0.02\% & 548.391µs & 0.87\% & 21.351ms & 148.269µs & 144 \\
\rowcolor[HTML]{FFF5CC} aten::clone & 0.06\% & 1.422ms & 0.85\% & 20.802ms & 144.461µs & 144 \\
\rowcolor[HTML]{FFF5CC} aten::pow & 0.81\% & 19.877ms & 0.81\% & 19.970ms & 554.714µs & 36 \\
\rowcolor[HTML]{FFF5CC} aten::tanh & 0.70\% & 17.260ms & 0.70\% & 17.260ms & 479.437µs & 36 \\
\hline
\multicolumn{7}{l}{Self CPU time total: 2456ms} \\
\hline
\end{tabular}
}
\caption{CPU profiling results for GPT-2 Large showing operation-wise breakdown of computation time.}
\label{tab:gpt2large_profile}
\end{table}

\begin{table}[htbp]
\centering
\resizebox{\textwidth}{!}{%
\begin{tabular}{l|c|c|c|c|c|c}
\hline
\rowcolor{gray!20}
\textbf{Name} & \textbf{Self CPU \%} & \textbf{Self CPU} & \textbf{CPU total \%} & \textbf{CPU total} & \textbf{CPU time avg} & \textbf{\# of Calls} \\
\hline
\rowcolor[HTML]{FFF5CC} aten::linear & 0.35\% & 1.889ms & 80.94\% & 441.637ms & 2.103ms & 210 \\
\rowcolor[HTML]{FFF5CC} aten::matmul & 1.44\% & 7.863ms & 79.89\% & 435.925ms & 2.066ms & 211 \\
\rowcolor[HTML]{FFF5CC} aten::mm & 77.90\% & 425.052ms & 77.93\% & 425.217ms & 2.025ms & 210 \\
\rowcolor[HTML]{FFF5CC} aten::scaled\_dot\_product\_attention & 0.07\% & 360.301µs & 6.26\% & 34.135ms & 1.138ms & 30 \\
\rowcolor[HTML]{FFF5CC} aten::\_scaled\_dot\_product\_flash\_attention\_for\_cpu & 5.84\% & 31.891ms & 6.19\% & 33.775ms & 1.126ms & 30 \\
\rowcolor[HTML]{FFF5CC} aten::mul & 2.73\% & 14.911ms & 2.74\% & 14.958ms & 54.590µs & 274 \\
\rowcolor[HTML]{FFF5CC} aten::clone & 0.18\% & 963.449µs & 1.87\% & 10.198ms & 84.981µs & 120 \\
\rowcolor[HTML]{FFF5CC} aten::copy\_ & 1.54\% & 8.398ms & 1.54\% & 8.398ms & 34.277µs & 245 \\
\rowcolor[HTML]{FFF5CC} aten::silu & 1.51\% & 8.256ms & 1.51\% & 8.256ms & 275.204µs & 30 \\
\rowcolor[HTML]{FFF5CC} aten::add & 1.29\% & 7.025ms & 1.48\% & 8.054ms & 44.496µs & 181 \\
\hline
\multicolumn{7}{l}{Self CPU time total: 545.639ms} \\
\hline
\end{tabular}
}
\caption{CPU profiling results for SmolLM-135M showing operation-wise breakdown of computation time.}
\label{tab:smollm135m_profile}

\end{table}

\begin{table}[htbp]
\centering
\resizebox{\textwidth}{!}{%
\begin{tabular}{l|c|c|c|c|c|c}
\hline
\rowcolor{gray!20}
\textbf{Name} & \textbf{Self CPU \%} & \textbf{Self CPU} & \textbf{CPU total \%} & \textbf{CPU total} & \textbf{CPU time avg} & \textbf{\# of Calls} \\
\hline
\rowcolor[HTML]{FFF5CC} aten::linear & 0.14\% & 1.401ms & 87.03\% & 895.172ms & 3.996ms & 224 \\
\rowcolor[HTML]{FFF5CC} aten::matmul & 0.44\% & 4.559ms & 86.59\% & 890.629ms & 3.958ms & 225 \\
\rowcolor[HTML]{FFF5CC} aten::mm & 85.92\% & 883.710ms & 85.93\% & 883.826ms & 3.946ms & 224 \\
\rowcolor[HTML]{FFF5CC} aten::scaled\_dot\_product\_attention & 0.18\% & 1.871ms & 3.82\% & 39.269ms & 1.227ms & 32 \\
\rowcolor[HTML]{FFF5CC} aten::\_scaled\_dot\_product\_flash\_attention\_for\_cpu & 3.49\% & 35.847ms & 3.64\% & 37.398ms & 1.169ms & 32 \\
\rowcolor[HTML]{FFF5CC} aten::mul & 2.46\% & 25.292ms & 2.46\% & 25.319ms & 86.708µs & 292 \\
\rowcolor[HTML]{FFF5CC} aten::silu & 1.36\% & 13.992ms & 1.36\% & 13.992ms & 437.260µs & 32 \\
\rowcolor[HTML]{FFF5CC} aten::add & 1.07\% & 11.014ms & 1.14\% & 11.728ms & 60.769µs & 193 \\
\rowcolor[HTML]{FFF5CC} aten::clone & 0.07\% & 706.630µs & 1.00\% & 10.261ms & 80.166µs & 128 \\
\rowcolor[HTML]{FFF5CC} aten::copy\_ & 0.87\% & 8.908ms & 0.87\% & 8.908ms & 34.131µs & 261 \\
\hline
\multicolumn{7}{l}{Self CPU time total: 1029ms} \\
\hline
\end{tabular}
}
\caption{CPU profiling results for SmolLM-360M showing operation-wise breakdown of computation time.}
\label{tab:smollm360m_profile}

\end{table}

\begin{table}[htbp]
\centering
\resizebox{\textwidth}{!}{%
\begin{tabular}{l|c|c|c|c|c|c|c}
\hline
\rowcolor{gray!20}
\textbf{Name} & & \textbf{Self CPU \%} & \textbf{Self CPU} & \textbf{CPU total \%} & \textbf{CPU total} & \textbf{CPU time avg} & \textbf{\# of Calls} \\
\hline
\rowcolor[HTML]{FFF5CC} aten::linear & & 0.15\% & 1.007ms & 87.11\% & 600.140ms & 2.844ms & 211 \\
\rowcolor[HTML]{FFF5CC} aten::matmul & & 0.52\% & 3.615ms & 86.62\% & 596.730ms & 2.815ms & 212 \\
\rowcolor[HTML]{FFF5CC} aten::mm & & 85.81\% & 591.196ms & 85.83\% & 591.306ms & 2.802ms & 211 \\
\rowcolor[HTML]{FFF5CC} aten::scaled\_dot\_product\_attention & & 0.06\% & 386.293µs & 4.25\% & 29.303ms & 976.771µs & 30 \\
\rowcolor[HTML]{FFF5CC} aten::\_scaled\_dot\_product\_flash\_attention\_for\_cpu & & 4.04\% & 27.832ms & 4.20\% & 28.917ms & 963.894µs & 30 \\
\rowcolor[HTML]{FFF5CC} aten::mul & & 2.28\% & 15.710ms & 2.29\% & 15.770ms & 57.554µs & 274 \\
\rowcolor[HTML]{FFF5CC} aten::silu & & 1.45\% & 9.993ms & 1.45\% & 9.993ms & 333.109µs & 30 \\
\rowcolor[HTML]{FFF5CC} aten::add & & 0.98\% & 6.723ms & 1.06\% & 7.271ms & 40.174µs & 181 \\
\rowcolor[HTML]{FFF5CC} aten::clone & & 0.09\% & 604.621µs & 0.91\% & 6.256ms & 52.131µs & 120 \\
\rowcolor[HTML]{FFF5CC} aten::copy\_ & & 0.76\% & 5.251ms & 0.76\% & 5.215ms & 21.432µs & 245 \\
\hline
\multicolumn{8}{l}{Self CPU time total: 688.943ms} \\
\hline
\end{tabular}
}
\caption{CPU profiling results for MobileLLM-125M showing operation-wise breakdown of computation time.}
\label{tab:mobilellm125m_profile}

\end{table}

\begin{table}[htbp]
\centering
\resizebox{\textwidth}{!}{%
\begin{tabular}{l|c|c|c|c|c|c}
\hline
\rowcolor{gray!20}
\textbf{Name} & \textbf{Self CPU \%} & \textbf{Self CPU} & \textbf{CPU total \%} & \textbf{CPU total} & \textbf{CPU time avg} & \textbf{\# of Calls} \\
\hline
\rowcolor[HTML]{FFF5CC} aten::linear & 0.14\% & 1.401ms & 87.03\% & 895.172ms & 3.996ms & 224 \\
\rowcolor[HTML]{FFF5CC} aten::matmul & 0.44\% & 4.559ms & 86.59\% & 890.629ms & 3.958ms & 225 \\
\rowcolor[HTML]{FFF5CC} aten::mm & 85.92\% & 883.710ms & 85.93\% & 883.826ms & 3.946ms & 224 \\
\rowcolor[HTML]{FFF5CC} aten::scaled\_dot\_product\_attention & 0.18\% & 1.871ms & 3.82\% & 39.269ms & 1.227ms & 32 \\
\rowcolor[HTML]{FFF5CC} aten::\_scaled\_dot\_product\_flash\_attention\_for\_cpu & 3.49\% & 35.847ms & 3.64\% & 37.398ms & 1.169ms & 32 \\
\rowcolor[HTML]{FFF5CC} aten::mul & 2.46\% & 25.292ms & 2.46\% & 25.319ms & 86.708µs & 292 \\
\rowcolor[HTML]{FFF5CC} aten::silu & 1.36\% & 13.992ms & 1.36\% & 13.992ms & 437.260µs & 32 \\
\rowcolor[HTML]{FFF5CC} aten::add & 1.07\% & 11.014ms & 1.14\% & 11.728ms & 60.769µs & 193 \\
\rowcolor[HTML]{FFF5CC} aten::clone & 0.07\% & 706.630µs & 1.00\% & 10.261ms & 80.166µs & 128 \\
\rowcolor[HTML]{FFF5CC} aten::copy\_ & 0.87\% & 8.908ms & 0.87\% & 8.908ms & 34.131µs & 261 \\
\hline
\multicolumn{7}{l}{Self CPU time total: 1029ms} \\
\hline
\end{tabular}
}
\caption{CPU profiling results for SmolLM-360M showing operation-wise breakdown of computation time.}
\label{tab:smollm360m_profile_2}

\end{table}

\begin{table}[h!]
\centering
\resizebox{\textwidth}{!}{%
\begin{tabular}{l|c|c|c|c|c|c}
\hline
\rowcolor{gray!20}
\textbf{Name} & \textbf{Self CPU \%} & \textbf{Self CPU} & \textbf{CPU total \%} & \textbf{CPU total} & \textbf{CPU time avg} & \textbf{\# of Calls} \\
\hline
\rowcolor[HTML]{FFF5CC} aten::linear & 0.10\% & 1.933ms & 90.92\% & 1.808s & 6.433ms & 281 \\
\rowcolor[HTML]{FFF5CC} aten::matmul & 0.30\% & 6.000ms & 90.62\% & 1.802s & 6.389ms & 282 \\
\rowcolor[HTML]{FFF5CC} aten::mm & 90.18\% & 1.793s & 90.18\% & 1.793s & 6.381ms & 281 \\
\rowcolor[HTML]{FFF5CC} aten::scaled\_dot\_product\_attention & 0.02\% & 431.170µs & 2.74\% & 54.424ms & 1.361ms & 40 \\
\rowcolor[HTML]{FFF5CC} aten::\_scaled\_dot\_product\_flash\_attention\_for\_cpu & 2.62\% & 52.116ms & 2.72\% & 53.992ms & 1.350ms & 40 \\
\rowcolor[HTML]{FFF5CC} aten::mul & 1.65\% & 32.805ms & 1.65\% & 32.838ms & 90.214µs & 364 \\
\rowcolor[HTML]{FFF5CC} aten::silu & 1.46\% & 28.972ms & 1.46\% & 28.972ms & 724.307µs & 40 \\
\rowcolor[HTML]{FFF5CC} aten::add & 0.77\% & 15.238ms & 0.81\% & 16.094ms & 66.778µs & 241 \\
\rowcolor[HTML]{FFF5CC} aten::clone & 0.05\% & 1.018ms & 0.65\% & 13.012ms & 81.323µs & 160 \\
\rowcolor[HTML]{FFF5CC} aten::copy\_ & 0.55\% & 10.926ms & 0.55\% & 10.926ms & 33.617µs & 325 \\
\hline
\multicolumn{7}{l}{Self CPU time total: 1988ms} \\
\hline
\end{tabular}
}
\caption{CPU profiling results for MobileLLM-600M showing operation-wise breakdown of computation time.}
\label{tab:mobilellm600m_profile}
\end{table}

\begin{table*}[htbp]
\centering
\resizebox{\textwidth}{!}{%
\begin{tabular}{l|l|l|c|c|c|c|c|c|c}
\rowcolor{gray!20}
\textbf{Model}     & \textbf{PEFT}    & \textbf{Method}    & \textbf{ WASSA } & \textbf{SICK} & \textbf{STSB} & \textbf{LCP} & \textbf{CRP} & \textbf{Humicroedit} & \textbf{Avg.}\\ \hline
\multirow{12}{*}{\textbf{Llama2-7B}} 
                    & \cellcolor[HTML]{EAF3FA}    LoRA              & \cellcolor[HTML]{EAF3FA}    Predictor           & \cellcolor[HTML]{EAF3FA}    0.454/0.151         & \cellcolor[HTML]{EAF3FA}    0.860/0.280         & \cellcolor[HTML]{EAF3FA}    0.965/0.950         & \cellcolor[HTML]{EAF3FA}    0.930/0.105              & \cellcolor[HTML]{EAF3FA}    1.014/0.784    & \cellcolor[HTML]{EAF3FA}    1.348/1.046  & \cellcolor[HTML]{EAF3FA}    0.928/0.553\\ 
                    & \cellcolor[HTML]{EAF3FA}                     & \cellcolor[HTML]{EAF3FA}    Generator           & \cellcolor[HTML]{EAF3FA}    0.090/0.023           & \cellcolor[HTML]{EAF3FA}    0.340/0.195         & \cellcolor[HTML]{EAF3FA}    0.610/0.630          & \cellcolor[HTML]{EAF3FA}    0.900/0.105              & \cellcolor[HTML]{EAF3FA}    0.465/0.349  & \cellcolor[HTML]{EAF3FA}    0.650/0.505  & \cellcolor[HTML]{EAF3FA}    0.509/0.301 \\ 
                    & \cellcolor[HTML]{EAF3FA}               & \cellcolor[HTML]{EAF3FA}    PredGen           & \cellcolor[HTML]{EAF3FA}    0.088/0.022           & \cellcolor[HTML]{EAF3FA}    0.320/0.190          & \cellcolor[HTML]{EAF3FA}    0.576/0.569         & \cellcolor[HTML]{EAF3FA}    0.062/0.008               & \cellcolor[HTML]{EAF3FA}    0.420/0.280  & \cellcolor[HTML]{EAF3FA}    0.550/0.455  & \cellcolor[HTML]{EAF3FA}    0.338/0.257\\ 
                    & \cellcolor[HTML]{FFE8D6}               & \cellcolor[HTML]{FFE8D6}    Generation*       & \cellcolor[HTML]{FFE8D6}    0.089/0.023           & \cellcolor[HTML]{FFE8D6}    0.315/0.192          & \cellcolor[HTML]{FFE8D6}    0.582/0.574           & \cellcolor[HTML]{FFE8D6}    0.065/0.009               & \cellcolor[HTML]{FFE8D6}    0.430/0.290  & \cellcolor[HTML]{FFE8D6}    0.548/0.457  & \cellcolor[HTML]{FFE8D6}    0.335/0.258\\

                    & \cellcolor[HTML]{EAF3FA}     AdaLoRA           & \cellcolor[HTML]{EAF3FA}    Predictor           & \cellcolor[HTML]{EAF3FA}    0.424/0.148         & \cellcolor[HTML]{EAF3FA}    0.845/0.270          & \cellcolor[HTML]{EAF3FA}    0.950/0.935        & \cellcolor[HTML]{EAF3FA}    0.918/0.100             & \cellcolor[HTML]{EAF3FA}    1.020/0.790  & \cellcolor[HTML]{EAF3FA}    1.360/1.050  & \cellcolor[HTML]{EAF3FA}    0.920/0.549 \\ 
                    & \cellcolor[HTML]{EAF3FA}            & \cellcolor[HTML]{EAF3FA}    Generator           & \cellcolor[HTML]{EAF3FA}    0.087/0.022         & \cellcolor[HTML]{EAF3FA}    0.325/0.185         & \cellcolor[HTML]{EAF3FA}    0.600/0.620         & \cellcolor[HTML]{EAF3FA}    0.890/0.097              & \cellcolor[HTML]{EAF3FA}    0.455/0.335     & \cellcolor[HTML]{EAF3FA}     0.630/0.490   & \cellcolor[HTML]{EAF3FA}    0.498/0.291 \\ 
                    & \cellcolor[HTML]{EAF3FA}        & \cellcolor[HTML]{EAF3FA}    PredGen            & \cellcolor[HTML]{EAF3FA}    0.080/0.020           & \cellcolor[HTML]{EAF3FA}    0.305/0.185        & \cellcolor[HTML]{EAF3FA}    0.575/0.570         & \cellcolor[HTML]{EAF3FA}    0.058/0.006               & \cellcolor[HTML]{EAF3FA}    0.405/0.270        & \cellcolor[HTML]{EAF3FA}     0.535/0.440  & \cellcolor[HTML]{EAF3FA}    0.326/0.248  \\ 
                    & \cellcolor[HTML]{FFE8D6}        & \cellcolor[HTML]{FFE8D6}    Generation*        & \cellcolor[HTML]{FFE8D6}    0.079/0.020          & \cellcolor[HTML]{FFE8D6}    0.308/0.186         & \cellcolor[HTML]{FFE8D6}    0.578/0.572         & \cellcolor[HTML]{FFE8D6}    0.057/0.006               & \cellcolor[HTML]{FFE8D6}    0.410/0.274        & \cellcolor[HTML]{FFE8D6}    0.532/0.442  & \cellcolor[HTML]{FFE8D6}    0.325/0.247  \\

                    & \cellcolor[HTML]{EAF3FA}    RoCoFT           & \cellcolor[HTML]{EAF3FA}    Predictor           & \cellcolor[HTML]{EAF3FA}    0.424/0.148           & \cellcolor[HTML]{EAF3FA}    0.854/0.274         & \cellcolor[HTML]{EAF3FA}    0.958/0.942         & \cellcolor[HTML]{EAF3FA}    0.924/0.102              & \cellcolor[HTML]{EAF3FA}    0.990/0.770   & \cellcolor[HTML]{EAF3FA}    1.340/1.040  & \cellcolor[HTML]{EAF3FA}   0.915/0.546\\ 
                    & \cellcolor[HTML]{EAF3FA}                    & \cellcolor[HTML]{EAF3FA}    Generator           & \cellcolor[HTML]{EAF3FA}    0.085/0.021          & \cellcolor[HTML]{EAF3FA}    0.332/0.191          & \cellcolor[HTML]{EAF3FA}    0.605/0.623         & \cellcolor[HTML]{EAF3FA}    0.895/0.099              & \cellcolor[HTML]{EAF3FA}    0.460/0.337    & \cellcolor[HTML]{EAF3FA}     0.641/0.497     & \cellcolor[HTML]{EAF3FA}     0.503/0.295  \\ 
                    & \cellcolor[HTML]{EAF3FA}                    & \cellcolor[HTML]{EAF3FA}    PredGen            & \cellcolor[HTML]{EAF3FA}    0.084/0.021          & \cellcolor[HTML]{EAF3FA}    0.311/0.187         & \cellcolor[HTML]{EAF3FA}    0.583/0.580         & \cellcolor[HTML]{EAF3FA}    0.060/0.007               & \cellcolor[HTML]{EAF3FA}    0.405/0.274   & \cellcolor[HTML]{EAF3FA}    0.543/0.448  & \cellcolor[HTML]{EAF3FA}    0.332/0.253 \\ 
                    & \cellcolor[HTML]{FFE8D6}                    & \cellcolor[HTML]{FFE8D6}    Generation*        & \cellcolor[HTML]{FFE8D6}    0.083/0.020          & \cellcolor[HTML]{FFE8D6}    0.308/0.186         & \cellcolor[HTML]{FFE8D6}    0.578/0.575         & \cellcolor[HTML]{FFE8D6}    0.061/0.008               & \cellcolor[HTML]{FFE8D6}    0.410/0.278   & \cellcolor[HTML]{FFE8D6}    0.548/0.450  & \cellcolor[HTML]{FFE8D6}    0.332/0.253\\

                    & \cellcolor[HTML]{EAF3FA}  DoRA              & \cellcolor[HTML]{EAF3FA}    Predictor           & \cellcolor[HTML]{EAF3FA}    0.511/0.150           & \cellcolor[HTML]{EAF3FA}    0.850/0.275  & \cellcolor[HTML]{EAF3FA}    0.960/0.945         & \cellcolor[HTML]{EAF3FA}     0.922/0.104               & \cellcolor[HTML]{EAF3FA}    0.980/0.780       & \cellcolor[HTML]{EAF3FA}     1.355/1.048  & \cellcolor[HTML]{EAF3FA}    0.930/0.550\\ 
                    & \cellcolor[HTML]{EAF3FA}                   & \cellcolor[HTML]{EAF3FA}    Generator           & \cellcolor[HTML]{EAF3FA}    0.086/0.022          & \cellcolor[HTML]{EAF3FA}    0.330/0.190          & \cellcolor[HTML]{EAF3FA}    0.607/0.625          & \cellcolor[HTML]{EAF3FA}    0.885/0.100               & \cellcolor[HTML]{EAF3FA}    0.462/0.338      & \cellcolor[HTML]{EAF3FA}      0.645/0.500  & \cellcolor[HTML]{EAF3FA}     0.503/0.296 \\ 
                    & \cellcolor[HTML]{EAF3FA}     & \cellcolor[HTML]{EAF3FA}    PredGen            & \cellcolor[HTML]{EAF3FA}    0.085/0.021          & \cellcolor[HTML]{EAF3FA}    0.301/0.184          & \cellcolor[HTML]{EAF3FA}    0.580/0.578        & \cellcolor[HTML]{EAF3FA}    0.061/0.007               & \cellcolor[HTML]{EAF3FA}    0.415/0.275    & \cellcolor[HTML]{EAF3FA}     0.540/0.445    & \cellcolor[HTML]{EAF3FA}    0.333/0.252 \\ 
                    & \cellcolor[HTML]{FFE8D6}     & \cellcolor[HTML]{FFE8D6}    Generation*        & \cellcolor[HTML]{FFE8D6}    0.084/0.021          & \cellcolor[HTML]{FFE8D6}    0.303/0.185          & \cellcolor[HTML]{FFE8D6}    0.584/0.580        & \cellcolor[HTML]{FFE8D6}    0.062/0.008               & \cellcolor[HTML]{FFE8D6}    0.418/0.278    & \cellcolor[HTML]{FFE8D6}    0.538/0.444    & \cellcolor[HTML]{FFE8D6}    0.334/0.253 \\

\hline
\multirow{12}{*}{\textbf{Llama2-13B}} 
                    & \cellcolor[HTML]{F8DADA}    LoRA         & \cellcolor[HTML]{F8DADA}    Predictor           & \cellcolor[HTML]{F8DADA}    0.370/0.130         & \cellcolor[HTML]{F8DADA}    0.800/0.250          & \cellcolor[HTML]{F8DADA}    0.920/0.910    & \cellcolor[HTML]{F8DADA}    0.880/0.090             & \cellcolor[HTML]{F8DADA}    0.950/0.720   & \cellcolor[HTML]{F8DADA}     1.280/1.000   & \cellcolor[HTML]{F8DADA}    0.867/0.517 \\ 
                    & \cellcolor[HTML]{F8DADA}            & \cellcolor[HTML]{F8DADA}    Generator           & \cellcolor[HTML]{F8DADA}    0.075/0.018          & \cellcolor[HTML]{F8DADA}    0.310/0.175          & \cellcolor[HTML]{F8DADA}    0.580/0.590          & \cellcolor[HTML]{F8DADA}    0.850/0.090               & \cellcolor[HTML]{F8DADA}    0.430/0.310     & \cellcolor[HTML]{F8DADA}     0.600/0.460      & \cellcolor[HTML]{F8DADA}    0.474/0.274    \\ 
                    & \cellcolor[HTML]{F8DADA}          & \cellcolor[HTML]{F8DADA}    PredGen            & \cellcolor[HTML]{F8DADA}    0.074/0.018          & \cellcolor[HTML]{F8DADA}    0.287/0.169    & \cellcolor[HTML]{F8DADA}    0.550/0.540         & \cellcolor[HTML]{F8DADA}    0.052/0.006               & \cellcolor[HTML]{F8DADA}    0.380/0.250     & \cellcolor[HTML]{F8DADA}     0.500/0.400     & \cellcolor[HTML]{F8DADA}    0.308/0.231   \\
                    & \cellcolor[HTML]{FFE8D6}          & \cellcolor[HTML]{FFE8D6}    Generation*        & \cellcolor[HTML]{FFE8D6}    0.073/0.018          & \cellcolor[HTML]{FFE8D6}    0.289/0.170          & \cellcolor[HTML]{FFE8D6}    0.553/0.542         & \cellcolor[HTML]{FFE8D6}    0.051/0.006               & \cellcolor[HTML]{FFE8D6}    0.385/0.254     & \cellcolor[HTML]{FFE8D6}    0.495/0.402     & \cellcolor[HTML]{FFE8D6}    0.309/0.232   \\

                    & \cellcolor[HTML]{F8DADA}     AdaLoRA           & \cellcolor[HTML]{F8DADA}    Predictor           & \cellcolor[HTML]{F8DADA}     0.360/0.125          & \cellcolor[HTML]{F8DADA}    0.810/0.255   & \cellcolor[HTML]{F8DADA}    0.930/0.920         & \cellcolor[HTML]{F8DADA}    0.890/0.095   & \cellcolor[HTML]{F8DADA}    0.960/0.730   & \cellcolor[HTML]{F8DADA}     1.300/1.010  & \cellcolor[HTML]{F8DADA}    0.875/0.522 \\ 
                    & \cellcolor[HTML]{F8DADA}                  & \cellcolor[HTML]{F8DADA}    Generator           & \cellcolor[HTML]{F8DADA}    0.078/0.019         & \cellcolor[HTML]{F8DADA}    0.315/0.178          & \cellcolor[HTML]{F8DADA}    0.585/0.600          & \cellcolor[HTML]{F8DADA}    0.860/0.093               & \cellcolor[HTML]{F8DADA}    0.440/0.320     & \cellcolor[HTML]{F8DADA}    0.610/0.470  & \cellcolor[HTML]{F8DADA}     0.481/0.280 \\ 
                    & \cellcolor[HTML]{F8DADA}                   & \cellcolor[HTML]{F8DADA}    PredGen            & \cellcolor[HTML]{F8DADA}    0.078/0.019          & \cellcolor[HTML]{F8DADA}    0.300/0.175          & \cellcolor[HTML]{F8DADA}    0.530/0.530         & \cellcolor[HTML]{F8DADA}   0.054/0.006             & \cellcolor[HTML]{F8DADA}    0.390/0.255  & \cellcolor[HTML]{F8DADA}    0.510/0.410  & \cellcolor[HTML]{F8DADA}    0.315/0.236 \\ 
                    & \cellcolor[HTML]{FFE8D6}                   & \cellcolor[HTML]{FFE8D6}    Generation*        & \cellcolor[HTML]{FFE8D6}    0.077/0.019          & \cellcolor[HTML]{FFE8D6}    0.302/0.176          & \cellcolor[HTML]{FFE8D6}    0.528/0.529         & \cellcolor[HTML]{FFE8D6}    0.055/0.007             & \cellcolor[HTML]{FFE8D6}    0.395/0.258  & \cellcolor[HTML]{FFE8D6}    0.508/0.411  & \cellcolor[HTML]{FFE8D6}    0.316/0.237 \\

                    & \cellcolor[HTML]{F8DADA}    RoCoFT            & \cellcolor[HTML]{F8DADA}    Predictor           & \cellcolor[HTML]{F8DADA}    0.380/0.135           & \cellcolor[HTML]{F8DADA}    0.790/0.245           & \cellcolor[HTML]{F8DADA}    0.910/0.900           & \cellcolor[HTML]{F8DADA}    0.870/0.088               & \cellcolor[HTML]{F8DADA}    0.940/0.710  & \cellcolor[HTML]{F8DADA}    1.270/0.990  & \cellcolor[HTML]{F8DADA}    0.860/0.511 \\ 
                    & \cellcolor[HTML]{F8DADA}                    & \cellcolor[HTML]{F8DADA}    Generator           & \cellcolor[HTML]{F8DADA}    0.072/0.017          & \cellcolor[HTML]{F8DADA}    0.305/0.172          & \cellcolor[HTML]{F8DADA}    0.575/0.580          & \cellcolor[HTML]{F8DADA}    0.845/0.088              & \cellcolor[HTML]{F8DADA}    0.425/0.305    & \cellcolor[HTML]{F8DADA}     0.590/0.450  & \cellcolor[HTML]{F8DADA}    0.860/0.511\\ 
                    & \cellcolor[HTML]{F8DADA}                   & \cellcolor[HTML]{F8DADA}    PredGen            & \cellcolor[HTML]{F8DADA}    0.070/0.017         & \cellcolor[HTML]{F8DADA}    0.288/0.169        & \cellcolor[HTML]{F8DADA}    0.545/0.538        & \cellcolor[HTML]{F8DADA}    0.053/0.007               & \cellcolor[HTML]{F8DADA}   0.375/0.248     & \cellcolor[HTML]{F8DADA}     0.495/0.401  & \cellcolor[HTML]{F8DADA}    0.307/0.232\\ 
                    & \cellcolor[HTML]{FFE8D6}                   & \cellcolor[HTML]{FFE8D6}    Generation*        & \cellcolor[HTML]{FFE8D6}    0.071/0.018         & \cellcolor[HTML]{FFE8D6}    0.286/0.170         & \cellcolor[HTML]{FFE8D6}    0.548/0.540        & \cellcolor[HTML]{FFE8D6}    0.054/0.007               & \cellcolor[HTML]{FFE8D6}    0.378/0.250     & \cellcolor[HTML]{FFE8D6}    0.493/0.400  & \cellcolor[HTML]{FFE8D6}    0.308/0.233\\

                    & \cellcolor[HTML]{F8DADA}     DoRA              & \cellcolor[HTML]{F8DADA}    Predictor           & \cellcolor[HTML]{F8DADA}    0.365/0.128           & \cellcolor[HTML]{F8DADA}    0.805/0.252          & \cellcolor[HTML]{F8DADA}    0.925/0.915         & \cellcolor[HTML]{F8DADA}    0.924/0.102              & \cellcolor[HTML]{F8DADA}    0.955/0.725   & \cellcolor[HTML]{F8DADA}    1.290/1.005  & \cellcolor[HTML]{F8DADA}    0.877/0.521 \\ 
                    & \cellcolor[HTML]{F8DADA}                   & \cellcolor[HTML]{F8DADA}    Generator           & \cellcolor[HTML]{F8DADA}    0.076/0.018          & \cellcolor[HTML]{F8DADA}    0.312/0.176        & \cellcolor[HTML]{F8DADA}    0.590/0.605         & \cellcolor[HTML]{F8DADA}    0.855/0.092             & \cellcolor[HTML]{F8DADA}    0.435/0.315    & \cellcolor[HTML]{F8DADA}     0.605/0.465  & \cellcolor[HTML]{F8DADA}    0.479/0.279\\ 
                    & \cellcolor[HTML]{F8DADA}            & \cellcolor[HTML]{F8DADA}    PredGen            & \cellcolor[HTML]{F8DADA}    0.070/0.016           & \cellcolor[HTML]{F8DADA}    0.295/0.172        & \cellcolor[HTML]{F8DADA}    0.555/0.548          & \cellcolor[HTML]{F8DADA}    0.053/0.006               & \cellcolor[HTML]{F8DADA}    0.385/0.252   & \cellcolor[HTML]{F8DADA}     0.505/0.405   & \cellcolor[HTML]{F8DADA}    0.311/0.233 \\ 
                    & \cellcolor[HTML]{FFE8D6}            & \cellcolor[HTML]{FFE8D6}    Generation*        & \cellcolor[HTML]{FFE8D6}    0.069/0.016          & \cellcolor[HTML]{FFE8D6}    0.297/0.173          & \cellcolor[HTML]{FFE8D6}    0.558/0.550          & \cellcolor[HTML]{FFE8D6}    0.054/0.007               & \cellcolor[HTML]{FFE8D6}    0.388/0.254   & \cellcolor[HTML]{FFE8D6}    0.502/0.406   & \cellcolor[HTML]{FFE8D6}    0.312/0.234 \\

\hline
\multirow{12}{*}{\textbf{Llama2-8B}} 
                    & \cellcolor[HTML]{ECEEFF}    LoRA              & \cellcolor[HTML]{ECEEFF}    Predictor           & \cellcolor[HTML]{ECEEFF}    0.380/0.140           & \cellcolor[HTML]{ECEEFF}    0.820/0.260          & \cellcolor[HTML]{ECEEFF}    0.940/0.925          & \cellcolor[HTML]{ECEEFF}    0.910/0.098              & \cellcolor[HTML]{ECEEFF}    0.970/0.740  & \cellcolor[HTML]{ECEEFF}    1.310/1.020   & \cellcolor[HTML]{ECEEFF}    0.888/0.531\\ 
                    & \cellcolor[HTML]{ECEEFF}                    & \cellcolor[HTML]{ECEEFF}    Generator           & \cellcolor[HTML]{ECEEFF}    0.081/0.019         & \cellcolor[HTML]{ECEEFF}    0.320/0.180          & \cellcolor[HTML]{ECEEFF}    0.595/0.610          & \cellcolor[HTML]{ECEEFF}    0.870/0.095            & \cellcolor[HTML]{ECEEFF}    0.440/0.325    & \cellcolor[HTML]{ECEEFF}    0.620/0.480   & \cellcolor[HTML]{ECEEFF}     0.488/0.285 \\ 
                    & \cellcolor[HTML]{ECEEFF}                   & \cellcolor[HTML]{ECEEFF}    PredGen            & \cellcolor[HTML]{ECEEFF}    0.077/0.019           & \cellcolor[HTML]{ECEEFF}    0.298/0.173          & \cellcolor[HTML]{ECEEFF}    0.565/0.555         & \cellcolor[HTML]{ECEEFF}    0.055/0.006              & \cellcolor[HTML]{ECEEFF}    0.395/0.260      & \cellcolor[HTML]{ECEEFF}    0.520/0.420  & \cellcolor[HTML]{ECEEFF}   0.318/0.239 \\ 
                    & \cellcolor[HTML]{FFE8D6}                   & \cellcolor[HTML]{FFE8D6}    Generation*        & \cellcolor[HTML]{FFE8D6}    0.078/0.019           & \cellcolor[HTML]{FFE8D6}    0.300/0.174          & \cellcolor[HTML]{FFE8D6}    0.562/0.553         & \cellcolor[HTML]{FFE8D6}    0.054/0.006               & \cellcolor[HTML]{FFE8D6}    0.398/0.263      & \cellcolor[HTML]{FFE8D6}    0.518/0.419  & \cellcolor[HTML]{FFE8D6}    0.320/0.240 \\

                    & \cellcolor[HTML]{ECEEFF}    AdaLoRA           & \cellcolor[HTML]{ECEEFF}    Predictor           & \cellcolor[HTML]{ECEEFF}    0.375/0.135           & \cellcolor[HTML]{ECEEFF}    0.830/0.265          & \cellcolor[HTML]{ECEEFF}    0.945/0.930         & \cellcolor[HTML]{ECEEFF}    0.910/0.098               & \cellcolor[HTML]{ECEEFF}    0.980/0.750    & \cellcolor[HTML]{ECEEFF}      1.320/1.030  & \cellcolor[HTML]{ECEEFF}    0.893/0.535\\ 
                    & \cellcolor[HTML]{ECEEFF}                    & \cellcolor[HTML]{ECEEFF}    Generator           & \cellcolor[HTML]{ECEEFF}    0.080/0.020           & \cellcolor[HTML]{ECEEFF}    0.325/0.183  & \cellcolor[HTML]{ECEEFF}    0.600/0.615         & \cellcolor[HTML]{ECEEFF}    0.875/0.097              & \cellcolor[HTML]{ECEEFF}    0.450/0.330     & \cellcolor[HTML]{ECEEFF}     0.630/0.485   & \cellcolor[HTML]{ECEEFF}    0.493/0.288 \\ 
                    & \cellcolor[HTML]{ECEEFF}                   & \cellcolor[HTML]{ECEEFF}    PredGen            & \cellcolor[HTML]{ECEEFF}    0.078/0.019          & \cellcolor[HTML]{ECEEFF}    0.303/0.177          & \cellcolor[HTML]{ECEEFF}    0.570/0.560          & \cellcolor[HTML]{ECEEFF}    0.057/0.007              & \cellcolor[HTML]{ECEEFF}    0.400/0.265     & \cellcolor[HTML]{ECEEFF}    0.509/0.410  & \cellcolor[HTML]{ECEEFF}     0.323/0.243  \\ 
                    & \cellcolor[HTML]{FFE8D6}                   & \cellcolor[HTML]{FFE8D6}    Generation*        & \cellcolor[HTML]{FFE8D6}    0.077/0.019          & \cellcolor[HTML]{FFE8D6}    0.305/0.178          & \cellcolor[HTML]{FFE8D6}    0.573/0.562          & \cellcolor[HTML]{FFE8D6}    0.058/0.007              & \cellcolor[HTML]{FFE8D6}    0.403/0.268     & \cellcolor[HTML]{FFE8D6}    0.505/0.412  & \cellcolor[HTML]{FFE8D6}     0.322/0.242  \\

                    & \cellcolor[HTML]{ECEEFF}     RoCoFT            & \cellcolor[HTML]{ECEEFF}    Predictor           & \cellcolor[HTML]{ECEEFF}    0.390/0.145          & \cellcolor[HTML]{ECEEFF}    0.810/0.255          & \cellcolor[HTML]{ECEEFF}    0.935/0.920         & \cellcolor[HTML]{ECEEFF}    0.910/0.098              & \cellcolor[HTML]{ECEEFF}    0.960/0.730     & \cellcolor[HTML]{ECEEFF}    1.300/1.015   & \cellcolor[HTML]{ECEEFF}    0.884/0.527 \\ 
                    & \cellcolor[HTML]{ECEEFF}                & \cellcolor[HTML]{ECEEFF}    Generator           & \cellcolor[HTML]{ECEEFF}    0.082/0.020   & \cellcolor[HTML]{ECEEFF}    0.315/0.177         & \cellcolor[HTML]{ECEEFF}    0.585/0.605          & \cellcolor[HTML]{ECEEFF}    0.865/0.092              & \cellcolor[HTML]{ECEEFF}    0.435/0.320      & \cellcolor[HTML]{ECEEFF}     0.610/0.475   & \cellcolor[HTML]{ECEEFF}    0.482/0.282\\ 
                    & \cellcolor[HTML]{ECEEFF}                   & \cellcolor[HTML]{ECEEFF}    PredGen            & \cellcolor[HTML]{ECEEFF}    0.079/0.020           & \cellcolor[HTML]{ECEEFF}    0.288/0.169         & \cellcolor[HTML]{ECEEFF}    0.565/0.558    & \cellcolor[HTML]{ECEEFF}    0.058/0.007             & \cellcolor[HTML]{ECEEFF}     0.385/0.255  & \cellcolor[HTML]{ECEEFF}    0.530/0.425  & \cellcolor[HTML]{ECEEFF}    0.317/0.238 \\ 
                    & \cellcolor[HTML]{FFE8D6}                   & \cellcolor[HTML]{FFE8D6}    Generation*        & \cellcolor[HTML]{FFE8D6}    0.078/0.020           & \cellcolor[HTML]{FFE8D6}    0.290/0.170          & \cellcolor[HTML]{FFE8D6}    0.567/0.559         & \cellcolor[HTML]{FFE8D6}    0.059/0.008             & \cellcolor[HTML]{FFE8D6}     0.388/0.258  & \cellcolor[HTML]{FFE8D6}    0.528/0.426  & \cellcolor[HTML]{FFE8D6}    0.318/0.239 \\
                    & \cellcolor[HTML]{ECEEFF}    DoRA              & \cellcolor[HTML]{ECEEFF}    Predictor           & \cellcolor[HTML]{ECEEFF}    0.385/0.138           & \cellcolor[HTML]{ECEEFF}    0.825/0.261   & \cellcolor[HTML]{ECEEFF}    0.950/0.935         & \cellcolor[HTML]{ECEEFF}    0.905/0.096     & \cellcolor[HTML]{ECEEFF}    0.975/0.745  & \cellcolor[HTML]{ECEEFF}    1.315/1.025   & \cellcolor[HTML]{ECEEFF}     0.893/0.533     \\ 
                    & \cellcolor[HTML]{ECEEFF}                & \cellcolor[HTML]{ECEEFF}    Generator           & \cellcolor[HTML]{ECEEFF}    0.078/0.019          & \cellcolor[HTML]{ECEEFF}    0.322/0.179        & \cellcolor[HTML]{ECEEFF}    0.592/0.608         & \cellcolor[HTML]{ECEEFF}    0.880/0.096               & \cellcolor[HTML]{ECEEFF}    0.445/0.328  & \cellcolor[HTML]{ECEEFF}    0.625/0.482  & \cellcolor[HTML]{ECEEFF}   0.490/0.285\\ 
                    & \cellcolor[HTML]{ECEEFF}                  & \cellcolor[HTML]{ECEEFF}    PredGen            & \cellcolor[HTML]{ECEEFF}    0.073/0.018          & \cellcolor[HTML]{ECEEFF}    0.300/0.175         & \cellcolor[HTML]{ECEEFF}0.562/0.558         & \cellcolor[HTML]{ECEEFF}    0.066/0.007    & \cellcolor[HTML]{ECEEFF}    0.390/0.262   & \cellcolor[HTML]{ECEEFF}    0.525/0.425     & \cellcolor[HTML]{ECEEFF}     0.319/0.241  \\ 
                    & \cellcolor[HTML]{FFE8D6}                  & \cellcolor[HTML]{FFE8D6}    Generation*        & \cellcolor[HTML]{FFE8D6}    0.072/0.018          & \cellcolor[HTML]{FFE8D6}    0.302/0.176         & \cellcolor[HTML]{FFE8D6}    0.564/0.560         & \cellcolor[HTML]{FFE8D6}    0.065/0.007    & \cellcolor[HTML]{FFE8D6}    0.393/0.265   & \cellcolor[HTML]{FFE8D6}    0.523/0.426     & \cellcolor[HTML]{FFE8D6}     0.320/0.242  \\

\hline
\end{tabular}%
}
\caption{Regression performance of different PEFT methods across benchmarks, reported as MAE/MSE. \textbf{Generation*} denotes single-token generation.}
\label{tab:predgen_regression}
\end{table*}

\section{PredGen vs. One-Token Generation:} \label{sec:predgenvsone}
The original PredGen framework~\citep{kowsher2025predicting} showed that generating multiple output tokens retains higher mutual information with the input, leading to better performance on regression and classification tasks compared to pooling-based methods. However, this approach incurs high computational cost due to sequence-level decoding. To improve efficiency, we propose a simplified variant that performs \textit{single-token generation} or \textit{masked prediction}, predicting one specific token (e.g., via a masked or prompt-inserted position). We extract its hidden state and pass it through a lightweight MLP for final prediction. This method achieves competitive results across six regression benchmarks (Table~\ref{tab:predgen_regression}).

\begin{figure*}[htbp]
\begin{center}
    \includegraphics[width=1.0\linewidth]{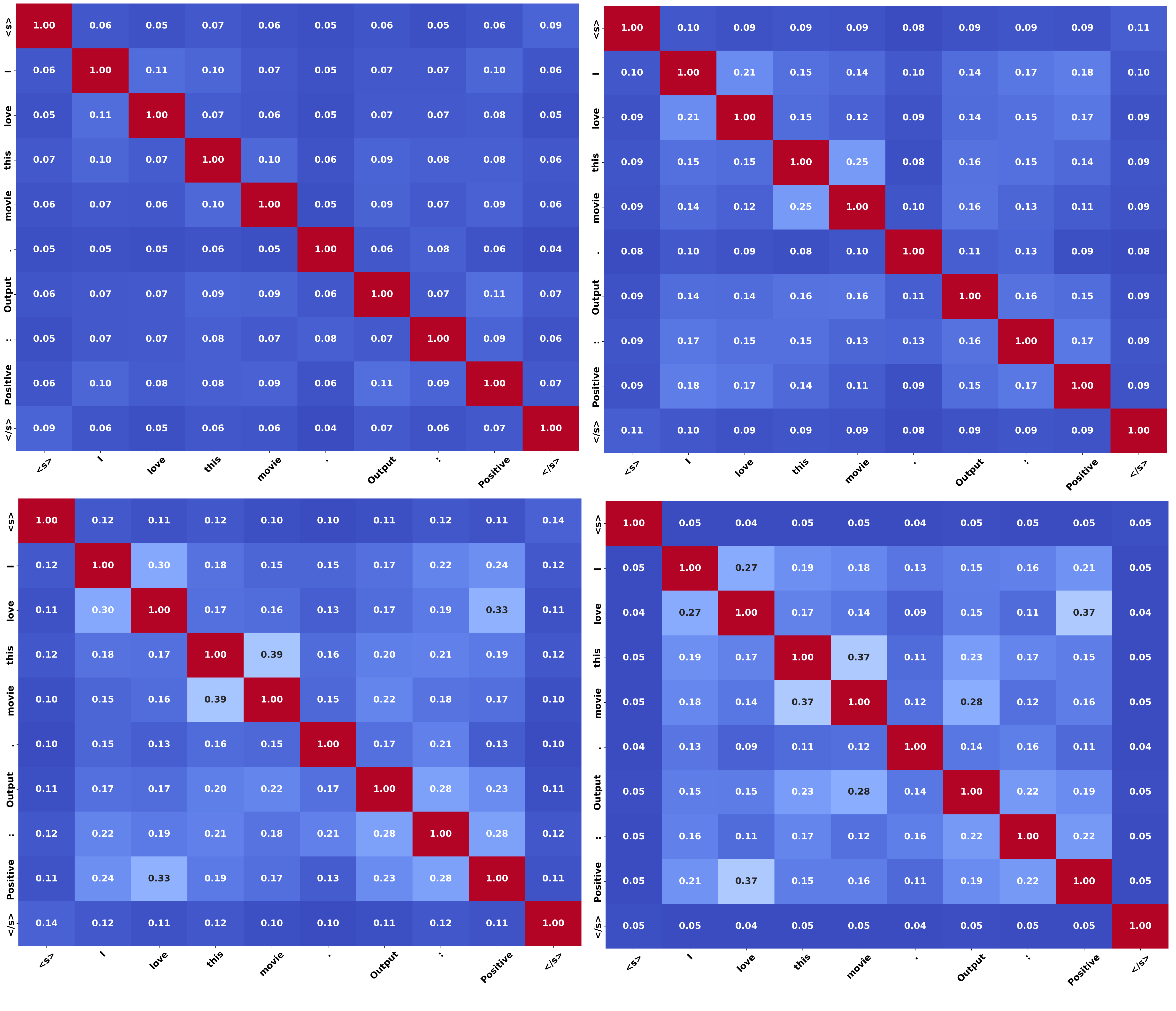}
   \end{center}
\caption{Token-level mutual information on the SST-2 dataset, computed using representations from layers 1, 8, 16, and 30 of MobileLLM. The figure highlights how information evolves across layers during fine-tuning.}
\label{fig:sst_example}
\end{figure*}

\begin{figure*}[htbp]
\begin{center}
    \includegraphics[width=1.0\linewidth]{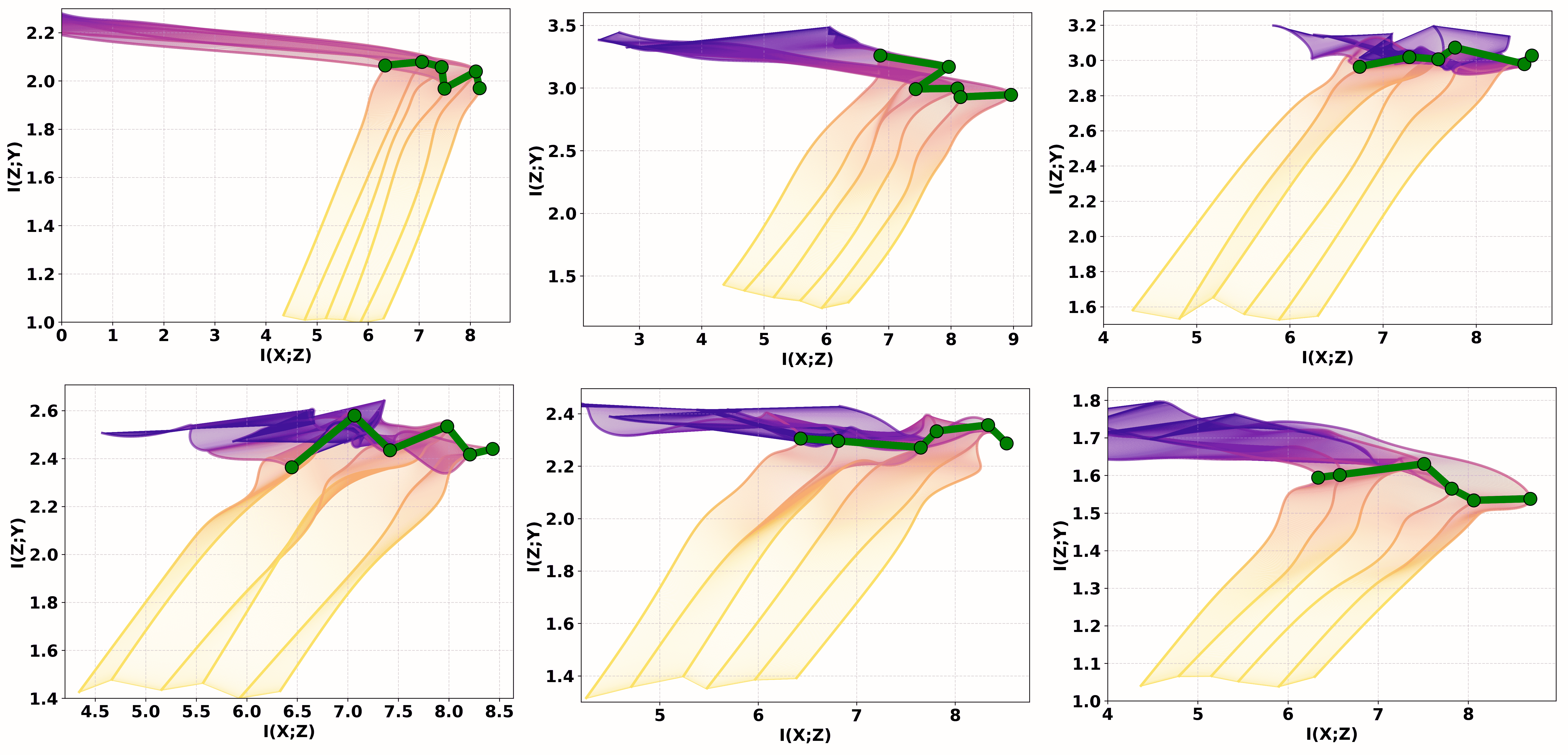}
   \end{center}
\caption{Mutual information on the ETTh1 dataset for different prediction horizons: 24, 96, 128, 380, 512, and 1038. The figure illustrates how information flow varies as the prediction target becomes more distant.}
\label{fig:y2}
\end{figure*}

\end{document}